\documentclass[11pt,a4paper]{article}
\usepackage[reqno]{amsmath}
\usepackage[pdftex]{graphicx}
\usepackage{epstopdf}
\usepackage{amsthm}
\usepackage[ruled,vlined]{algorithm2e}
\usepackage{etoolbox}
\makeatletter
\patchcmd{\@algocf@start}
  {-1.5em}
  {0pt}
  {}{}
\makeatother

\usepackage{amssymb,array}
\usepackage{float,caption,subcaption}
\usepackage{multirow,footmisc,picinpar}
\RequirePackage{xcolor}[2.11]
\usepackage[colorlinks=true, linkcolor=cyan,
    		   citecolor=red!60!black, urlcolor=cyan]{hyperref}
\usepackage[noabbrev,capitalise,nameinlink]{cleveref}

\def\ba{\mathbf{a}}
\def\bb{\mathbf{b}}

\def\bx{\mathbf{x}}
\def\by{\mathbf{y}}
\def\bz{\mathbf{z}}

\def\bpi{{\boldsymbol  \pi}}
\def\bg{{\boldsymbol  g}}
\def\bt{{\boldsymbol  \theta}}
\def\D{\mathcal{D}}
\def\L{\mathcal{L}}
\def\R{\mathbb{R}}

\def\S{\mathcal{S}}
\def\K{\mathcal{K}}

\graphicspath{ {./images/} }

\numberwithin{equation}{section}
\newtheorem{theorem}{Theorem}[section]
\newtheorem{lemma}{Lemma}[section]

\newtheorem{corollary}{Corollary}[section]
\newtheorem{definition}{Definition}[section]
\newtheorem{remark}{Remark}[section]
\newtheorem{example}{Example}[section]
\newtheorem{assumption}{Assumption}[section]

\author{
  Shenglong Zhou and Geoffrey Ye Li\\
  ITP Lab, Department of EEE\\  Imperial College London, United Kingdom \\
  Emails:\{shenglong.zhou, geoffrey.li\}@imperial.ac.uk
  
}

\title{\vspace{-1.0cm}
 Communication-Efficient
ADMM-based Federated Learning
\vspace{-0.25cm}}

\date{}

\begin{document}

\maketitle
 
\vspace{-1.3cm}

\begin{abstract}  
\noindent \textbf{Abstract:} Federated learning has shown its advances over the last few years but is facing many challenges, such as how algorithms save communication resources, how they reduce computational costs, and whether they converge. To address these issues, this paper proposes exact and inexact ADMM-based federated learning. They  are not only communication-efficient but also converge linearly under very mild conditions, such as convexity-free and irrelevance to data distributions. Moreover,  the inexact version has low computational complexity, thereby alleviating the computational burdens significantly.
\vspace{0.3cm}
 
 
\end{abstract}

\numberwithin{equation}{section}
\section{Introduction}

Federated learning, as an effective machine learning technique, gains popularity in recent years due to its ability to deal with various issues like data privacy, data security, and data access  to heterogeneous data. Typical applications include  vehicular communications \cite{samarakoon2019distributed, pokhrel2020federated,
elbir2020federated, posner2021federated}, digital health 
\cite{rieke2020future}, and  smart manufacturing \cite{fazel2013hankel}, just to name a few. The earliest work for federated learning can be traced back to \cite{konevcny2015federated} in 2015  and \cite{konevcny2016federated} in 2016. It is still undergoing development and also facing many challenges.  For example, how do algorithms save the communication resources during the learning process? What is their computational and convergence performance? We refer to some nice surveys \cite{kairouz2019advances,li2020federated,qin2021federated} for more open questions.
\subsection{Related work}
{\bf (a) Saving communication resources.} In distributed learning,  local clients and a central server communicate frequently, which sometimes is relatively inefficient. For example, if the updated parameters from some local clients are negligible, then their impact on the global aggregation (or averaging) by the central server can be ignored. On the other hand, even though the updated parameters from some local clients are desirable, the channel conditions might be poor and thus  more communication resources (e.g., transmission power and bandwidth) are required  as a remedy. Therefore, saving communications resources which can be fulfilled by skipping or reducing unnecessary communication rounds tends to be inevitable in practical.

Motivated by this, there is an impressive body of work on developing algorithms that aims to reduce the communication rounds. The stochastic gradient descent (SGD) is one of the most extensively used schemes. It executes global aggregation in a periodic fashion so as to reduce the communication rounds \cite{DeepLearning2015,AsynchronousStochastic2017, stich2018local, yu2019parallel, Lin2020Don, wang2021cooperative}. However, to establish the convergence theory, the family of SGD frequently assumes the data from local clients to be identically and independently distributed (i.i.d.), which is apparently unrealistic for federated learning settings where data distributions are  heterogeneous. 

A separated line of research investigates algorithms that make assumptions on the objective functions themselves, and hence they are irrelevant to the distributions of the involved data, such as the distributed gradient descent algorithm \cite{wang2019adaptive} for the edge computing in wireless communication, lazily aggregated gradient algorithm \cite{LAG2018} for centralized learning,  lazy and approximate dual gradients algorithm \cite{liu2021decentralized} for 
decentralized learning, and federated (iterative) hard thresholding algorithm  \cite{tong2020federated}. Despite no assumptions on the data distributions, strong conditions are often imposed on the objective functions of the learning optimization problem, including (gradient) Lipschitz continuity, strong smoothness, convexity, or strong convexity.

{\bf (b) ADMM-based learning.} Over the last few decades, the alternating direction method  of multipliers (ADMM) has shown its advances both in theoretical and numerical aspects, with extensive applications into various disciplinarians. Fairly recently, there is a success of implementation ADMM into distributed learning \cite{boyd2011distributed,song2016fast,
zhang2018improving,zheng2018stackelberg, 
graf2019distributed,huang2019dp,elgabli2020fgadmm,issaid2020communication}. When it comes to federated learning settings,  a robust ADMM-based decentralized federated learning has been  proposed in \cite{Li2017RobustFL} to mitigate the influence of some falsified
data. Very lately, inexact ADMM has drawn some attention in federated learning, thanks to its ability to alleviate computational burdens. For instance,   an inexact ADMM-based federated meta-learning algorithm has been cast in \cite{Inexact-ADMM2021} for fast and continual edge learning and a differential private inexact ADMM-based federated learning algorithm has been designed in \cite{ryu2021differentially} to accelerate the computation and protect data privacy. Again, we shall emphasize that these algorithms have been provably convergent but still required some aforementioned restrictive assumptions.

\subsection{Our contributions}
In conventional federated learning, at every step, all local clients update their parameters in parallel and then send them to the central server for aggregation. When  ADMM  is applied into federated learning, it can be viewed as a scene where the local clients update not only their parameters but also the dual parameters of the target optimization problem, followed by the global aggregation from the central server using both parameters.  In this regard, the scheme of the conventional federated learning can be deemed as a special case of ADMM. As a by-product of this paper, we reveal the relationship between federated learning and linearised inexact ADMM-based federated learning. Their frameworks are presented in Figure \ref{fig:FL-ADMM}, from which one can see that the former can be deemed as a special case of the latter.

\begin{figure}[!th]
\begin{subfigure}{.47\textwidth}
	\centering
\includegraphics[width=.99\linewidth]{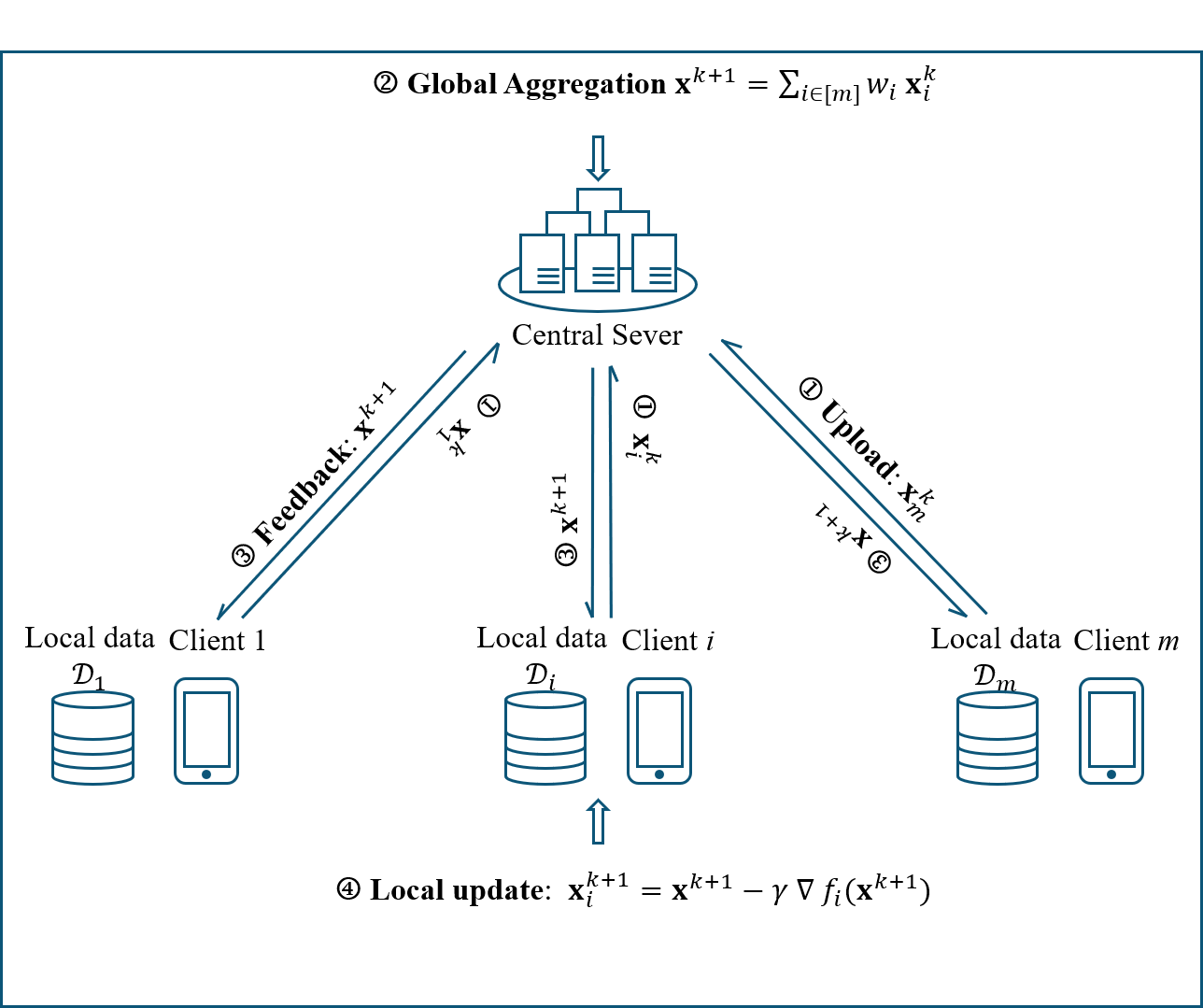}
	\caption{Federated learning.}
	\label{fig:succ-s-ex1-greedy}
\end{subfigure}  
\begin{subfigure}{.53\textwidth}
	\centering
	\includegraphics[width=.87\linewidth]{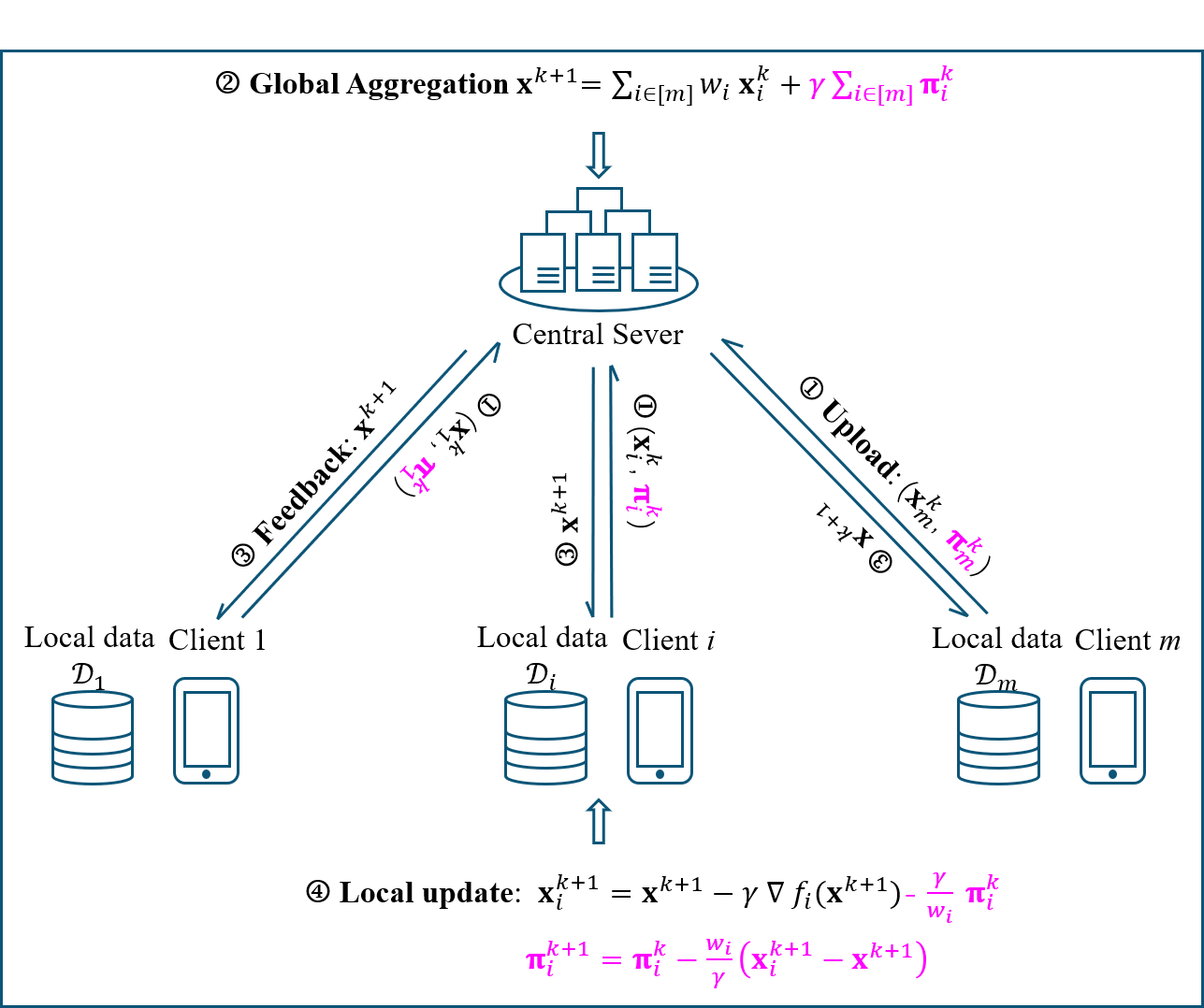}
	\caption{Linearised inexact ADMM-based federated learning.}
	\label{fig:succ-s-ex1-non-greedy}
\end{subfigure} 
\caption{Algorithm \ref{algorithm-FL} v.s. Algorithm \ref{algorithm-FL-riADMM}.\label{fig:FL-ADMM}}
\end{figure}

The main contribution of this paper is developing two ADMM-based federated learning algorithms  to save communication resources, reduce computational burdens, and converge under relatively weak assumptions. 

 I)  Based on the framework of ADMM, we design a communication-efficient ADMM ({\tt CEADMM}) algorithm  and  an inexact ADMM  ({\tt ICEADMM}) algorithm that  reduce the communication rounds between the local clients and  the central server. Precisely, between two rounds of communications, there are $k_0$ iterations allowing for local clients to update their parameters. The numerical experiments demonstrate that the larger $k_0$, the fewer communication rounds required to converge, as shown in Figures  \ref{fig:log-k0-aggr-1} or \ref{fig:log-k0-aggr-2}. This hints that setting a proper larger $k_0$ would save communication resources greatly.

II) To alleviate the computational burdens, {\tt ICEADMM} is also capable of cutting down the computational costs, thereby accelerating the entire learning process dramatically, as shown in Figures \ref{fig:k0-time-1} and \ref{fig:k0-time-2}. Nevertheless, the established theory and empirical numerical experiments show that {\tt ICEADMM} does not compromise any accuracy.
 
III) Both algorithms converge to a stationary point (see Definition \ref{def-sta}) of the learning optimization problem in \eqref{FL-opt} with a linear rate $O(k_0/k)$ only under two mild conditions: gradient Lipschitz continuity (also known for the L-smoothness in many publications)  and the boundedness of a level set, as in Theorems \ref{global-convergence-exact} and \ref{complexity-thorem-gradient}, where $k$ is the iteration number and $k_0$ is the number of the gap between two rounds of communications. These conditions are convexity-free and independent of distributions of the data. Hence, they are weaker than those used to establish convergence for the current distributed and federated learning algorithms. If we further assume the convexity, then both algorithms can achieve the optimal parameters, as shown in Corollary \ref{L-global-convergence}.

\subsection{Organization and notation}
This paper is organized as follows.  In the next section, we introduce federated learning and the framework of ADMM.  In  Section \ref{sec:ceadmm},  we design the communication-efficient ADMM ({\tt CEADMM}), followed by the establishment of its global convergence and the complexity analysis. Then similar results  are achieved for the inexact communication-efficient ADMM ({\tt ICEADMM}) in Section  \ref{sec:iceadmm}. In Section  \ref{sec:num}, we conduct some numerical experiments to demonstrate the performance of two algorithms and testify our established theorems.  Concluding remarks are given in the last section.

We end this section with summarizing the notation that will be employed throughout this paper. We use  plain,  bold, and capital letters to present scalars, vectors, and matrices, respectively, e.g., $\gamma$ and $\sigma$ are scalars, $\bx, \bx_i$ and $\bx_i^k$  are vectors, $X $ and $X ^k$ are matrices. Let $\lfloor t\rfloor$ represent  the smallest integer that is no greater than $t$ and $[m]:=\{1,2,\ldots,m\}$ with `$:=$' meaning define.  In this paper, $\R^n$ denotes the $n$-dimensional Euclidean space equipped with the inner product $\langle\cdot,\cdot\rangle$ defined by $\langle\bx,\by\rangle:=\sum_i x_iy_i$. Let $\|\cdot\|$ be the Euclidean norm, namely,  $\|\bx\|^2=\langle\bx,\bx\rangle$ and $\|\cdot\|_H$ be the  weighted norm defined by  $\|\bx\|_H^2:=\langle H\bx,\bx\rangle$.
 Write the identity matrix as $I$ and a positive semidefinite matrix $A$ as $A\succeq 0$. In particular, $A\succeq B$ represents $A-B\succeq 0$.  A function, $f$, is said to be gradient Lipschitz continuous with a constant $r>0$ if 
 \begin{eqnarray}\label{Lip-r} 
\|\nabla f(\bx)-\nabla f(\bz  ) \|  \leq r\| \bx -\bz  \|.
  \end{eqnarray}
for any $\bx$ and $\bz$, where $\nabla  f(\bx)$  represents the gradient of $f$ with respect to $\bx$.

 \section{Federated Learning and ADMM}
Suppose we have $m$ local clients/edge nodes with datasets $\{\D_i\}_{i\in[m]}$ as shown in Figure \ref{fig:succ-s-ex1-non-greedy}. Each client has the total loss $f_i(\bx) := \sum_{\bt\in\D_i} \ell_i(\bx; \bt)$, where $\ell_i(\cdot; \bt):\R^n\mapsto\R$ is a continuous loss function and bounded from below,  and $\bx\in\R^n$ is the parameter to be learned.
Below are two examples that will be used for our numerical experiments.

\begin{example}[Least square loss] \label{ex:lr} Suppose the $i$th client has data $\D_i=\{(\ba^i_1,b^i_1),\ldots,(\ba^i_{d_i},b^i_{d_i})\}$, where $\ba^i_j\in\R^n$, $b^i_j\in\R$, and $d_i$ is the cardinality of $\D_i$.  Let $\bt:=(\ba,b)$, then the least square loss is defined by
\begin{eqnarray} \label{least-squares}
 \arraycolsep=1.4pt\def\arraystretch{1.75}
\begin{array}{llll}
f_i(\bx)&=& \sum_{j=1}^{d_i} \frac{1}{2}(\langle \ba^i_j,\bx\rangle-b^i_j)^2.
\end{array} 
\end{eqnarray}
\end{example}
\begin{example}[$\ell_2$ norm regularized logistic loss]    \label{ex:lg}
Similarly, the $i$th client has data $\D_i$ but with $b^i_j\in\{0,1\}$. The $\ell_2$ norm regularized logistic loss is given by 
\begin{eqnarray} \label{logist-loss}
 \arraycolsep=1.4pt\def\arraystretch{1.75}
\begin{array}{llll}
f_i(\bx)&=&
\sum_{j=1}^{d_i}[\ln (1+e^{\langle\ba^i_j,\bx\rangle} )-b^i_j\langle\ba^i_j,\bx\rangle+\frac{\mu}{2d_i}\|\bx\|^2],
\end{array} 
\end{eqnarray}
where  $\mu>0$ is a penalty parameter.
\end{example}
The overall loss function  can be defined by 
\begin{eqnarray*}\begin{array}{lll}
f(\bx) :=   \sum_{i=1}^{m}  w_i f_i(\bx),
\end{array}\end{eqnarray*}
where  $w_i, i\in[m]$ are positive weights  that  satisfy $\sum_{i=1}^{m}  w_i=1.$ 
 Particular choices for such weights are $ w_i = d_i/d$ with   $d:= \sum_{i=1}^{m}  d_i$. 
  Federated learning aims to learn a best parameter $\bx^*$ at the central server that minimizes the overall loss function, $f$, namely,
\begin{eqnarray}\label{FL-opt}\begin{array}{lll}
 \bx^*:={\rm argmin}_{\bx\in\R^n }~f(\bx).
\end{array}\end{eqnarray}
Since $f_i$ is bounded from below, we have
\begin{eqnarray}\label{FL-opt-lower-bound}\begin{array}{lll}
 f^*:=f(\bx^*)>-\infty.
\end{array}\end{eqnarray}
 By introducing auxiliary variables, $\bx_i=\bx$,  problem  \eqref{FL-opt} can be rewritten as
\begin{eqnarray}\label{FL-opt-ver1}\begin{array}{lll}
 \underset{ \bx, \bx_1, \ldots,\bx_m\in\R^n}{\min}~~  \sum_{i=1}^{m}  w_i f_i(\bx_i),~~{\rm s.t.}~~\bx_i=\bx,~i\in[m].
\end{array}\end{eqnarray}
Throughout the paper, we shall place our interest on the above optimization problem. For simplicity, we also denote
\begin{eqnarray}\label{FX}
\begin{array}{lll}
F(X):=  \sum_{i=1}^{m} w_i f_i(\bx_i), \qquad{\rm where}\qquad X:=(\bx_1, \bx_2,\ldots,\bx_m).
\end{array}\end{eqnarray}
It is easy to see that $f(\bx)=F(X)$ if $X=(\bx,\bx,\ldots,\bx)$.

\subsection{Conventional federated learning}
The conventional federated learning  can be summarized as   Algorithm \ref{algorithm-FL}.
\begin{algorithm}[!th]
\SetAlgoLined
{Initialize}  $\bx_i^0=\bx^0,i\in[m]$, a step size $\gamma >0$. Set $k \Leftarrow 0$.
 
\For{$k=0,1,2,\ldots$}{

{\it Weights upload:}  Each  client sends the parameter $\bx^{k}_i$ to the central server. 

{\it Global aggregation:} The central server calculates the average parameter  $\bx^{k+1}$  by 
\begin{eqnarray}\label{FL-global-aggregation}\begin{array}{lll}
\bx^{k+1} = \sum_{i=1}^{m} {w_i \bx^{k}_i}.
\end{array}\end{eqnarray} 
	  
{\it  Weights feedback:} The central server broadcasts the  parameter $\bx^{k+1}$ to  every local  client.

\For{$i=1,2,\ldots,m$}{{\it Local update:}  Each client updates its parameter  locally and in parallel by 
\begin{eqnarray}\label{FL-local-update}\begin{array}{lll} 
\bx^{k+1}_i = \bx^{k+1} - \gamma \nabla f_i( \bx^{k+1}).
\end{array}\end{eqnarray}   
}}
\caption{Federated Learning \label{algorithm-FL}}
\end{algorithm}
Since we initialize $\bx_i^0=\bx^0,i\in[m]$ and  $\sum_{i=1}^{m} w_i=1$, the first step has $ \bx^{1} = \sum_{i=1}^{m} {w_i \bx^{0}_i} = \bx^0$. Therefore, \begin{eqnarray}\label{x-i-1}
\begin{array}{lll} 
\bx^{1}_i = \bx^{1} -  \gamma \nabla f_i( \bx^{1}) = \bx^0 - \gamma \ \nabla f_i( \bx^0) = \bx^0_i - \gamma \nabla f_i( \bx^0_i).
\end{array}\end{eqnarray}  
So the framework of Algorithm \ref{algorithm-FL} is same as the standard one where the first step does the local update in \eqref{x-i-1}, followed by the global aggregation. We prefer the framework of Algorithm \ref{algorithm-FL} because it has a clearer and closer link to the framework of ADMM  introduced in the sequel.
\subsection{ADMM}
We first briefly introduce some backgrounds of the alternating direction method of multipliers (ADMM). For more details, one can refer to the earliest work \cite{gabay1976dual} and a nice book \cite{boyd2011distributed}. Suppose we are given an optimization problem,
\begin{eqnarray} \label{ADMM-opt-prob}
\begin{array}{cll}
\min\limits_{\bx\in\R^n,\bz\in\R^q}~~f (\bx)+  g(\bz),\qquad
{\rm s.t.}~~ A\bx+B\bz=\bb, 
\end{array}\end{eqnarray}
where $A\in\R^{p\times n}$,  $B\in\R^{p\times q}$, and   $\bb\in\R^{p}$.
To implement ADMM, we need the so-called augmented Lagrange function of  problem  \eqref{ADMM-opt-prob}, which is defined by
\begin{eqnarray}\label{ADMM-opt-ver11}
\begin{array}{lll}
\L(\bx,\bz ,\bpi) := f (\bx)+  g(\bz) + \langle A\bx+B\bz-\bb, \bpi \rangle + \frac{\sigma}{2}\|A\bx+B\bz-\bb\|^2,
\end{array}\end{eqnarray}
where $\bpi$ is known for the Lagrange multiplier and $\sigma>0$. Based on the augmented Lagrange function, ADMM  executes the following steps for a given starting point $(\bx^0, \bz^0,\bpi^0)$ and any $k\geq 0$,
\begin{eqnarray}\label{framework-ADMM-0}
 \arraycolsep=1.4pt\def\arraystretch{1.15}
\begin{array}{llll} 
 \bx^{k+1} &=&  {\rm argmin}_{\bx \in\R^n}~\L(\bx,\bz^k ,\bpi^k), \\ 
  \bz^{k+1} &=& {\rm argmin}_{\bz\in\R^q}~\L(\bx^{k+1},\bz,\bpi^k), \\
  \bpi^{k+1} &=&    \bpi^{k} + \sigma(A\bx^{k+1}+B\bz^{k+1}-\bb).
\end{array} 
\end{eqnarray}
Therefore, to implement ADMM for our problem  in \eqref{FL-opt-ver1},  
the corresponding augmented Lagrange function  can be defined by, 
\begin{eqnarray}\label{FL-opt-ver11}
\begin{array}{lll}
\L(\bx,X ,\Pi) := \sum_{i=1}^{m}   L(\bx,\bx_i ,\bpi_i),
\end{array}\end{eqnarray}
where $X :=(\bx_1,\bx_2,\ldots,\bx_m),~\Pi:=(\bpi_1,\bpi_2,\ldots,\bpi_m)$, and $L(\bx,\bx_i ,\bpi_i)$ is  
\begin{eqnarray} \label{Def-L}
\begin{array}{lll}
L(\bx,\bx_i ,\bpi_i):= w_i f_i(\bx_i)+    \langle \bx_i-\bx, \bpi_i\rangle + \frac{\sigma_i}{2}\|\bx_i-\bx\|^2.
\end{array}\end{eqnarray}
Here, $\bpi_i\in\R^n,i\in[m]$ are the Lagrange multipliers and $\sigma_i>0,i\in[m]$. Similar to  \eqref{framework-ADMM-0}, we have the framework of ADMM for problem \eqref{FL-opt-ver1}. That is,   for an initialized point  $(\bx^0, X^0, \Pi^0)$ and any $k\geq0$, perform the following updates iteratively,
\begin{eqnarray}\label{framework-ADMM}
 \arraycolsep=1.4pt\def\arraystretch{1.15}
\begin{array}{llll} 
 \bx^{k+1} &=&  {\rm argmin}_{\bx \in\R^n}~\L(\bx,X ^k, \Pi^k), \\ 
  \bx^{k+1}_i &=& {\rm argmin}_{\bx_i\in\R^n}~L(\bx^{k+1},\bx_i, \bpi_i^k ),~~& i\in[m], \\
  \bpi^{k+1}_i &=&    \bpi_i^{k} + \sigma_i(\bx_i^{k+1}-\bx^{k+1}),~~& i\in[m] .
\end{array} 
\end{eqnarray}

\begin{algorithm}[!th]
\SetAlgoLined
Initialize $\bx_i^0, \bpi_i^0, \sigma_i>0, i\in[m]$. Set $k \Leftarrow 0$. 

\For{$k=0,1,2,\ldots$}{

{\it Weights upload:}  Each  client sends the parameters $\bx^{k}_i$ and $\bpi_i^{k}$ to the central server. 

 {\it Global aggregation:} The central server calculates the average parameter $\bx^{k+1}$ by   \begin{eqnarray} \label{admm-sub1}
 \begin{array}{llll}
\bx^{k+1} &=&  {\rm argmin}_{\bx}~ \L(\bx,X ^k, \Pi^k).
\end{array}
\end{eqnarray}
{\it Weights feedback:} The central server broadcasts the parameter $\bx^{k+1}$ to every local  client.  

\For{$i=1,2,\ldots,m$}{
{\it Local update:} Each client updates its parameters locally and in parallel  by
\begin{eqnarray*} 
&& \begin{array}{llll}
\bx^{k+1}_i &=&{\rm argmin}_{\bx_i}~L(\by^{k+1},\bx_i, \bpi_i^k ),~~\qquad~~ 
    \end{array}\\ 
&&\begin{array}{llll}   
\bpi^{k+1}_i &=&    \bpi_i^{k} +\sigma_i(\bx_i^{k+1}-\by ^{k+1}).  
\end{array}
\end{eqnarray*} }} 
\caption{ADMM-based federated learning \label{algorithm-FL-ADMM}}
\end{algorithm}

Integrating the standard ADMM into federated learning, we derive the algorithmic framework presented in Algorithm \ref{algorithm-FL-ADMM}, where  subproblem  \eqref{admm-sub1} can be derived by
\begin{eqnarray} \label{admm-sub1-closed}
 \begin{array}{llll}
\bx^{k+1} =  {\rm argmin}_{\bx}~ \L(\bx,X ^k, \Pi^k) =  \sum_{i=1}^{m}  \frac{{\sigma_i}\bx^{k}_i}{\sigma}   +   \sum_{i=1}^{m}   \frac{\bpi_i^k}{\sigma},
\end{array}
\end{eqnarray}  with \begin{eqnarray}\label{sum-sigma-i}\begin{array}{lll}
 \sigma : = \sum_{i=1}^{m}  \sigma_i.
\end{array}\end{eqnarray} 
It is noted that in traditional federated learning,   as shown in (\ref{FL-global-aggregation}), the central server usually calculates the parameter by using the averaged value of all local parameters, $\bx_i^k$.  Besides that,   ADMM also exploits parameters $\bpi_i^k$ from the dual problem, see (\ref{admm-sub1-closed}). If we  set $\sigma_i=w_i/\gamma$ for all $i\in[m]$ and 
 a  given step size $\gamma$, then \eqref{admm-sub1-closed}  turns to
\begin{eqnarray}\label{sub-1-g0-0}\begin{array}{lll}
\bx^{k+1} =  \sum_{i=1}^{m}  w_i \bx^{k}_i +  {\gamma}\sum_{i=1}^{m}  \bpi_i^k. 
\end{array}\end{eqnarray}
It is very similar to (\ref{FL-global-aggregation}) but with an additional term involving the dual parameter, $\bpi_i^k$.

\subsection{Stationary points}
To end this section, we present the optimality conditions of problems \eqref{FL-opt-ver1} and (\ref{FL-opt}).
\begin{definition}\label{def-sta} A point $(\bx^*, X^*,\Pi^*)$ is a stationary point of   problem  (\ref{FL-opt-ver1}) if it satisfies 
\begin{eqnarray}\label{opt-con-FL-opt-ver1}
  \left\{\begin{array}{rcll}
 w_i\nabla  f_i(\bx_i^*)+\bpi_i^* &=&   0, ~~&i\in[m],  \\ 
 \bx_i^*-\bx^* &=&0,&i\in[m],\\ 
  \sum_{i=1}^{m} \bpi_i^* &=& 0.
\end{array} \right.
\end{eqnarray} 
A point $\bx^*$ is  a stationary point of   problem  (\ref{FL-opt}) if it satisfies 
\begin{eqnarray}\label{opt-con-FL-opt}
 \begin{array}{llll}
\nabla  f(\bx^*)=0.
\end{array}
\end{eqnarray} 
\end{definition}
It is not difficult to prove that any locally optimal solution to  problem  \eqref{FL-opt-ver1} (resp. (\ref{FL-opt})) must satisfy \eqref{opt-con-FL-opt-ver1} (resp. \eqref{opt-con-FL-opt}). If  $f_i$ is convex for every $i\in[m]$, then a point  is a globally optimal solution to  problem  \eqref{FL-opt-ver1} (resp.(\ref{FL-opt})) if and only if it satisfies condition \eqref{opt-con-FL-opt-ver1} (resp. \eqref{opt-con-FL-opt}).   Moreover, it is easy to see that a stationary point $(\bx^*, X^*,\Pi^*)$  of   problem  \eqref{FL-opt-ver1} indicates
\begin{eqnarray} \label{grad-x-*=0}
\begin{array}{llll}
\nabla   f(\bx^*) = \sum_{i=1}^{m}   w_i \nabla   f_i(\bx^*) =  \sum_{i=1}^{m}  w_i\nabla    f_i(\bx^*_i)=0.\end{array}
\end{eqnarray} 
That is, $\bx^*$ is also a stationary point of problem  \eqref{FL-opt}.  

%

\section{Communication-Efficient ADMM}\label{sec:ceadmm}

The framework  of ADMM in \eqref{framework-ADMM}  repeats the global aggregation and local update in every step. In federated learning (see Algorithm \ref{algorithm-FL-ADMM}), this manifests that  local clients and the central server have to communicate in every step. That is, the central server broadcasts to weight $\bx^{k+1}$ all local clients, and each client uploads their weights $\bx^{k+1}_i$ and $\bpi^{k+1}_i$  to the central server afterwards. However, frequent communications would come at a huge price, such as extremely long learning time and large amounts of resources, which should be avoided in reality. 

\begin{algorithm}[!th]
\SetAlgoLined
 Initialize $\bx_i^0,\bpi_i^0, \sigma_i>0, i\in[m]$,   an integer $k_0>0$. Set $k \Leftarrow 0$. 

\For{$k=0,1,2,\ldots$}{

\If{$  k\in\K:=\{0,k_0,2k_0,3k_0,\ldots\}$}{
{\it Weights upload:}  Each  client sends  its parameters $\bx^{k}_i$ and $\bpi_i^{k}$ to the central server. 

 {\it Global aggregation:} The central server calculates the average parameter  $\bx^{k+1}$ by \begin{eqnarray}\label{ceadmm-sub1}
 \begin{array}{llll}
\bx^{k+1} 
=  \sum_{i=1}^{m}  \frac{{\sigma_i}\bx^{k}_i}{\sigma}   +   \sum_{i=1}^{m} \frac{\bpi_i^k}{\sigma}.
\end{array}
\end{eqnarray}
{\it Weights feedback:} The central server broadcasts the parameter $\bx ^{k+1}$ to every  client.  
}
\For{$i=1,2,\ldots,m$}{
{{\it Local update:}} By letting $$\by^{k+1} :=\bx^{\tau_k+1},\qquad {\rm where}~~ \tau_k:=\lfloor k/k_0 \rfloor k_0,$$  each client update its parameters locally and in parallel via solving  
\begin{eqnarray} 
\label{ceadmm-sub2}
&& \begin{array}{llll}
\bx^{k+1}_i  
    &=&  {\rm argmin}_{\bx_i}~w_i f_i(\bx_i) +   \langle \bx_i-\by^{k+1}, \bpi_i^{k}\rangle +\frac{\sigma_i}{2}\|\bx_i-\by^{k+1}\|^2,
    \end{array}\\ 
\label{ceadmm-sub3}  
&&\begin{array}{llll}   
\bpi^{k+1}_i &=&    \bpi_i^{k} +\sigma_i(\bx_i^{k+1}-\by ^{k+1}).  
\end{array}
\end{eqnarray} }
}
\caption{{\tt CEADMM}: Communication-efficient ADMM-based federated learning \label{algorithm-CEADMM}}
\end{algorithm}

\subsection{Algorithmic framework}
Therefore, it is inevitable to reduce the communication rounds since they decide the efficiency of the learning process. To proceed with that, we allow local clients to update their parameters a few times and then upload their weights to the central server. In other words, the central server collects parameters from local clients only in some steps \cite{DeepLearning2015,AsynchronousStochastic2017, stich2018local, yu2019parallel, Lin2020Don, wang2021cooperative}. Following this idea, we design a communication-efficient ADMM ({\tt CEADMM}) for federated learning in  Algorithm \ref{algorithm-CEADMM}.

The framework of {\tt CEADMM} indicates that communications only occur  when $k\in\K=\{0,k_0,2k_0,\ldots\},$ where $k_0$ is a predefined positive integer. Therefore, communications rounds (e.g., weights feedback and weights upload)  can be reduced, thereby  saving the cost vastly. 

 For local updates in Algorithm \ref{algorithm-CEADMM}, we introduce an auxiliary point $\by^{k+1} =\bx^{\tau_k+1}$, where $\tau_k=\lfloor k/k_0 \rfloor k_0.$  It is easy to see that $ \tau_k < k< \tau_k+k_0$ if $k\notin\K$ and  $k=\tau_k$  if $k\in\K$.   Because of this, $\by^{k+1}$ has the following updates:
\begin{eqnarray}\label{x-y-relation} 
\by^{k+1}=\left\{ \begin{array}{llll}
  \bx^{k+1},&~~{\rm if}~~&k\in\K,\\ 
   \bx^{\tau_k+1},&~~{\rm if}~~& k\notin\K.
\end{array}\right.
\end{eqnarray}

\subsection{Global convergence}
For notational simplicity, hereafter, we denote
\begin{eqnarray} 
 \label{decreasing-property-0}   
\begin{array}{llllll}
 \bg_i^{k}&:=&w_i\nabla f_i(\bx_i^{k}),~~&\L^k&:=&\L(\by^{k},X^{k},\Pi^{k})
    \end{array}  
 \end{eqnarray} 
 and let   $\ba^{k} \rightarrow \ba$ stand for $\lim_{k\rightarrow\infty} \ba^{k} = \ba$. To establish the convergence properties, we need some assumptions on $f_i, i\in[m]$. 
\begin{assumption}\label{ass-fi} Every $f_i, i\in[m]$ is gradient Lipschitz continuous with a constant $r_i>0$. 
\end{assumption}
 With the help of the above assumption, our first result shows that the whole sequence of three objective function values $\{\L^{k}\}$, $\{ F(X^{k})  \}$, and $\{f (\by^{k})\}$ converge. 

\begin{theorem}\label{global-obj-convergence-exact}   Let $\{(\by^{k},X^{k},\Pi^{k})\}$ be the sequence generated by Algorithm \ref{algorithm-CEADMM} with $\sigma_i> 2w_ir_i$ for every $i\in[m]$. The following results hold under Assumption \ref{ass-fi}.
 \begin{itemize}
 \item[i)] 
  Three  sequences $\{\L^{k}\}$, $\{ F(X^{k})  \}$, and $\{f (\by^{k})\}$ converge to the same value, namely,
   \begin{eqnarray}  \label{L-local-convergence-limit-1}
   \begin{array}{lll}
 {\lim}_{k \rightarrow \infty}  \L^{k} = {\lim}_{k \rightarrow \infty} F(X^{k})  ={\lim}_{k \rightarrow\infty} f(\by^{k}).
    \end{array} 
 \end{eqnarray} 
 \item[ii)]  $\nabla F(X^{k})$ and $\nabla f(\by^{k})$ eventually vanish, namely, 
    \begin{eqnarray}  \label{L-local-convergence-limit-grad}
   \begin{array}{lll}
 {\lim}_{k \rightarrow \infty}\nabla F(X^{k})  ={\lim}_{k \rightarrow\infty} \nabla f(\by^{k}) =0.
    \end{array} 
 \end{eqnarray} 
 \end{itemize}
 \end{theorem}  
 
  Theorem \ref{global-obj-convergence-exact}  states that the objective function values converge. In the below theorem, we would like to see the convergence
performance of sequence $\{(\by^{k},X^{k},\Pi^{k})\}$ itself. To proceed with that, we need the assumption on the boundedness of the following level set    \begin{eqnarray} \label{level-set-S} 
  \begin{array}{l}
  \S(\alpha):=\{\bx\in\R^n: f(\bx)\leq \alpha\}
     \end{array}
  \end{eqnarray} 
   for a given $\alpha>0$.  It is worth mentioning that the boundedness of the level set is frequently used in establishing the convergence property of optimization algorithms. There are many functions satisfying this condition, such as the coercive functions\footnote{A continuous function $f:\R^n\mapsto \R$  is coercive if $ f(\bx)\rightarrow+\infty$ when $\|\bx\|\rightarrow +\infty$.}.
  
  \begin{theorem}\label{global-convergence-exact}   Let $\{(\by^{k},X^{k},\Pi^{k})\}$ be the sequence generated by Algorithm \ref{algorithm-CEADMM} with $\sigma_i> 2w_ir_i$ for every $i\in[m]$. The following results hold under Assumption \ref{ass-fi} and the boundedness of $\S(\L^0)$.
  \begin{itemize}
\item[i)]  Then the sequence $\{(\by^{k},X^{k},\Pi^{k})\}$ is bounded, and any its accumulating point $(\by^{\infty},X^{\infty},\Pi^{\infty})$ is a stationary point of  problem  (\ref{FL-opt-ver1}), where $\by^{\infty}$ is a stationary point of  problem  (\ref{FL-opt}).
\item[ii)] If further assume that $\by^{\infty}$ is isolated, then the whole sequence $\{(\by^{k},X^{k},\Pi^{k})\}$ converges to $(\by^{\infty},X^{\infty},\Pi^{\infty})$. 
  \end{itemize}
   \end{theorem}

One can see that if $f$ is locally strongly convex at $\by^{\infty}$, then $\by^{\infty}$ is unique and hence is isolated. However, being isolated is a weaker assumption than locally strong convexity. It is worth mentioning that the establishment of Theorem \ref{global-convergence-exact} does not require the convexity of $f_i$ or $f$, because of this, the sequence is guaranteed to converge to the stationary point of problems (\ref{FL-opt-ver1}) and (\ref{FL-opt}). In this regard, if we further assume the convexity of $f$, then the sequence is capable of converging to the optimal solution to problems (\ref{FL-opt-ver1}) and (\ref{FL-opt}), which is stated by the following corollary.
 
\begin{corollary}\label{L-global-convergence}Let $\{(\by^{k},X^{k},\Pi^{k})\}$ be the sequence generated by Algorithm \ref{algorithm-CEADMM} with $\sigma_i>2w_ir_i$ for every $i\in[m]$.    The following results hold under Assumption \ref{ass-fi}, the boundedness of $\S(\L^0)$, and the convexity of $f$.
\begin{itemize}
\item[i)] Three  sequences $\{\L^{k}\}$, $\{ F(X^{k})\}$, and $\{ f (\by^{k})\}$ converge to the optimal function value of  problem   (\ref{FL-opt}), namely
 \begin{eqnarray}  \label{L-global-convergence-limit}
   \begin{array}{lll}
 {\lim}_{k \rightarrow \infty}  \L^{k} = {\lim}_{k \rightarrow \infty}  F(X^{k}) = {\lim}_{k \rightarrow\infty} f(\by^{k}) =  f^*.
    \end{array} 
 \end{eqnarray}   
  \item[ii)] Any accumulating point  $(\by^{\infty},X^{\infty},\Pi^{\infty})$ of   sequence  $\{(\by^{k},X^{k},\Pi^{k})\}$ is an optimal solution to  (\ref{FL-opt-ver1}), where $\by^{\infty}$ is an optimal solution to   (\ref{FL-opt}). 
 
\item [iii)]  If further assume $f$ is strongly convex. Then whole sequence  $\{(\by^{k},X^{k},\Pi^{k})\}$ converges the unique optimal solution $(\bx^*,X^*,\Pi^*)$ to  (\ref{FL-opt-ver1}), where $\bx^*$ is the unique optimal solution to  (\ref{FL-opt}).  
\end{itemize} 
\end{corollary} 
  \begin{remark}Regarding the assumption in Corollary \ref{L-global-convergence}, we note that $f$ being strongly convex does not require that every $f_i,i\in[m]$ is strongly convex. If one of $f_i$s is strongly convex and the remaining is convex, then $f=\sum_{i=1}^{m} w_if_i$ is strongly convex. Therefore, the strong convexity of $f$  is not a very strict assumption.  Moreover, the strongly convexity suffices to the boundedness of level set  $\S(\alpha)$ for any $\alpha$. Therefore, under the strongly convexity, the assumption on the boundedness of $\S(\L^0)$ can be exempted.
  \end{remark}
 \subsection{Complexity analysis}
 In this subsection, we investigate the convergence speed of the proposed  Algorithm \ref{algorithm-CEADMM}. The following result states that the minimal value among $\| \nabla  F(X^{j})\|^2,\|\nabla f (\by^{j})\|^2, j\in[k]$  vanishes with a  rate $O(k_0/k)$.
 \begin{theorem}\label{complexity-thorem-gradient}   Let $\{(\by^{k},X^{k},\Pi^{k})\}$ be the sequence generated by Algorithm \ref{algorithm-CEADMM} with $\sigma_i>2w_ir_i$ for every  $i\in[m]$. If Assumption \ref{ass-fi} holds, then  it follows
       \begin{eqnarray*}
  \begin{array}{lllll}
   \underset{j=1,2,\ldots,k}{\min}\max\{\| \nabla  F(X^{j})\|^2, \| \nabla f (\by^{j})\|^2\}
 \leq  \frac{\rho k_0 }{k} (\L^0 -f^*),
   \end{array}
  \end{eqnarray*} 
  where $\rho:=\max_{i\in[m]} {8m\sigma_i^2}/{\theta_i}$ with $\theta_i$ being given by (\ref{def-ci}).  
 \end{theorem}
 We would like to point out that the establishment of such a convergence rate requires nothing but the assumption of gradient Lipschitz continuity, namely, Assumption \ref{ass-fi}. 

\section{Inexact CEADMM}\label{sec:iceadmm}
From Algorithm \ref{algorithm-CEADMM}, each client $i$ needs to calculate two parameters $\bx^{k+1}_i$ and $ \bpi^{k+1}_i$ after receiving global parameter $\by^{k+1}$. The latter parameter can be calculated directly by \eqref{ceadmm-sub3}  while the former is obtained by solving  problem  \eqref{ceadmm-sub2},  which generally does not admit a closed-form solution, thereby leading to expensive computational cost. To accelerate the computation for local clients,  many strategies aim to solve subproblem  \eqref{ceadmm-sub2} approximately.
\subsection{Inexact updates}

A common approach to find an approximate solution to \eqref{ceadmm-sub2} takes advantage of the second-order Taylor-like expansion. More precisely, expand $f_i$ at point $\bz_i^k$ near $\bx_i$  by
\begin{eqnarray} 
\label{f-i-h-i} 
 \begin{array}{lll} 
 h_i(\bx_i;\bz_i^k,H_i):= f_i(\bz_i^k) +  \langle \nabla f_i(\bz_i^k), \bx_i-\bz_i^k\rangle + \frac{1}{2}  \| \bx_i-\bz_i^k\|^2_{H_i}. 
    \end{array}
\end{eqnarray}
Then \eqref{ceadmm-sub2} can be addressed approximately by
\begin{eqnarray} 
\label{framework-ADMM-sub-2-tylor} 
 \begin{array}{lll}
  \bx^{k+1}_i 
    &=&  {\rm argmin}_{\bx_i}~w_i  h_i(\bx_i;\bz_i^k,H_i )  + \langle \bx_i-\by^{k+1}, \bpi_i^{k}\rangle +\frac{\sigma_i}{2}\|\bx_i-\bx^{k+1}\|^2\\[1.15ex] 
    &=&\bz^{k}_i - (w_iH_i  + \sigma_i I)^{-1} \Big[\sigma_i( \bz^{k}_i-\by^{k+1})+w_i \nabla f_i(\bz_i^k) + \bpi_i^{k}\Big].
    \end{array}
\end{eqnarray}
Here, $H_i\succeq0 $ can be chosen to  satisfy  $H_i \approx \nabla^2 f_i $. If $f_i$ is  gradient Lipschitz continuous with a constant $r_{i}>0$, then  $H_i$ can be chosen as $H_i \approx r_iI  $. 
For local point $\bz_i^k$, we have two potential candidates: previous local parameter $\bx_i^k$ or  updated parameter $\bx^{k+1}$ from the central server. 

\textit{Choice 1}: If $\bz_i^k=\bx_i^k$, then \eqref{framework-ADMM-sub-2-tylor} turns to
 \begin{eqnarray*} 
 \begin{array}{lll}
  \bx^{k+1}_i  
    &=&\bx_i^{k}- (w_iH_i  + \sigma_i I)^{-1}  \Big[\sigma_i( \bx_i^{k}-\by^{k+1})+\bg_i^{k} + \bpi_i^{k}\Big].
    \end{array}
\end{eqnarray*}

\textit{Choice 2:}  If $\bz_i^k=\by^{k+1}$, then \eqref{framework-ADMM-sub-2-tylor} becomes
\begin{eqnarray} 
\label{framework-ADMM-sub-2-xk1} 
 \begin{array}{lll}
  \bx^{k+1}_i 
        &=& \bx^{k+1} - (w_iH_i + \sigma_i I)^{-1}\Big[ w_i\nabla f_i(\by^{k+1})+ \bpi_i^{k}\Big].
    \end{array}
\end{eqnarray}

  \subsection{Standard linearised inexact  ADMM}

\begin{algorithm}[!th]
\SetAlgoLined

Initialize $\bx_i^0,\bpi_i^0$, a step size  $\gamma>0$. Set $k \Leftarrow 0$. 

\For{$k=0,1,2,\ldots$}{
{\it Weights upload:}  Each client sends its parameters $\bx^{k}_i$ and $\bpi_i^{k}$ to the central server. 

 {\it Global aggregation:} The central server calculates the average parameter  $\bx^{k+1}$ by
  \begin{eqnarray}\label{framework-ADMM-sub-1-T-R}
 \begin{array}{llll}
\bx^{k+1} ~=~ \sum_{i=1}^{m}  w_i \bx^{k}_i +  {\gamma}\sum_{i=1}^{m}  \bpi_i^k .
\end{array}
\end{eqnarray}

{\it Weights feedback:} The central server broadcasts the parameter $\bx^{k+1}$ to every local client.  

\For{$i=1,2,\ldots,m$}{
{\it Local update:} Each  client update its parameters locally and in parallel by
\begin{eqnarray} 
\label{framework-ADMM-sub-2-T-R}
\hspace{-8mm}&& 
\begin{array}{lll}
\bx^{k+1}_i  
    &=&  \bx^{k+1} - \gamma \nabla f_i(\bx^{k+1}) - \frac{\gamma}{w_i}\bpi_i^{k},~~~~
    \end{array}\\ 
\label{framework-ADMM-sub-3-T-R}  
\hspace{-8mm}&&
\begin{array}{lll}   
\bpi^{k+1}_i &=&    \bpi_i^{k} + \frac{w_i}{\gamma} (\bx_i^{k+1}-\bx^{k+1}).  
\end{array}
\end{eqnarray}}}
\caption{{\tt LIADMM}: Linearised inexact ADMM-based federated learning \label{algorithm-FL-riADMM}}
\end{algorithm}

To compare with the entire framework of the federated learning described in Algorithm \ref{algorithm-FL}, we focus on  the following settings for Algorithm \ref{algorithm-CEADMM}: 
\begin{itemize}

\item $k_0=1$; This means $\by^{k+1}=\bx^{k+1}$. 
\item $\sigma_i=w_i/\gamma, i\in[m]$ for a given step size $\gamma$. Since $\sum_{i=1}^{m} w_i=1$, we have $\sum_{i=1}^{m} \sigma_i=1/\gamma$. This turns  update \eqref{ceadmm-sub1} to   update \eqref{framework-ADMM-sub-1-T-R}.  
\item By exploiting choice 2 (namely, $\bz_i^k=\bx^{k+1}$) and $H_i =0$ (namely, using the linearisation of $f_i$),  subproblem  \eqref{framework-ADMM-sub-2-xk1}  becomes \eqref{framework-ADMM-sub-2-T-R} due to $\sigma_i=w_i/\gamma$, $\by^{k+1}=\bx^{k+1}$, and 
  \begin{eqnarray*}  
 \begin{array}{lll} 
  \bx^{k+1}_i 
= \bx^{k+1} -  \frac{w_i}{\sigma_i }\nabla f_i(\by^{k+1}) -\frac{\bpi_i^{k}}{\sigma_i }  = \bx^{k+1} - \gamma \nabla f_i(\bx^{k+1}) - \frac{\gamma}{w_i}\bpi_i^{k}.
 \end{array}\end{eqnarray*}
\end{itemize}
 
Based on these settings, we derive the framework of standard linearised inexact ADMM  in Algorithm \ref{algorithm-FL-riADMM}.  In comparison with \eqref{FL-global-aggregation}  and  \eqref{FL-local-update}  in Algorithm \ref{algorithm-FL},  both \eqref{framework-ADMM-sub-1-T-R}  and \eqref{framework-ADMM-sub-2-T-R} in Algorithm \ref{algorithm-FL-riADMM} have an additional term  associated with the dual parameters. In this regard, the framework of the conventional federated learning falls into a special case of the linearised inexact ADMM ({\tt LIADMM}). Their similarities and dissimilarities have been shown in Figure \ref{fig:FL-ADMM}.

\subsection{Inexact communication-efficient ADMM}
\begin{algorithm}[!th]
\SetAlgoLined
 Initialize $\bx_i^0,\bpi_i^0, \sigma_i>0, H_i\succeq0, i\in[m]$ and an integer $k_0>0$. Set $k \Leftarrow 0$. 

\For{$k=0,1,2,\ldots$}{

\If{$  k\in\K:=\{0,k_0,2k_0,3k_0,\ldots\}$}{
{\it Weights upload:}  Each  client sends its parameters  $\bx^{k}_i$ and $\bpi_i^{k}$ to the central server. 

 {\it Global aggregation:} The central server calculates the average parameter $\bx^{k+1}$ by \begin{eqnarray}\label{iceadmm-sub1}
 \begin{array}{llll}
\bx^{k+1} =    \sum_{i=1}^{m}  \frac{{\sigma_i}\bx^{k}_i}{\sigma}   +   \sum_{i=1}^{m} \frac{\bpi_i^k}{\sigma}.
\end{array}
\end{eqnarray}
{\it Weights feedback:} The central server broadcasts the parameter $\bx ^{k+1}$ to every   client.  
}
\For{$i= 1,2,\ldots,m$}{
{{\it Local update:}} By letting $$\by^{k+1} :=\bx^{\tau_k+1},\qquad {\rm where}~~ \tau_k=\lfloor k/k_0 \rfloor k_0,$$  each client update its parameters locally and in parallel via solving 
\begin{eqnarray} 
\label{iceadmm-sub2}
&& \begin{array}{llll}
\bx^{k+1}_i  
& = &   \bx^{k}_i - (w_iH_i  + \sigma_i I)^{-1}\Big[\sigma_i(\bx^{k}_i-\by^{k+1})+\bg_i^{k}+ \bpi_i^{k}\Big], 
    \end{array}\\ 
\label{iceadmm-sub3}  
&&\begin{array}{llll}   
\bpi^{k+1}_i &=&    \bpi_i^{k} +\sigma_i(\bx_i^{k+1}-\by ^{k+1}).  
\end{array}
\end{eqnarray} }
}
\caption{{\tt ICEADMM}: Inexact communication-efficient ADMM-based federated learning \label{algorithm-ICEADMM}}
\end{algorithm}

Algorithm \ref{algorithm-FL-riADMM} focuses on $k_0=1$ which is not communication-efficient.  Therefore, following the idea of Algorithm \ref{algorithm-CEADMM}, we set $k_0>1$.  Moreover,   different with Algorithm \ref{algorithm-FL-riADMM} that exploits choice 2 in \eqref{framework-ADMM-sub-2-T-R}, we take advantage of choice 1 (i.e., $\bz_i^k=\bx_i^{k}$) to expand $f_i$ since the approximation function, $h_i(\bx_i^{k+1};\bx_i^k,H_i )$, would be closer to $h_i(\bx_i^{k};\bx_i^k,H_i )=f_i(\bx_i^{k})$ than  $h_i(\bx^{k+1};\bx_i^k,H_i )$ when $k_0>1$. Overall, instead of solving subproblem  \eqref{ceadmm-sub2} to update $\bx_i^{k+1}$, we address the  problem
\begin{eqnarray} 
\label{iceadmm-sub2-0}
 \arraycolsep=1.4pt\def\arraystretch{1.25}
\begin{array}{llll}
\bx^{k+1}_i  
& = &  {\rm argmin}_{\bx_i}~w_i  h_i(\bx_i;\bx_i^k,H_i )  + \langle \bx_i-\by^{k+1}, \bpi_i^{k}\rangle +\frac{\sigma_i}{2}\|\bx_i-\by^{k+1}\|^2\\ 
 &=&  {\rm argmin}_{\bx_i}~  \langle w_i\nabla f_i(\bx_i^k) + \bpi_i^{k}, \bx_i\rangle  + \frac{w_i}{2}  \| \bx_i-\bx_i^k\|^2_{H_i} + \frac{\sigma_i}{2}\|\bx_i-\by^{k+1}\|^2\\ 
&=&\bx^{k}_i - (w_iH_i  + \sigma_i I)^{-1} \Big[\sigma_i(\bx^{k}_i-\by^{k+1})+\bg_i^{k}+ \bpi_i^{k}\Big]. 
    \end{array}  
\end{eqnarray} 
To summarize, the framework of the inexact CEADMM is presented in Algorithm \ref{algorithm-ICEADMM}.

We would like to emphasize the computational complexity of {\tt ICEADMM} is much lower than {\tt CEADMM} since subproblem  \eqref{iceadmm-sub2} can be solved more efficiently than \eqref{ceadmm-sub2}. For example, if $H_i$ is chosen as $r_i I$, then the major computation is calculating  $\bg_i^k$, which is quite cheap in comparison with addressing \eqref{ceadmm-sub2}. Therefore, {\tt ICEADMM} alleviates the computational burdens for local clients dramatically.
\subsection{Global convergence}

To establish the convergence property for Algorithm \ref{algorithm-ICEADMM}, suppose every $f_i,i\in[m]$ is gradient Lipschitz continuous with a constant  $r_i>0$.  Then there exists a $\Theta_i$ such that $r_iI\succeq \Theta_i \succeq 0$ and
\begin{eqnarray} \label{grad-lip-theta}
 \begin{array}{l}
f_i(\bx_i)\leq   f_i(\bz_i )+\langle \nabla  f_i(\bz_i ), \bx_i-\bz_i \rangle + \frac{1}{2}\| \bx_i-\bz_i \|^2_{\Theta_i},
\end{array}
\end{eqnarray} 
for any $\bx_i,\bz_i\in\R^n$. The existence is obvious as we at least can choose $\Theta_i=r_iI$.
\begin{theorem}\label{global-obj-convergence-inexact}  Let   $\{(\by^{k},X^{k},\Pi^{k})\}$ be the sequence generated by Algorithm \ref{algorithm-CEADMM} with $H_i=\Theta_i$ and $\sigma_i>3\sqrt{2}w_ir_i$ for every $i\in[m]$. The following results hold under Assumption \ref{ass-fi}.
 \ \begin{itemize}
 \item[i)] 
  Three  sequences $\{\L^{k}\}$, $\{ F(X^{k})  \}$, and $\{f (\by^{k})\}$ converge to the same value, namely,
   \begin{eqnarray}  \label{L-local-convergence-limit-inexact}
   \begin{array}{lll}
 {\lim}_{k \rightarrow \infty}  \L^{k} = {\lim}_{k \rightarrow \infty} F(X^{k})  ={\lim}_{k \rightarrow\infty} f(\by^{k}).
    \end{array} 
 \end{eqnarray} 
 \item[ii)] $\nabla F(X^{k})$ and $\nabla f(\by^{k})$ eventually vanish, namely,  
    \begin{eqnarray}  \label{L-local-convergence-limit-grad-inexact}
   \begin{array}{lll}
 {\lim}_{k \rightarrow \infty}\nabla F(X^{k})  ={\lim}_{k \rightarrow\infty} \nabla f(\by^{k}) =0.
    \end{array} 
 \end{eqnarray} 
 \end{itemize}
 \end{theorem}  
 Theorem \ref{global-obj-convergence-inexact}  states that the objective function values of the sequence $\{(\by^{k},X^{k},\Pi^{k})\}$ converge. In the following theorem, we would like to see the convergence
performance of sequence $\{(\by^{k},X^{k},\Pi^{k})\}$ itself.  The proofs of the two theorems in the sequel are the same as those of Theorem \ref{global-convergence-exact} and Corollary \ref{L-global-convergence}, and hence omitted here
\begin{theorem}\label{global-convergence-inexact}   Let   $\{(\by^{k},X^{k},\Pi^{k})\}$ be the sequence generated by Algorithm \ref{algorithm-CEADMM} with $H_i=\Theta_i$ and $\sigma_i>3\sqrt{2}w_ir_i$ for every $i\in[m]$.  The following results hold under Assumption \ref{ass-fi} and the boundedness of $\S(\varphi^1)$.
 \ \begin{itemize}
 \item[i)]  The sequence, $\{(\by^{k},X^{k},\Pi^{k})\}$, is bounded, and any its accumulating point, $(\by^{\infty},X^{\infty},\Pi^{\infty})$, is a stationary point of  problem  (\ref{FL-opt-ver1}), where $\by^{\infty}$ is a stationary point of  problem  (\ref{FL-opt}). 
 \item[ii)] If further assuming that $\by^{\infty}$ is isolated, then the whole sequence, $\{(\by^{k},X^{k},\Pi^{k})\}$, converges to $(\by^{\infty},X^{\infty},\Pi^{\infty})$. 
 \end{itemize}
 \end{theorem}  
  Similar to Corollary \ref{L-global-convergence}, under the convexity,  Algorithm \ref{algorithm-ICEADMM} is capable of achieving the optimal parameter.
\begin{corollary}\label{L-global-convergence-inexact}  If Assumption \ref{ass-fi} holds, $\S(\varphi^1)$ is bounded, and $f$ is convex, then all the results in Corollary \ref{L-global-convergence}  hold for the sequence, $\{(\by^{k},X^{k},\Pi^{k})\}$, generated by Algorithm \ref{algorithm-ICEADMM} with $H_i=\Theta_i$ and $\sigma_i>3\sqrt{2}w_ir_i$ for every $i\in[m]$. 
\end{corollary}
 \subsection{Complexity analysis}
 Finally, we would like to see how fast the proposed  Algorithm \ref{algorithm-ICEADMM} converges. Similar to Theorem \ref{complexity-thorem-gradient}, the following result states that the minimal value among $\| \nabla  F(X^{j+k_0})\|^2$ and $\|\nabla f (\by^{j+k_0})\|^2, j=1,2,\ldots,k$ vanishes with a  rate $O(k_0/k)$.  
 \begin{theorem} \label{complexity-thorem-gradient-inexact} Let  $\{(\by^{k},X^{k},\Pi^{k})\}$ be the sequence generated by Algorithm \ref{algorithm-CEADMM} with $H_i=\Theta_i$ and $\sigma_i>3\sqrt{2}w_ir_i$ for every $i\in[m]$. If Assumption \ref{ass-fi} holds, then  it follows
       \begin{eqnarray*}
  \begin{array}{lllll}
   \underset{j=1,2,\ldots,k}{\min}&\max&\{\| \nabla F (X^{j+k_0})\|^2,\| \nabla f (\by^{j+k_0})\|^2\} \leq   \frac{\varrho k_0}{k} (\varphi^1 -f^*),
   \end{array}
  \end{eqnarray*} 
  where $\varrho:=\max_{i\in[m]}{12m\sigma_i^2}/{\vartheta_i}$,  $\varphi^1$ and $\vartheta_i$ are given by (\ref{def-ci-1}).  
 \end{theorem}

\section{Numerical Experiments}\label{sec:num}
This section conducts some numerical experiments to demonstrate the performance of the proposed methods {\tt CEADMM} in Algorithm \ref{algorithm-CEADMM} and {\tt ICEADMM} in Algorithm \ref{algorithm-ICEADMM}. MTTLAB code for both algorithms  is available at \url{https://github.com/ShenglongZhou/ICEADMM}.  All numerical experiments are implemented through MATLAB (R2019a) on a laptop with 32GB memory and Inter(R)
Core(TM) i9-9880H 2.3Ghz CPU.  We point out that when $k_0=1$, {\tt CEADMM} and {\tt ICEADMM} are reduced to the standard {\tt ADMM}  and inexact ADMM ({\tt IADMM}). Therefore,  {\tt ADMM} and  {\tt IADMM} can be used as  baselines.

 \subsection{Testing example}
 We take the linear regression and the logistic regression as examples to demonstrate the performance of the two proposed algorithms. Both objective functions are  gradient Lipschitz continuous.
\begin{example}[Linear regression]\label{ex-linear} For this problem, local clients have their objective functions as (\ref{least-squares}).   
We randomly divide $m$ clients into three groups with each group having $m/3$ clients.  Then entries of  $A_i:=[\ba^i_1,\ba^i_2,\ldots,\ba^i_{d_i}]^\top$ and $\bb_i:=[b^i_1,b^i_2,\ldots,b^i_{d_i}]^\top$ from three groups are generated from the standard normal distribution, the Student's $t$ distribution with degree $5$, and the uniform distribution in $[-5,5]$, respectively. The data size of each client, $d_i$, is randomly chosen from  $[50,150]$. Therefore, for each instance, we have dimensions $(m,n,d_1,\ldots,d_m)$ to be decided. For simplicity, we fix $n=100$  but choose $m\in\{30,60,90,120,150\}$ and $d_i\in[50,150]$. It is easy to see that $f_i$ in (\ref{least-squares}) is  gradient Lipschitz continuous with a constant $r_i=\lambda_{\max}(A_i^\top A_i)$, the maximum eigenvalue of $A_i^\top A_i$. 
\end{example}
\begin{example}[Logistic regression]\label{ex-logist} For this problem,  local clients have their objective functions as (\ref{logist-loss}), 
where $\mu=0.01$ in our  numerical experiments.  We use two real datasets described in Table \ref{tab:datasets}  to generate $\ba^i_j $ and $b^i_j$. In particular,  we split $d$ samples into $m$ groups corresponding to $m$ clients. For the first $(m-1)$ clients, we randomly pick $d_i=\lfloor {d}/{m}\rfloor$ samples from $d$ samples, and assign the remaining $d-(d_1+\cdots+d_{m-1})$ samples to the $m$th client. In  the sequel, we choose $m\in\{100,150,200,250,300\}$. It has shown in \cite[Lemma 4]{wang2019greedy} that $f_i$ defined by (\ref{logist-loss}) is the gradient Lipschitz continuous with a constant $r'_i=\lambda_{\max}(A_i^\top A_i)/4+\mu$, where $A_i=[\ba^i_1,\ba^i_2,\ldots,\ba^i_{d_i}]^\top$.

\begin{table}[H]
	\renewcommand{\arraystretch}{1}\addtolength{\tabcolsep}{0pt}
	\caption{Descriptions of  two real datasets.}\vspace{-5mm}
	\label{tab:datasets}
	\begin{center}
		\begin{tabular}{lllrrrr }
			\hline
Data&Datasets&	Source	&	$n$	&	$d$\\\hline
\texttt{qot}&	Qsar oral toxicity	&	uci	&	1024 	&	8992 	\\
 \texttt{sct}&	Santander customer transaction	&	kaggle	&	200 	&	200000 	\\
\hline
 		\end{tabular}
	\end{center}
\end{table}
\end{example}
\subsection{Implementations}
For the stopping criteria of the two algorithms: {\tt CEADMM} in Algorithm \ref{algorithm-CEADMM} and {\tt ICEADMM} in Algorithm \ref{algorithm-ICEADMM},  we terminate them if  the maximum number of iterations  exceeds $10^4$ or their generated point, $(\by^{k},X^{k},\Pi^{k})$, is almost a stationary point to  problem  \eqref{FL-opt-ver1}. To measure the closeness of a point to a stationary point, we check condition \eqref{opt-con-FL-opt-ver1} by
\begin{eqnarray*} 
\begin{array}{r}
\max\left\{ \sum_{i=1}^{m} \|\bg_i^{k}+\bpi_i^k\|^2, 
 \sum_{i=1}^{m} \|\bx_i^k-\by^k\|^2, 
 \| \sum_{i=1}^{m} \bpi_i^k\|^2\right\} 
 \leq \sqrt{nd}10^{-7}.
\end{array} 
\end{eqnarray*}
For the settings of parameters, as aforementioned, $w_i = d_i/d, i\in[m]$ with  $d=\sum_{i=1}^{m}  d_i$. Theorems \ref{L-global-convergence} and  \ref{global-convergence-inexact}  suggest that $\sigma_i$ should be chosen to satisfy $\sigma_i = O(w_ir_i), i\in[m]$. In particular, we set
\begin{eqnarray*} 
\begin{array}{r}\sigma_i :=\sigma_i(a):=\frac{a{\rm ln}(md_i)}{10{\rm ln}(2+k_0)}w_ir_i,~~ a>0.\end{array} 
\end{eqnarray*}
To implement {\tt CEADMM}, we need to solve subproblem  \eqref{ceadmm-sub2}. However, for the logistic regression problem, subproblem  \eqref{ceadmm-sub2} is uneasy to solve exactly. Therefore, we apply {\tt CEADMM} only into solving the linear regression problem, i.e., Example \ref{ex-linear}, where the subproblem can be addressed exactly by 
\begin{eqnarray*}  \label{iceadmm-sub2-21} 
\begin{array}{llll}
\bx^{k+1}_i  
= (w_i A^\top_iA_i + \sigma_i I)^{-1}\Big[w_i A^\top_i\bb_i+\sigma_i\by^{k+1} - \bpi_i^{k}\Big].
    \end{array}  
\end{eqnarray*} 
To implement {\tt ICEADMM}, we need to choose $H_i$. 
For Example \ref{ex-linear}, to accelerate the computation  for the local update in \eqref{iceadmm-sub2}, we let $H_i=r_iI$,
which specifies \eqref{iceadmm-sub2} as
\begin{eqnarray*}  \label{iceadmm-sub2-22}
\begin{array}{llll}
\bx^{k+1}_i  
= \bx^{k}_i- \frac{1}{w_i r_i + \sigma_i} \Big[ \sigma_i(\bx^{k}_i-\by^{k+1})+ \bg_i^{k}+ \bpi_i^{k}\Big]. 
    \end{array}
\end{eqnarray*} 
It is worth mentioning that if we set $H_i=A^\top_iA_i$, then \eqref{iceadmm-sub2-0} is the same as \eqref{ceadmm-sub2}, thereby reducing {\tt ICEADMM}  to {\tt CEADMM}. For Example \ref{ex-logist},  to satisfy $r_iI\succeq H_i$,  we set $H_i=\frac{1}{r}A^\top_iA_i$ with $r>4+\mu=4.01$, which specifies   \eqref{iceadmm-sub2} as
\begin{eqnarray*}  \label{iceadmm-sub2-22} 
\begin{array}{llll}
\bx^{k+1}_i 
&=& \bx^{k}_i - (\frac{w_i}{r} A^\top_iA_i + \sigma_i I)^{-1} \Big[ \sigma_i \bx^{k}_i+\bg_i^{k}-\sigma_i\by^{k+1} + \bpi_i^{k}\Big]. 
    \end{array}  
\end{eqnarray*} 
We summarize parameters for two algorithms in Table \ref{tab:para}. 
 \begin{table}[H]
	\renewcommand{\arraystretch}{1}\addtolength{\tabcolsep}{6pt}
	\caption{Parameters for {\tt CEADMM} and {\tt ICEADMM}.}\vspace{-5mm}
	\label{tab:para}
	\begin{center}
		\begin{tabular}{ll cccc cccc  }
			\hline
	Algs. & Exs. &  $\bx_i^0$ & $\bpi_i^0$		&  $\sigma_i$	 	& $k_0\in$ &	$H_i$	&	 	\\\hline
{\tt CEADMM}	& Ex. \ref{ex-linear} &	0&0 & $\sigma_i(1) $   &$[20]$&	 	 \\		 
{\tt ICEADMM}&Ex. \ref{ex-linear} &	0&0&	$\sigma_i(2) $	  &$[20]$&	$\lambda_{\max}(A_i^\top A_i)I$	 	\\	 
 &Ex. \ref{ex-logist} &	0&0&	$\sigma_i(1) $	  &$[20]$&	$\frac{1}{6}{A^\top_iA_i }$	 	\\		\hline

		\end{tabular}
	\end{center}
\end{table}

\subsection{Numerical results}
In this part, we conduct some simulation to demonstrate the performance of {\tt CEADMM}  and  {\tt ICEADMM}  including  global convergence,   convergence rate, and  effect of $k_0$. To measure the performance, we report the following factors: total number of iterations, total number of the communication rounds  (namely, global aggregations), total computational time (in second), objective function values   $f(\by^k)$ and   $F(X^k)$, and  error measurements $\|\nabla f(\by^{k})\|^2$ and $\|\nabla F(X^{k})\|^2$.
\begin{figure*}[!th]
	 \begin{subfigure}{.33\textwidth}
	\centering
	\includegraphics[width=.95\linewidth]{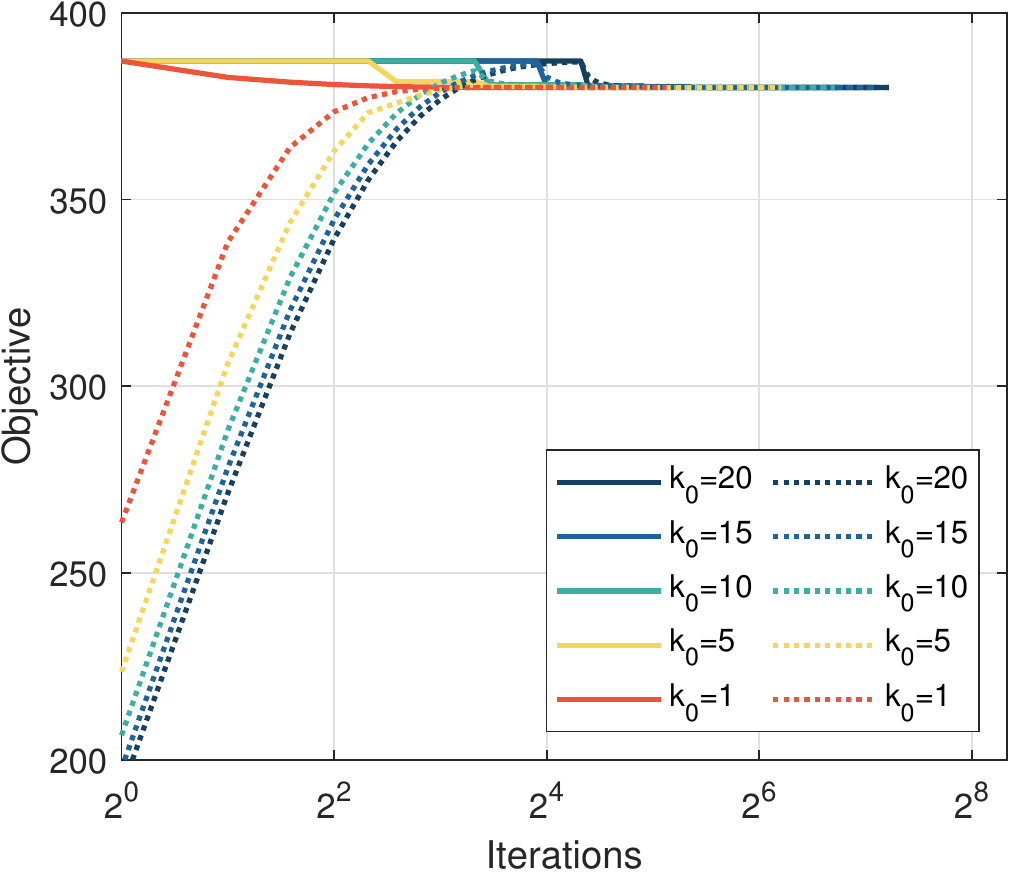}
	\caption{{\tt CEADMM} solving Ex. \ref{ex-linear}}
	\label{fig:CEADMM-diff}
\end{subfigure}	
\begin{subfigure}{.33\textwidth}
	\centering
	\includegraphics[width=.95\linewidth]{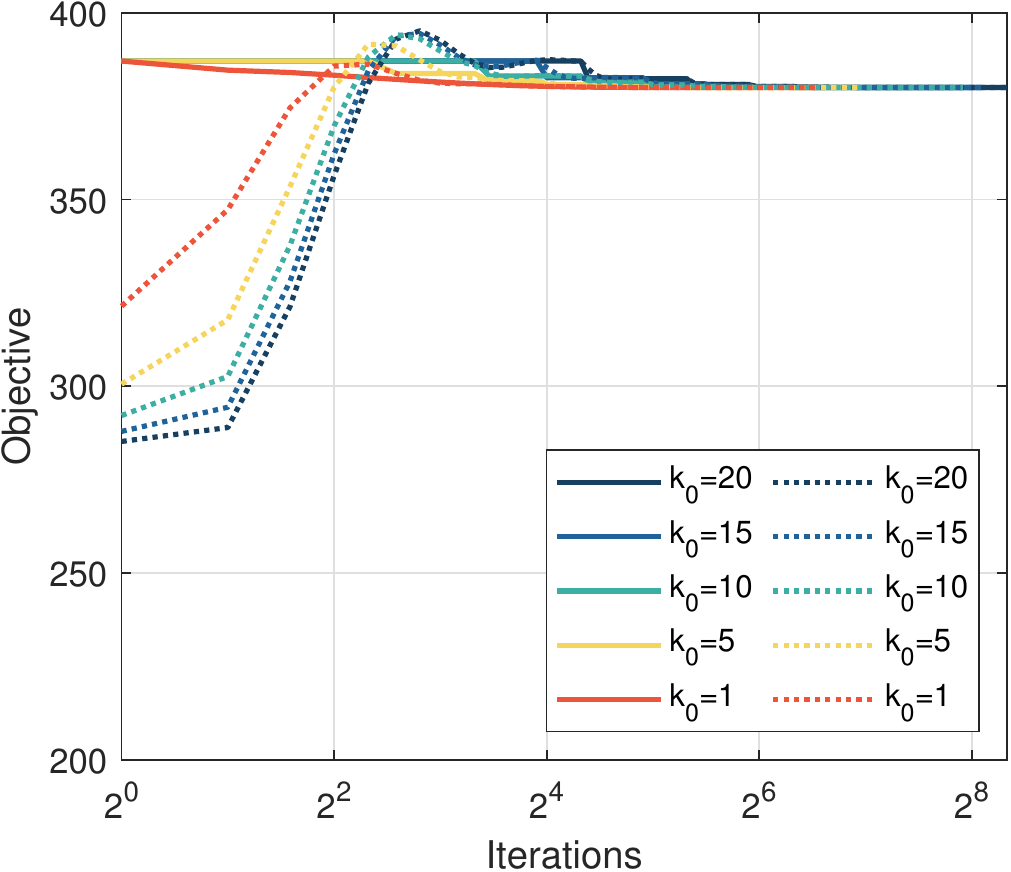}
	\caption{{\tt ICEADMM}  solving Ex. \ref{ex-linear}}
	\label{fig:ICEADMM-diff}
\end{subfigure}   
\begin{subfigure}{.33\textwidth}
	\centering
	\includegraphics[width=.93\linewidth]{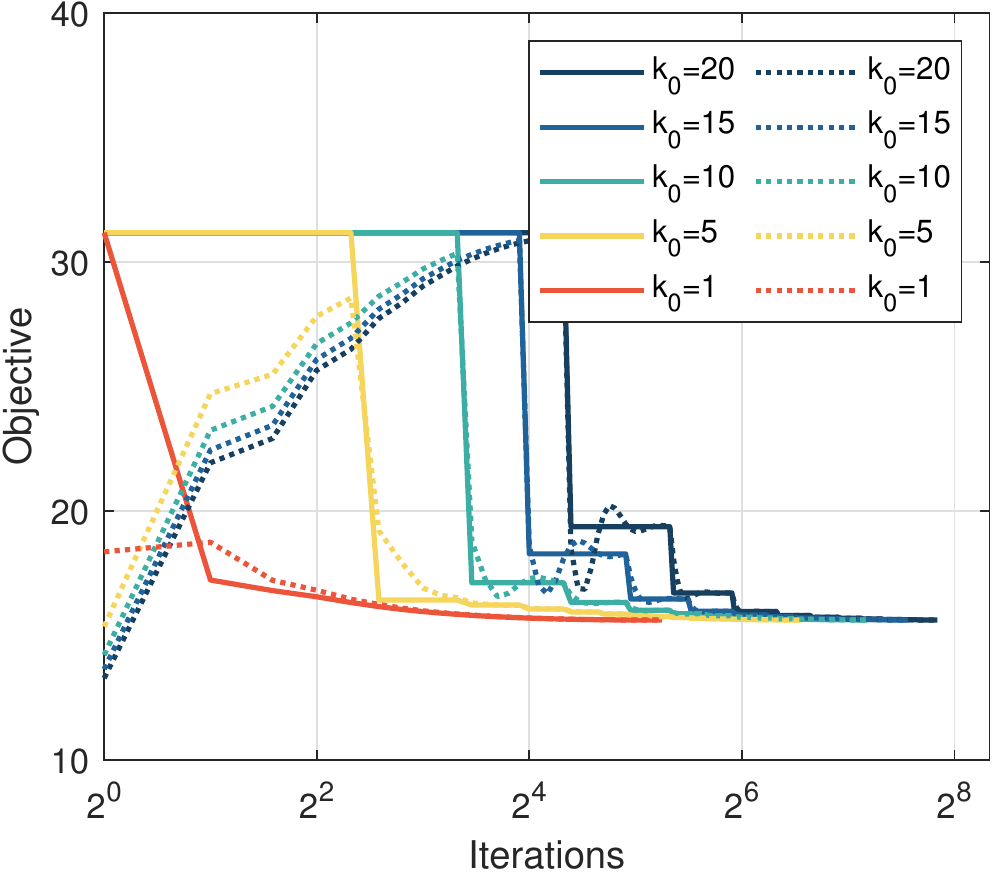}
	\caption{{\tt ICEADMM} solving Ex. \ref{ex-logist}}
	\label{fig:ICEADMM-log-obj-0}
\end{subfigure}	
 
\caption{$f(\by^k)$ (solid lines) and $F(X^k)$ (dashed lines) v.s. iterations.\label{fig:iterations-objective}}
\end{figure*}

\subsubsection{Global convergence}

To see the global convergence proven in Theorems \ref{L-global-convergence} and \ref{global-convergence-inexact},  we show how the objective function values decrease with iterations for   {\tt CEADMM} and {\tt ICEADMM} under five choices of $k_0\in\{1,5,10,15,20\}$. Corresponding results are presented in  Figure \ref{fig:iterations-objective}, where $m=90$ for Example \ref{ex-linear} and $m=200$ for Example \ref{ex-logist} with data {\tt qot}.  
As expected, all lines eventually tend to the same objective function value, well testifying Theorem \ref{L-global-convergence} that the whole sequence converges to the optimal function value of  problem  \eqref{FL-opt}.  It is clear that the bigger values of $k_0>1$ (i.e., the wider gap between two global aggregations) are, the more iterations are required to reach the optimal function value.

\subsubsection{Convergence rate}
To see the convergence speed of two algorithms, as stated in Theorems \ref{complexity-thorem-gradient} and \ref{complexity-thorem-gradient-inexact}, we present two errors, $\|\nabla f(\by^{k})\|^2$ and $\|\nabla F(X^{k})\|^2$,  in Figure \ref{fig:iterations-grad} where $m=90$ for Example \ref{ex-linear} and $m=200$ for Example \ref{ex-logist} with data {\tt sct}. The overall phenomena show that (i) both errors vanish gradually along with the iterations rising, (ii) the big values of $k_0$, the more iterations required by {\tt CEADMM} to converge, which perfectly justifies Theorems \ref{complexity-thorem-gradient} and \ref{complexity-thorem-gradient-inexact} that the convergence rate $O(k_0/k)$ relies on $k_0$. What is more, it is clear that for the same $k_0$, {\tt ICEADMM} takes more iterations than {\tt CEADMM} to converge, such as, 240 v.s. 100 iterations shown in curves of $k_0=10$ in Figures \ref{fig:CEADMM-diff-grad} and \ref{fig:ICEADMM-diff-grad}.
 
\begin{figure}[H]
	\begin{subfigure}{.33\textwidth}
	\centering
	\includegraphics[width=.95\linewidth]{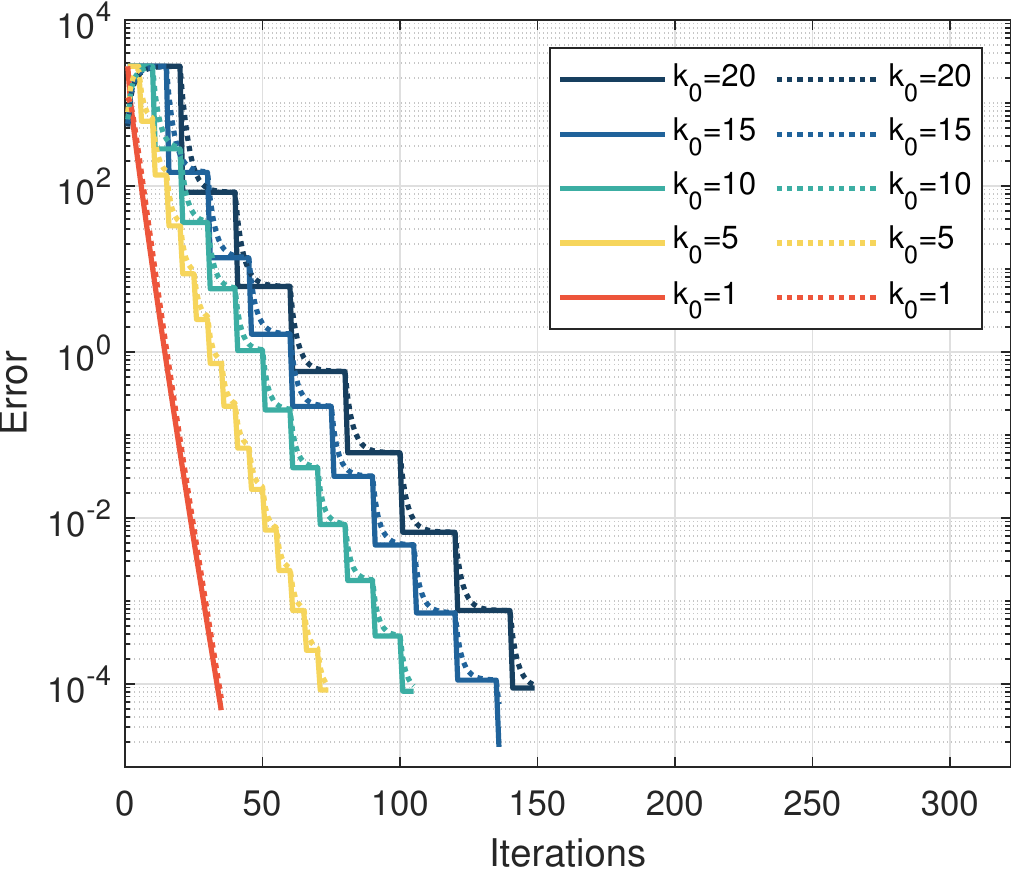}
	\caption{{\tt CEADMM} solving Ex. \ref{ex-linear}}
	\label{fig:CEADMM-diff-grad}
\end{subfigure}	
\begin{subfigure}{.33\textwidth}
	\centering
	\includegraphics[width=.95\linewidth]{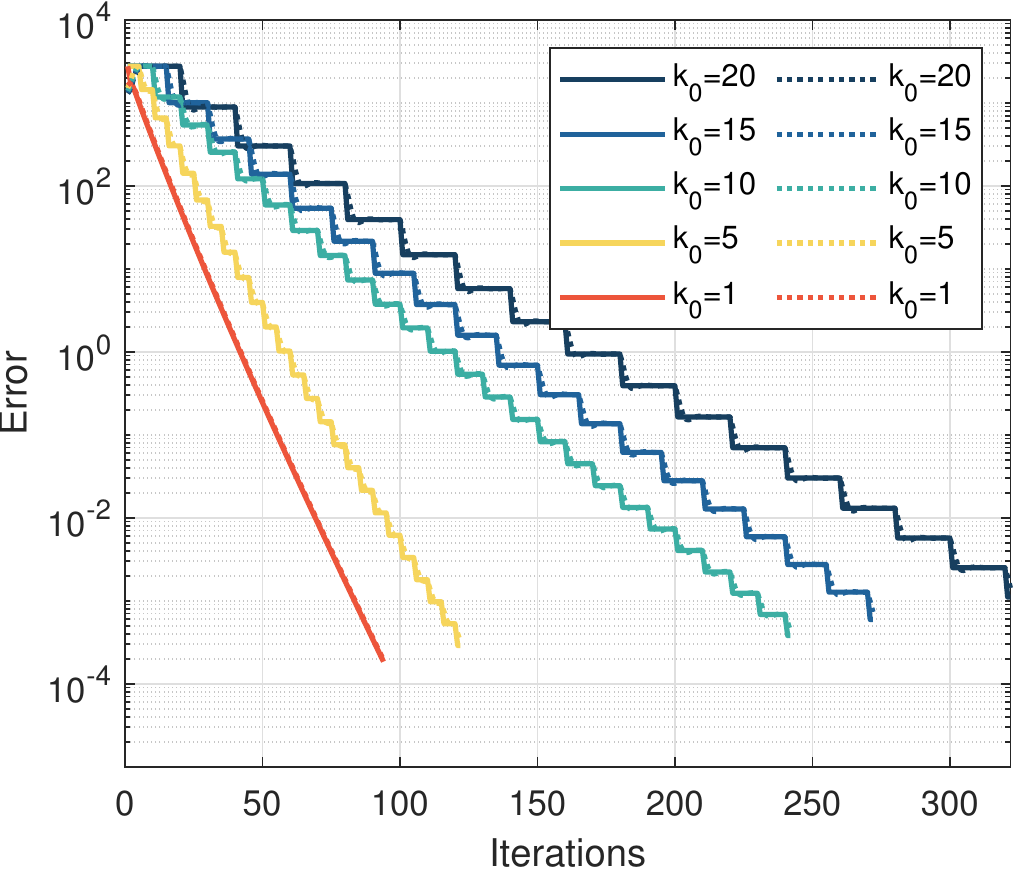}
	\caption{{\tt ICEADMM}  solving Ex. \ref{ex-linear}}
	\label{fig:ICEADMM-diff-grad}
\end{subfigure} 
\begin{subfigure}{.33\textwidth}
	\centering
	\includegraphics[width=.95\linewidth]{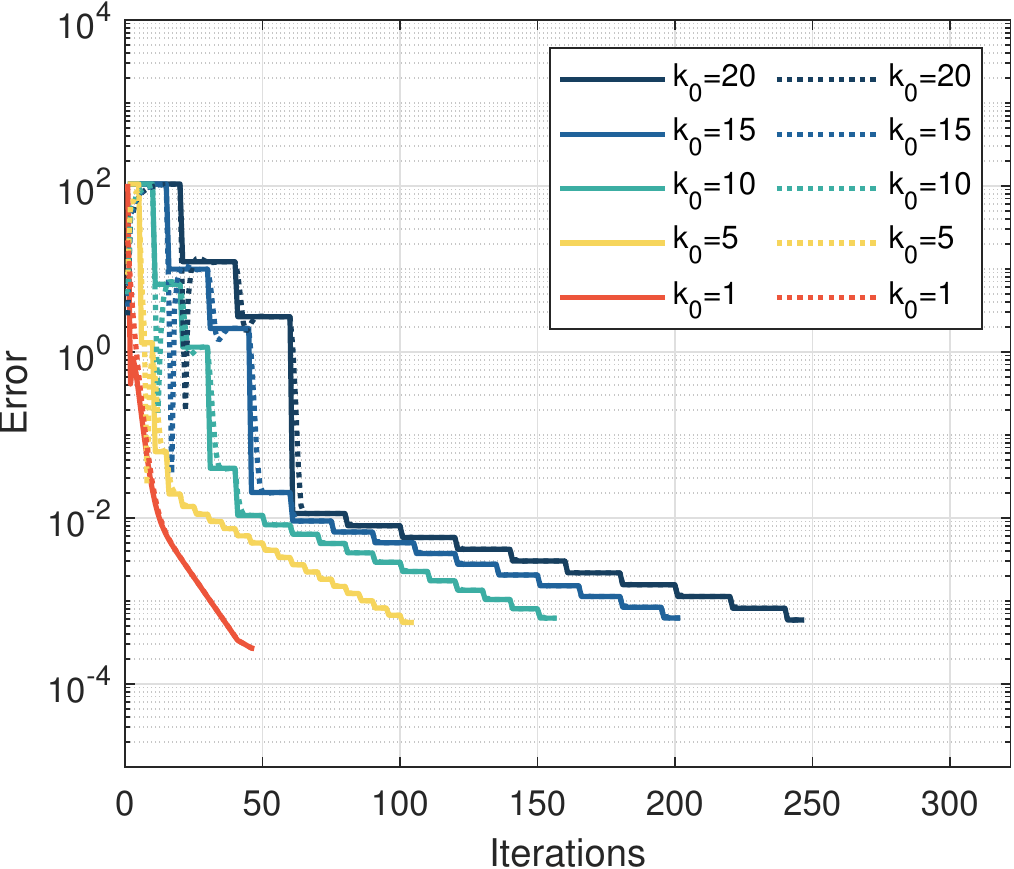}
	\caption{{\tt ICEADMM} solving Ex. \ref{ex-logist}}
	\label{fig:ICEADMM-log-grad-1}
\end{subfigure}   
\caption{$\|\nabla f(\by^k)\|^2$ (solid lines) and $\|\nabla F(X^k)\|^2$ (dashed lines) v.s. iterations.\label{fig:iterations-grad}}
\end{figure}
 
\subsubsection{Effect of $k_0$}
Next, we would like to see how the choices of $k_0$ impact the performance of the two algorithms. To proceed with that, for each dimension $(m,n,d_1,\ldots,d_m)$ of  dataset A, we generate  20 instances solved by the algorithm with a fixed $k_0\in[20]$ and present the average results in Figure  \ref{fig:effect-k0-diff}, where the following comments can be declared:

{\bf (a)  {\tt CEADMM} solving Example \ref{ex-linear}.} From Figures \ref{fig:k0-iter-1} to \ref{fig:k0-aggr-1}, with increasing $k_0$,  the total number of iterations is increasing but communication rounds are decreasing. That is,  {\tt CEADMM} with $k_0>1$ takes  fewer global aggregations than the standard {\tt ADMM}. For example, when $m=30$ in Figure \ref{fig:k0-aggr-2},  {\tt IADMM} requires 118 rounds of communications while {\tt ICEADMM} with $k_0=20$ only needs approximately 20 rounds.  To this end, it is much more efficient than {\tt IADMM} in terms of saving communication costs.

\begin{figure}[!th]
\begin{subfigure}{.33\textwidth}
	\centering
	\includegraphics[width=.99\linewidth]{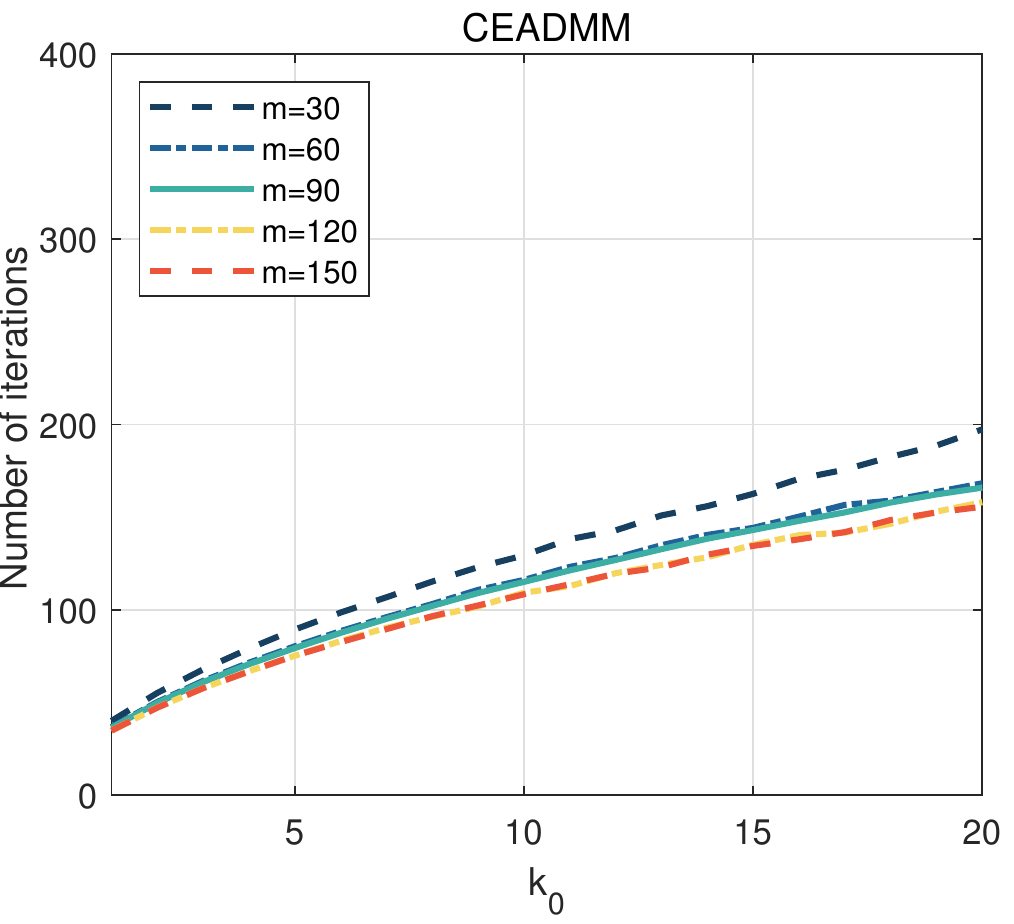}\vspace{-2mm}
	\caption{Number of iterations}
	\label{fig:k0-iter-1}
\end{subfigure}	 
\begin{subfigure}{.33\textwidth}
	\centering
	\includegraphics[width=.99\linewidth]{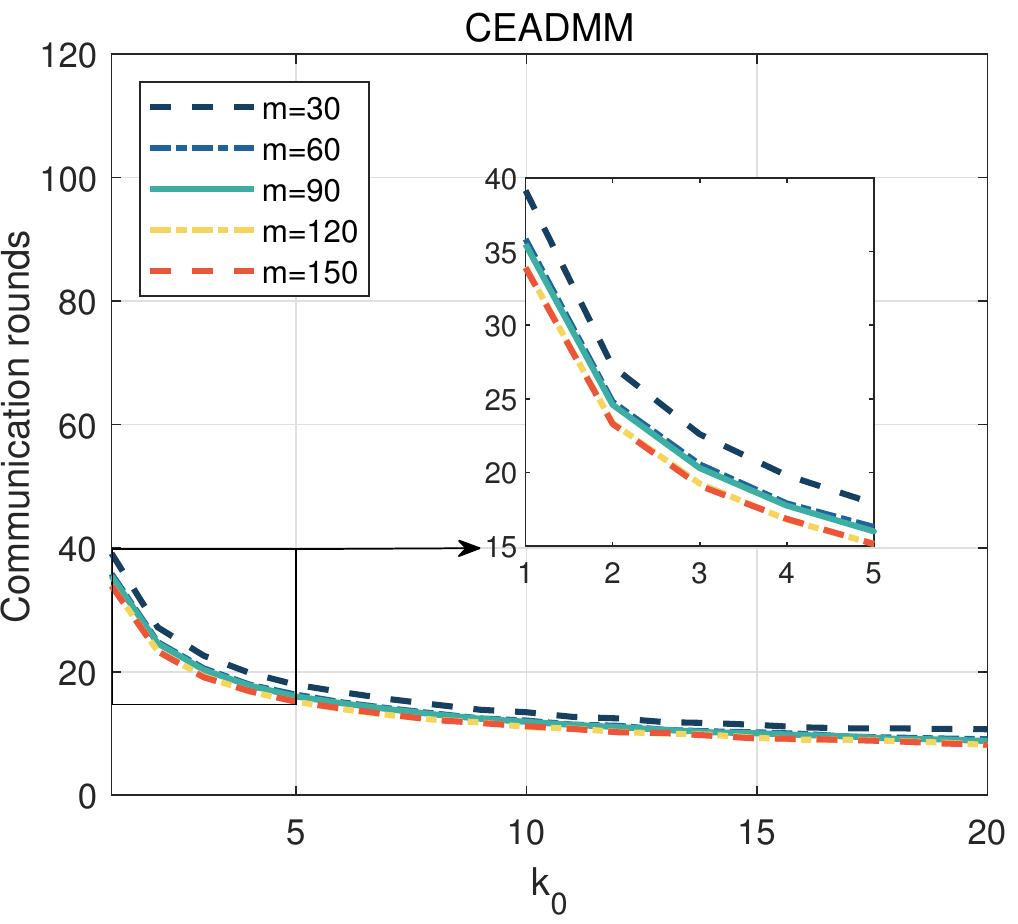}\vspace{-2mm}
	\caption{Communication rounds}
	\label{fig:k0-aggr-1}
\end{subfigure}  
\begin{subfigure}{.33\textwidth}
	\centering
	\includegraphics[width=.99\linewidth]{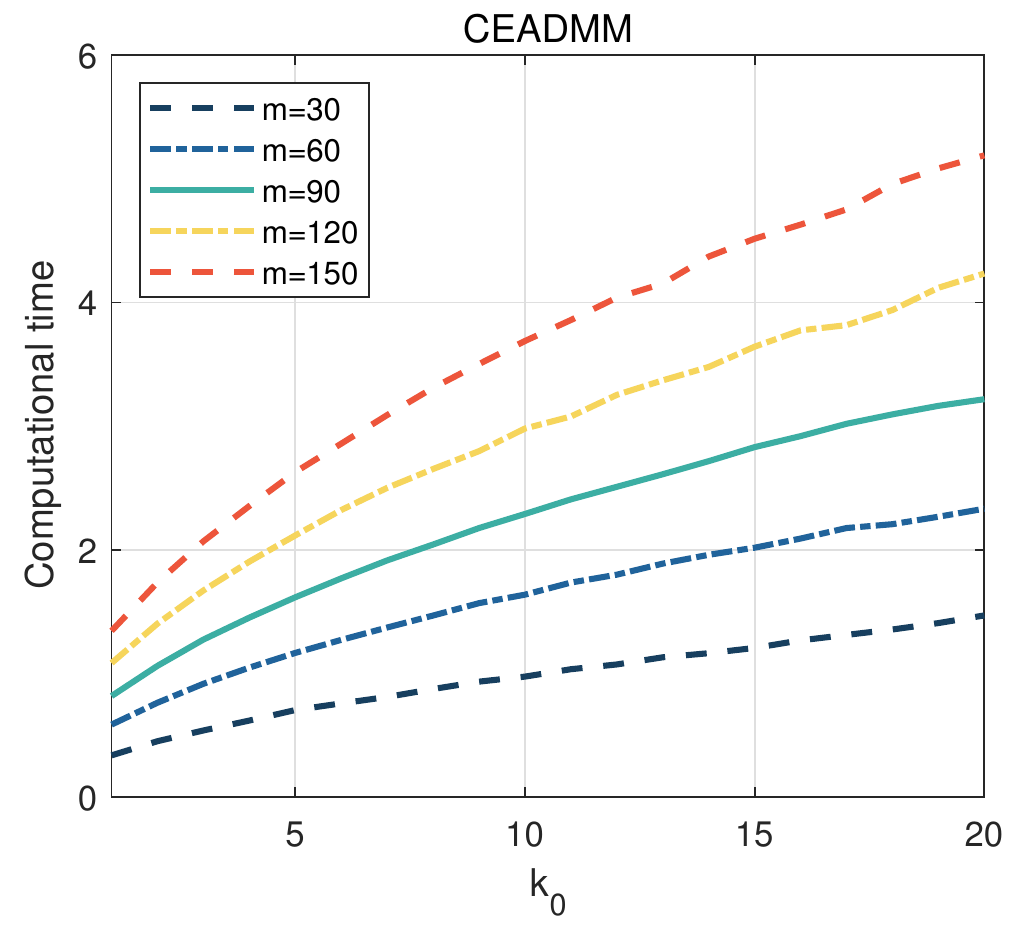}\vspace{-2mm}
	\caption{Computational time}
	\label{fig:k0-time-1}
\end{subfigure}\\
\vspace{3mm}

\begin{subfigure}{.33\textwidth}
	\centering
	\includegraphics[width=.99\linewidth]{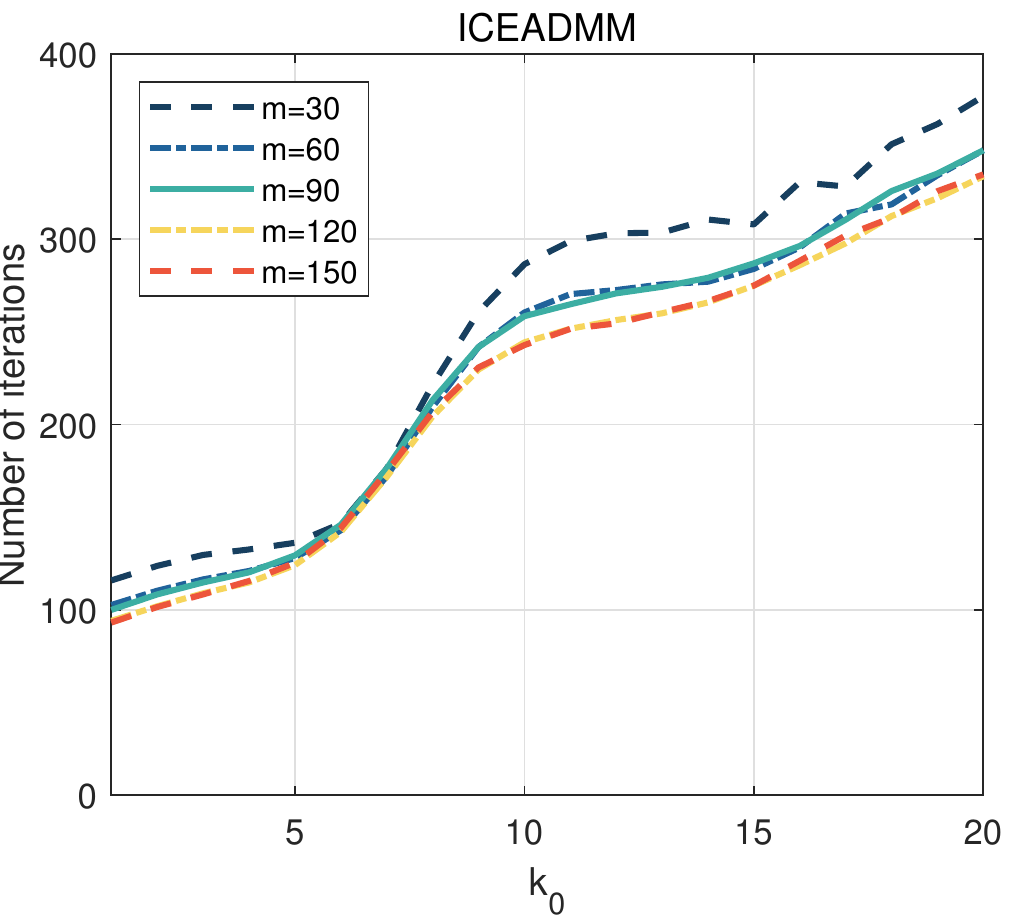}\vspace{-2mm}
	\caption{Number of iterations}
	\label{fig:k0-iter-2}
\end{subfigure}	 
\begin{subfigure}{.33\textwidth}
	\centering
	\includegraphics[width=.99\linewidth]{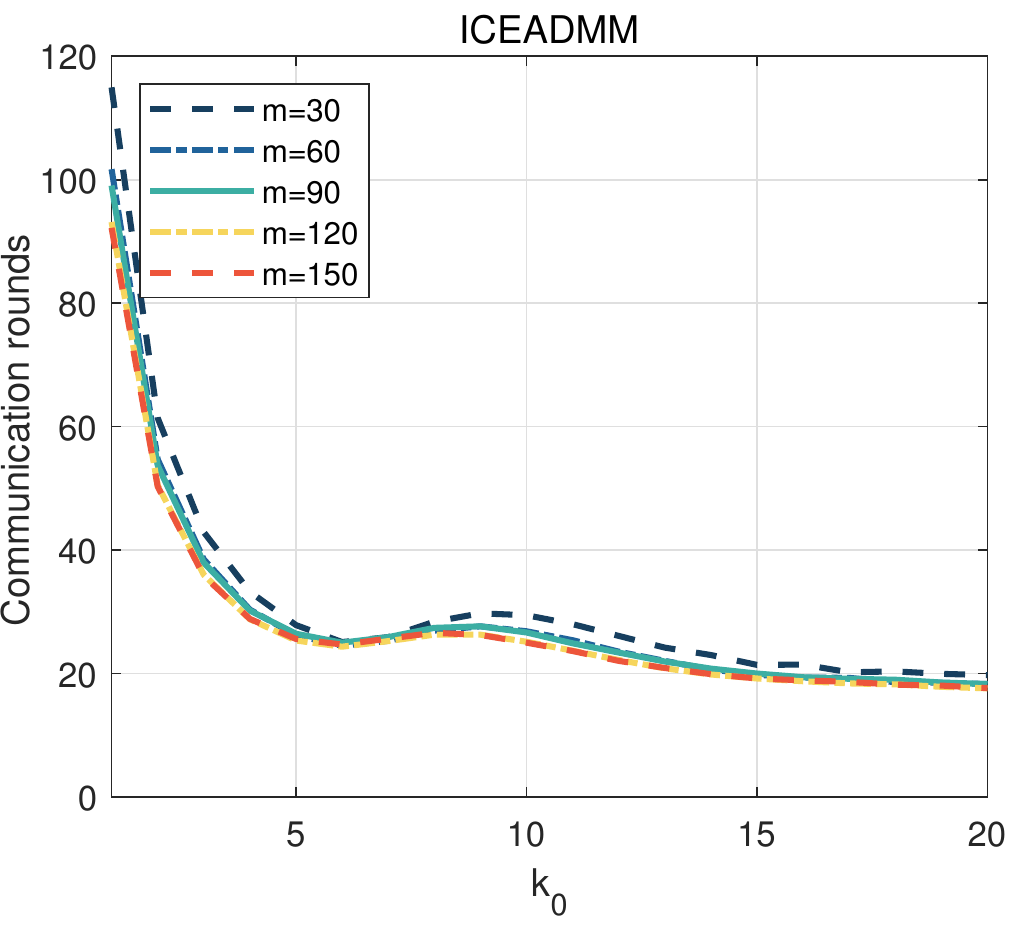}\vspace{-2mm}
	\caption{Communication rounds}
	\label{fig:k0-aggr-2}
\end{subfigure}  
\begin{subfigure}{.33\textwidth}
	\centering
	\includegraphics[width=.99\linewidth]{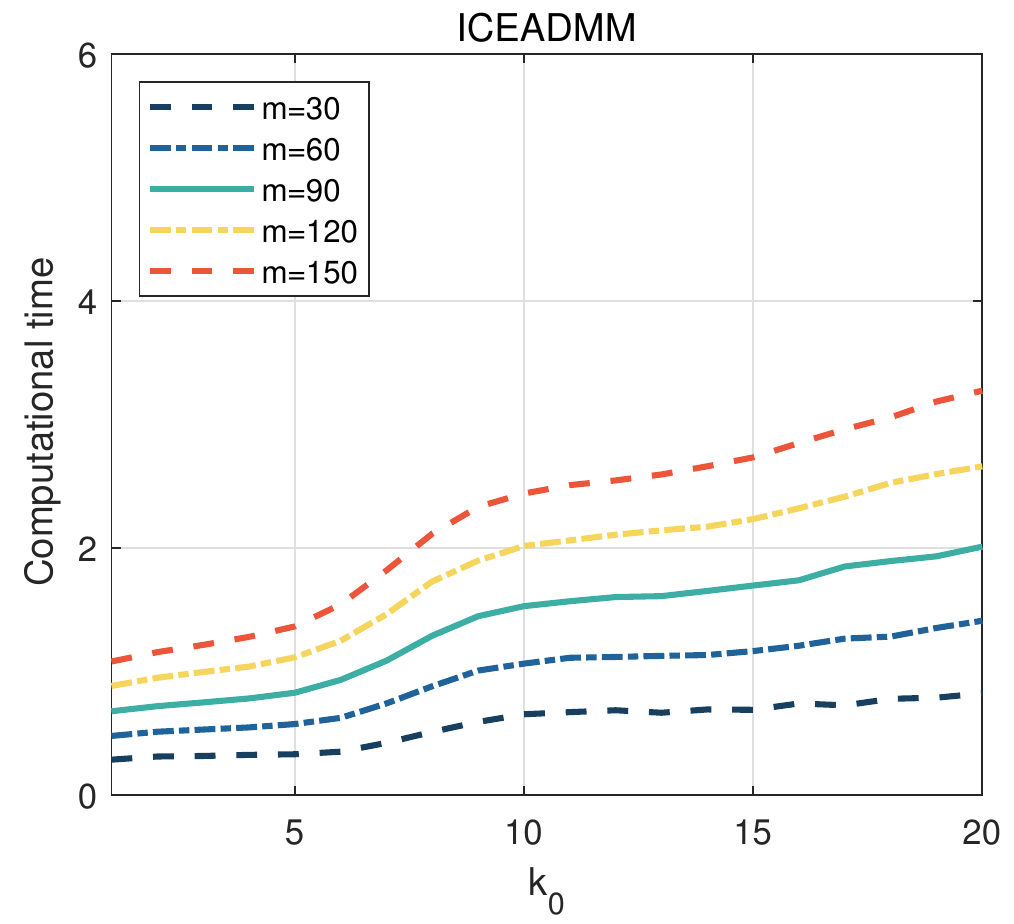}\vspace{-2mm}
	\caption{Computational time}
	\label{fig:k0-time-2}
\end{subfigure} 
\caption{Effect of $k_0$ for Example \ref{ex-linear}.\label{fig:effect-k0-diff}}
\end{figure}

 Since we conducted the numerical experiments on a single laptop, the total computational time relies on the number of iterations. The curves in Figure \ref{fig:k0-time-1} show that the bigger values of $k_0$, the longer time  because bigger $k_0$ results in more iterations, as shown in Figure \ref{fig:k0-iter-1}. Suppose that if the local updates (i.e., \eqref{ceadmm-sub2} and \eqref{ceadmm-sub3}) of  {\tt CEADMM}  are implemented on different local devices, such as cellphones, laptops, or desktops, then we must take the price of communications between the local devices and the central server into consideration since more communications lead to an extremely higher price. Hence, it is necessary to reduce the number of global aggregations, that is, to set a properly bigger $k_0$.
 
 {\bf (b)  {\tt ICEADMM} solving Example \ref{ex-linear}.} Corresponding results presented in Figures \ref{fig:k0-iter-2} - \ref{fig:k0-time-2} are similar to those in Figures \ref{fig:k0-iter-1} - \ref{fig:k0-time-1}. Again, the total number of iterations and computational time are rising but communication rounds are declining along with $k_0$ ascending. Once again,  {\tt ICEADMM} with $k_0>1$ is more communication-efficient than the standard {\tt IADMM} since it executes much fewer global  aggregations.

\begin{figure}[!th]
\begin{subfigure}{.33\textwidth}
	\centering
	\includegraphics[width=.99\linewidth]{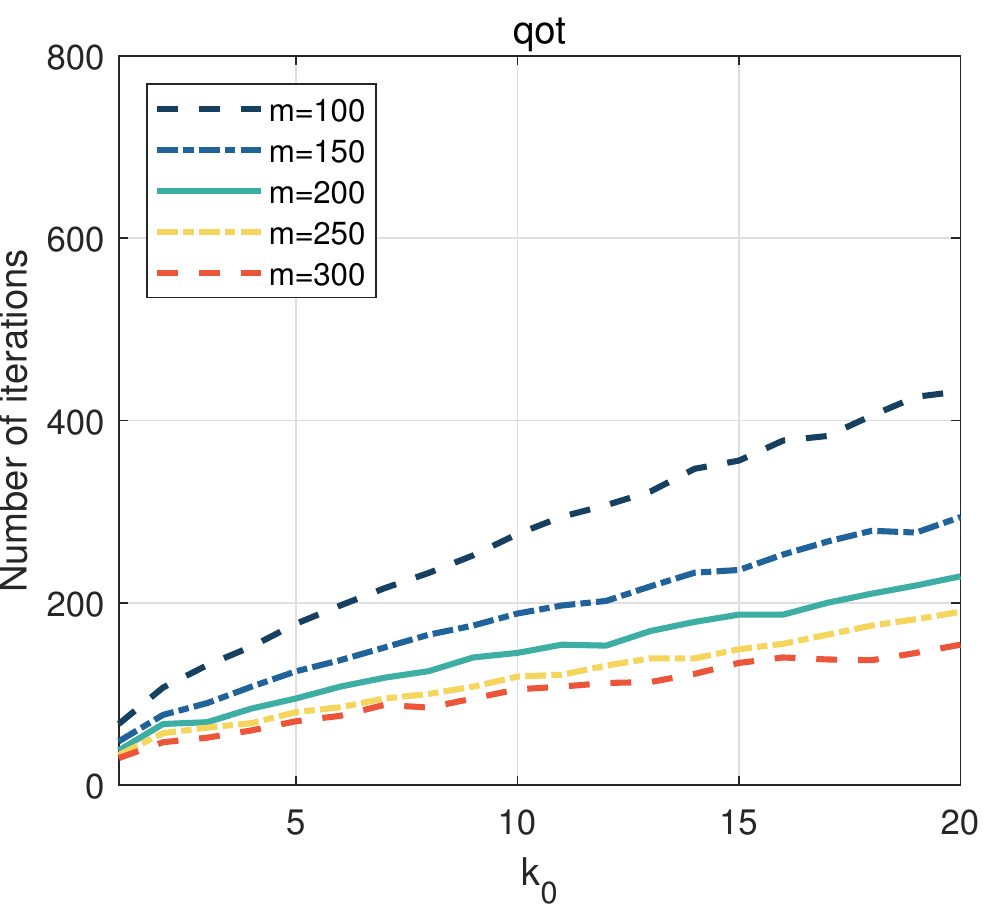}
	\caption{Number of iterations}
	\label{fig:log-k0-iter-1}
\end{subfigure}	 
\begin{subfigure}{.33\textwidth}
	\centering
	\includegraphics[width=.99\linewidth]{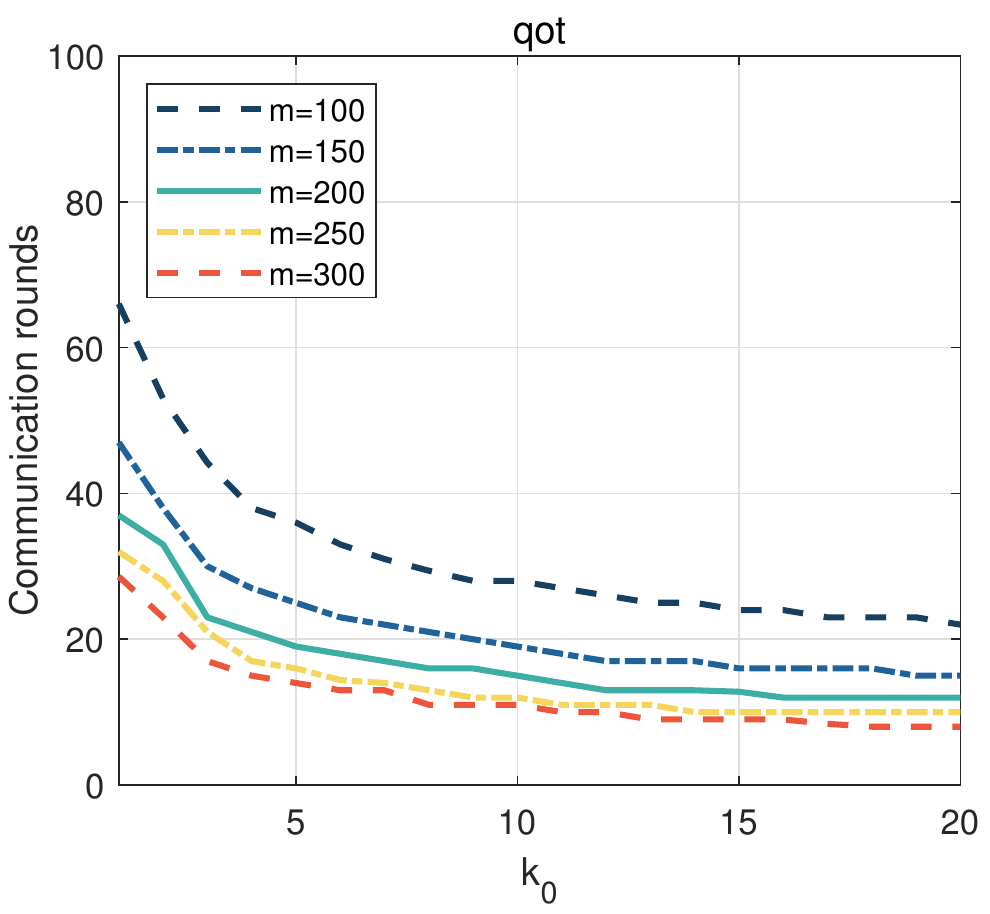}
	\caption{Communication rounds}
	\label{fig:log-k0-aggr-1}
\end{subfigure}  
\begin{subfigure}{.33\textwidth}
	\centering
	\includegraphics[width=.99\linewidth]{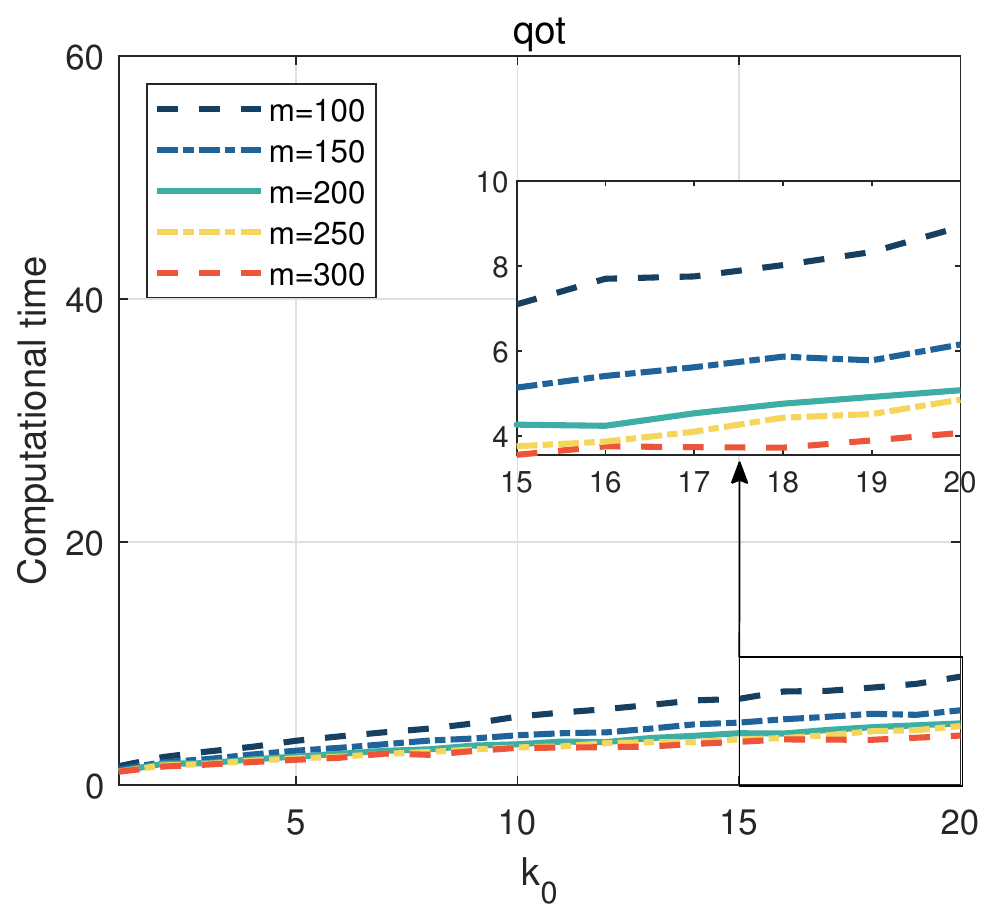}
	\caption{Computational time}
	\label{fig:log-k0-time-1}
\end{subfigure}\\
\vspace{1mm}

\begin{subfigure}{.33\textwidth}
	\centering
	\includegraphics[width=.99\linewidth]{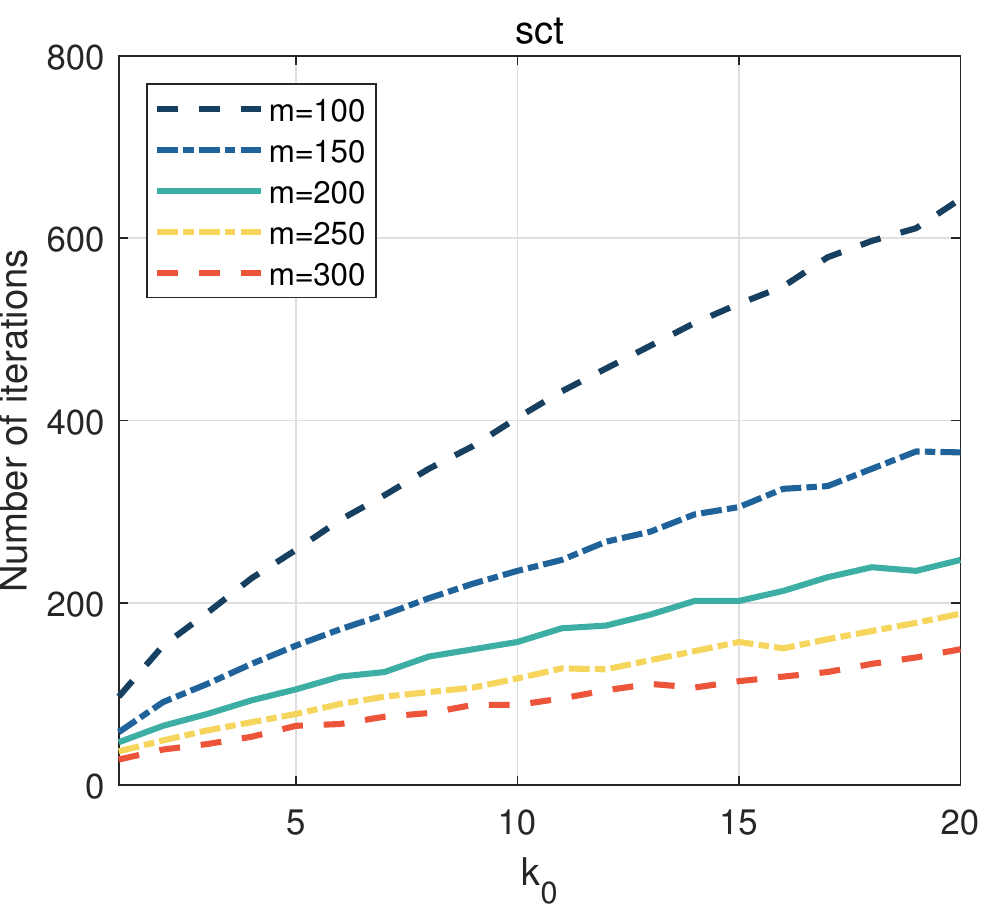}
	\caption{Number of iterations}
	\label{fig:log-k0-iter-2}
\end{subfigure}	 
\begin{subfigure}{.33\textwidth}
	\centering
	\includegraphics[width=.99\linewidth]{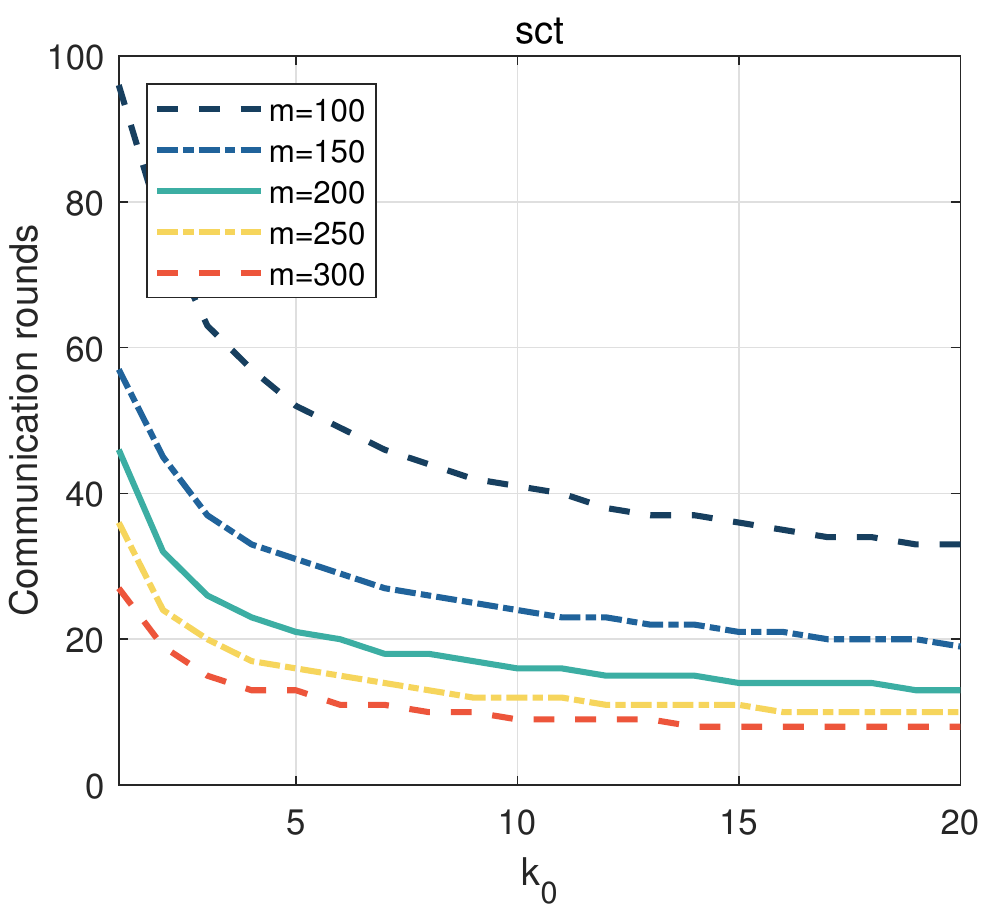}
	\caption{Communication rounds}
	\label{fig:log-k0-aggr-2}
\end{subfigure}  
\begin{subfigure}{.33\textwidth}
	\centering
	\includegraphics[width=.99\linewidth]{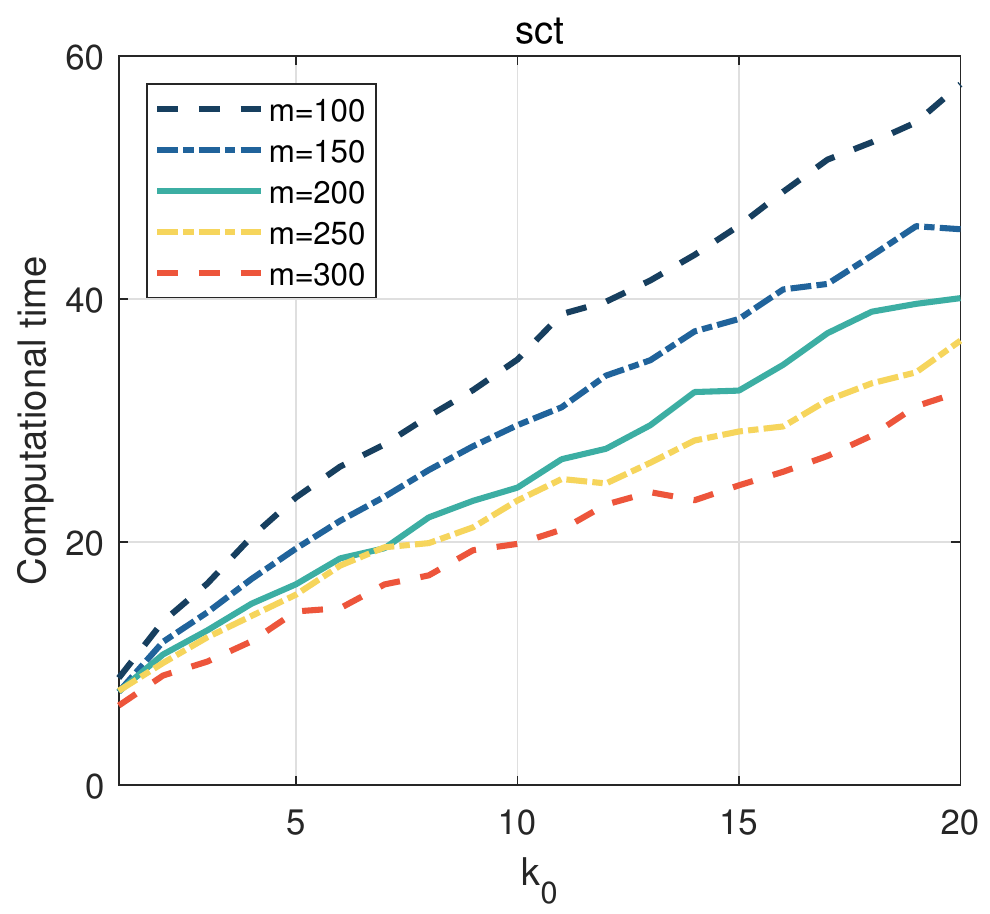}
	\caption{Computational time}
	\label{fig:log-k0-time-2}
\end{subfigure}  
\caption{Effect of $k_0$ for Example \ref{ex-logist}.\label{fig:log-effect-k0-diff}}
\end{figure} 

  {\bf (c)  {\tt ICEADMM} solving Example \ref{ex-logist}.}  Figure \ref{fig:log-effect-k0-diff} demonstrates that there is no big difference with  Figure \ref{fig:effect-k0-diff}. Moreover, the bigger $m$ (i.e., the more clients), the fewer number of iterations consumed to converge, leading to the shorter computational time. Once again, the larger $k_0$ results in the lower communication rounds, displaying the higher communication efficiency.

\section{Conclusion}
This paper developed two ADMM-based federated learning algorithms and managed to address three issues in federated learning, including saving communication resources, reducing computational complexity, and establishing convergence property under very reasonable assumptions. These advantages hint that the proposed algorithms might be practical to deal with many real applications such as mobile edge
computing \cite{mao2017survey,mach2017mobile}, over-the-air computation \cite{zhu2018mimo,yang2020federated}, vehicular
communications \cite{samarakoon2019distributed}, unmanned aerial vehicle online path control \cite{shiri2020communication} and so forth.  Moreover, we feel that the algorithmic schemes and techniques used to build the convergence theory could be also valid for tackling decentralized federated learning \cite{elgabli2020fgadmm,ye2021decentralized}. We leave these as future research.

\bibliographystyle{unsrt}
\bibliography{ref}

 \section{Appendix}
For notational simplicity, hereafter, we denote
\begin{eqnarray*} 
\qquad \begin{array}{llllll}
\triangle\by^{k+1}:=\by^{k+1}-\by^{k},\qquad\triangle\bx_i^{k+1}:=\bx_i^{k+1}-\bx_i^{k},\qquad
\triangle\bpi_i^{k+1}:=\bpi_i^{k+1}-\bx_i^{k},
    \end{array}  
 \end{eqnarray*} 
 and let   $\ba^{k} \rightarrow \ba$ stand for $\lim_{k\rightarrow\infty} \ba^{k} = \ba$.
For any  vectors $\ba,\bb$, and $\ba_i$, we have
\begin{eqnarray} \label{two-vecs}
 \arraycolsep=1.4pt\def\arraystretch{1.35}
 \begin{array}{rcl}
 -  \|\bb\|^2 &=&2\langle\ba,\bb\rangle +   \|\ba\|^2-  \|\ba+\bb\|^2,\\
\|\ba+\bb\|^2 &\leq& (1+t)\|\ba\|^2+(1+\frac{1}{t})\|\bb\|^2,~ \forall~ t>0,\\
\|\sum_{i=1}^m\ba_i\|^2 &\leq& m\sum_{i=1}^m\|\ba_i\|^2 .
 \end{array}
\end{eqnarray}
The gradient Lipschitz continuity 
in \eqref{Lip-r} with a Lipschitz constant $r>0$ indicates that 
\begin{eqnarray} \label{Lip-r/2}
\begin{array}{l}
f(\bx)\leq  f(\bz  )+\langle \nabla  f(\bz  ), \bx-\bz   \rangle +  \frac{r}{2}\| \bx-\bz  \|^2, \end{array}\end{eqnarray}
for any $\bx,\bz  \in\R^n$. Hence,  there exists $H$ such that $r I \succeq H\succeq0$ and
\begin{eqnarray}
\label{f-majorized-h}\begin{array}{l}
f(\bx)\leq  f(\bz  )+\langle \nabla f(\bz  ), \bx-\bz   \rangle +  \frac{1}{2}\| \bx-\bz \|^2_H.\end{array}\end{eqnarray} 
The above condition immediately implies \begin{eqnarray} \label{inequality-Gi}
\qquad \begin{array}{llll}
 \langle \nabla  f(\bx) - \nabla  f(\bz  ), \bx-\bz   \rangle &\leq& \| \bx-\bz \|^2_H.
 \end{array}
\end{eqnarray} 
A special case of $f$ is a quadratic function and $H =\nabla^2 f$, the Hessian matrix of $f$.  
It follows from the Mean Value Theorem that, for any $\bx$ and $\bz$,
\begin{eqnarray}  \label{H-Lip-continuity-fxy}
 \arraycolsep=1.4pt\def\arraystretch{1.35}  \begin{array}{llll}
f(\bx) - f(\bz  ) -\langle \nabla  f(\bx  ), \bx-\bz   \rangle &=&  \int_0^1 d f(\bz+t(\bx-\bz))-\langle \nabla  f(\bx  ), \bx-\bz   \rangle\\
&=&  \int_0^1 \langle \nabla  f(\bz+t(\bx-\bz)) - \nabla  f(\bx  ), \bx-\bz   \rangle dt  \\
&\leq&  \int_0^1 r\|\bz+t(\bx-\bz)) -  \bx \|\| \bx-\bz   \| dt\\
&=&  r\| \bx-\bz   \|^2\int_0^1 (1-t)dt\\
&=& \frac{r}{2}\| \bx-\bz   \|^2.
 \end{array}
\end{eqnarray}  

\subsection{Proofs of all theorems in Section \ref{sec:ceadmm}}


\begin{lemma}\label{lemma-decreasing-0}  Let $\{(\by^{k},X^{k},\Pi^{k})\}$ be the sequence generated by Algorithm \ref{algorithm-CEADMM}. If Assumption \ref{ass-fi} holds, then for any $k\geq0$, 
\begin{eqnarray} 
 \label{decreasing-property-0}   
\begin{array}{lll}
 \L^{k+1}-\L^{k}    \leq -\frac{\sigma}{2}\|\triangle\by^{k+1} \|^2 - \sum_{i=1}^{m} \frac{\theta_i}{2}\|\triangle\bx_i^{k+1} \|^2,
    \end{array}  
 \end{eqnarray} 
 where $\theta_i   $ is given by
 \begin{eqnarray} 
 \label{def-ci}  
 \begin{array}{lll} \theta_i   :=   \sigma_i  -w_ir_i- {2w_i^2r_i^2}/{\sigma_i}, \qquad i\in[m].
 \end{array}  \end{eqnarray}
\end{lemma}  
\begin{proof} The optimality condition for \eqref{ceadmm-sub1} is
 \begin{eqnarray}
 \label{opt-con-xk1}
 \arraycolsep=1.4pt\def\arraystretch{1.35}\begin{array}{lcl}
\forall~k\in\K:~~
  0 &=&  \sum_{i=1}^{m} [  -\bpi_i^k + \sigma_i (\bx^{k+1}-\bx_i^k)]\\
  &\overset{\eqref{ceadmm-sub3}}{=}&  \sum_{i=1}^{m} [  -\bpi_i^{k+1} + \sigma_i (\bx_i^{k+1}-\by^{k+1}) ] +  \sum_{i=1}^{m} [ \sigma_i (\bx^{k+1}-\bx_i^k)] \\
  &=&  \sum_{i=1}^{m} [  -\bpi_i^{k+1} + \sigma_i (\bx_i^{k+1}-\bx_i^k)], 
   \end{array} \end{eqnarray}  
 where the last equation holds from $\by^{k+1}=\bx^{k+1}, k\in\K$. The optimality condition for  \eqref{ceadmm-sub2} is   
   \begin{eqnarray}
 \label{opt-con-xik1}
\begin{array}{lcl}
\forall~k\geq0:~~0 =  \bg_i^{k+1} +\bpi_i^k + \sigma_i (\bx_i^{k+1}-\by^{k+1}) \overset{\eqref{ceadmm-sub3}}{=}  \bg_i^{k+1} +\bpi_i^{k+1},\qquad \forall~i\in[m]. 
   \end{array}
  \end{eqnarray}  
The gap of the left-hand side of \eqref{decreasing-property-0}   can be decomposed as 
  \begin{eqnarray} 
\label{three-cases}  
 \qquad   \begin{array}{llllll}
 \L^{k+1}-\L^{k}  =: e_1^k+e_2^k+e_3^k,
    \end{array} 
 \end{eqnarray} 
 with
   \begin{eqnarray}  \label{three-cases-sub}  
 \arraycolsep=1.4pt\def\arraystretch{1.35} \begin{array}{llllll}
e_1^k&:=&\L(\by^{k+1},X^{k},\Pi^{k})-\L^{k},\\
 e_2^k&:=& \L(\by^{k+1},X^{k+1},\Pi^{k})-\L(\by^{k+1},X^{k},\Pi^{k}), \\
 e_3^k&:=& \L^{k+1}-\L(\by^{k+1},X^{k+1},\Pi^{k}) .
    \end{array} 
 \end{eqnarray} 
To estimate  $e_1^k$, if  $k\notin\K$, then from \eqref{x-y-relation}, we have $\by^{k+1}= \bx^{\tau_k+1}=\by^{k}$, thereby leading to
\begin{eqnarray} 
\label{two-cases-1} \begin{array}{lll}  
   e_1^k =\L(\by^{k+1},X^{k},\Pi^{k})-\L^{k}=0 =- \frac{\sigma}{2}\|\triangle\by^{k+1} \|^2. 
   \end{array}
 \end{eqnarray}
If  $k\in\K$, then \eqref{x-y-relation}  indicates $\by^{k+1} = \bx^{k+1}$. Moreover, multiplying both sides of the first equation in \eqref{opt-con-xk1} by $\triangle\by^{k+1}$ gives rise to
 \begin{eqnarray}
 \label{opt-con-xk2}
   \begin{array}{lll}
 \sum_{i=1}^{m} \langle  \triangle\by^{k+1},\bpi_i^k  \rangle= \sum_{i=1}^{m} \langle\triangle\by^{k+1},  \sigma_i (\by^{k+1}-\bx_i^k)\rangle. 
\end{array}
  \end{eqnarray}  
These facts also allow  us to derive that
\begin{eqnarray} \label{two-cases-2}   
 \arraycolsep=1.4pt\def\arraystretch{1.35}
\begin{array}{lcl}
  e_1^k 
   &\overset{\eqref{FL-opt-ver11}}{=}& \sum_{i=1}^{m} ( L(\by^{k+1},\bx_i^{k},\bpi_i^{k})  - L(\by^{k},\bx_i^{k},\bpi_i^{k}))\\
    &\overset{\eqref{Def-L}}{=}&   \sum_{i=1}^{m} (  \langle \triangle\by^{k+1}, -\bpi_i^k\rangle +   \frac{\sigma_i}{2}\|\bx_i^k-\by^{k+1}\|^2 -\frac{\sigma_i}{2}\|\bx_i^k- \by^{k}\|^2) \\
&\overset{\eqref{opt-con-xk2}}{=}& \sum_{i=1}^{m} (  \langle  \triangle\by^{k+1},  \sigma_i(\bx_i^k-\by^{k+1}) \rangle +   \frac{\sigma_i}{2}\|\bx_i^k-\by^{k+1}\|^2 -\frac{\sigma_i}{2}\|\bx_i^k- \by^{k}\|^2 )\\
&\overset{\eqref{two-vecs}}{=}&  \sum_{i=1}^{m} ( -\frac{\sigma_i}{2}\| \triangle\by^{k+1}\|^2)\\ &\overset{\eqref{sum-sigma-i}}{=}& -\frac{\sigma}{2}\| \triangle\by^{k+1}\|^2.   
    \end{array} 
 \end{eqnarray}
To estimate $e_2^k$,  we have the following chain of inequalities,
\begin{eqnarray*} 
 \arraycolsep=1.4pt\def\arraystretch{1.35}  \begin{array}{lcl}
 &&     L(\by^{k+1},\bx_i^{k+1},\bpi_i^{k})  - L(\by^{k+1},\bx_i^{k},\bpi_i^{k})   \\
&\overset{\eqref{Def-L}}{=}&  w_i  f_i(\bx_i^{k+1})- w_i f_i(\bx_i^k)   +  \langle \triangle \bx_i^{k+1} , \bpi_i^k\rangle +    \frac{\sigma_i}{2}\|\bx_i^{k+1}-\by^{k+1}\|^2 -\frac{\sigma_i}{2}\|\bx_i^{k}-\by^{k+1}\|^2 \\
          &\overset{ \eqref{H-Lip-continuity-fxy}}{\leq}&     \langle  \triangle \bx_i^{k+1} , \bg_i^{k+1} + \bpi_i^k\rangle + \frac{w_ir_i}{2}\| \triangle\bx_i^{k+1}\|^2+ \frac{\sigma_i}{2}\|\bx_i^{k+1}-\by^{k+1}\|^2 -\frac{\sigma_i}{2}\|\bx_i^{k}-\by^{k+1}\|^2\\
   &\overset{ \eqref{opt-con-xik1}}{=}&  \langle \triangle\bx_i^{k+1},\sigma_i(\by^{k+1} -\bx_i^{k+1})\rangle + \frac{w_ir_i}{2}\| \triangle\bx_i^{k+1}\|^2  + \frac{\sigma_i}{2}\|\bx_i^{k+1}-\by^{k+1}\|^2 -\frac{\sigma_i}{2}\|\bx_i^{k}-\by^{k+1}\|^2   \\  
   &\overset{\eqref{two-vecs}}{=}&   -   \frac{\sigma_i-w_ir_i}{2}\| \triangle\bx_i^{k+1}\|^2,
    \end{array} 
 \end{eqnarray*}
 which leads to 
\begin{eqnarray}\label{gap-2}  
\begin{array}{lllll}
 e_2^k&=&  \sum_{i=1}^{m}  (     L(\by^{k+1},\bx_i^{k+1},\bpi_i^{k})  - L(\by^{k+1},\bx_i^{k},\bpi_i^{k}))\le\sum_{i=1}^{m}  -   \frac{\sigma_i-w_ir_i}{2}\| \triangle\bx_i^{k+1}\|^2.
    \end{array} 
 \end{eqnarray}
To estimate $e_3^k$,  the  gradient Lipschitz continuity delivers that, for any $i\in[m]$,
 \begin{eqnarray}\label{Lip-fi} 
  \begin{array}{lll}
\| \triangle \bpi_i^{k+1} \|\overset{\eqref{opt-con-xik1}}{=} \|-\bg_i^{k+1}+\bg_i^{k}\| \overset{\eqref{Lip-r}}{\leq}  w_i r_i\| \triangle\bx_i^{k+1}\|.
\end{array}
  \end{eqnarray}
This condition suffices to the following chain of inequalities,
\begin{eqnarray} \label{gap-3} 
 \arraycolsep=1.4pt\def\arraystretch{1.35}  \begin{array}{lcl}
e_3^k
 &= &  \L^{k+1}- \sum_{i=1}^{m}  L(\by^{k+1},\bx_i^{k+1},\bpi_i^{k})  \\
&\overset{\eqref{Def-L}}{=}&    \sum_{i=1}^{m} \langle \bx_i^{k+1}-\by^{k+1}, \triangle \bpi_i^{k+1} \rangle \\
       &\overset{ \eqref{ceadmm-sub3}}{=}&   \sum_{i=1}^{m}   \frac{1}{\sigma_i}\|\triangle \bpi_i^{k+1}  \|^2\\
       & \overset{\eqref{Lip-fi}}{\leq}   &\sum_{i=1}^{m}  \frac{w_i^2r_i^2}{\sigma_i}\|\triangle\bx_i^{k+1}\|^2, 
    \end{array} 
 \end{eqnarray}
 Overall, combining \eqref{three-cases}, \eqref{two-cases-1}, \eqref{two-cases-2}, \eqref{gap-2} and \eqref{gap-3}, we conclude
\eqref{decreasing-property-0}   immediately.
\end{proof}
 \begin{lemma}\label{L-bounded-decreasing}   Let $\{(\by^{k},X^{k},\Pi^{k})\}$ be the sequence generated by Algorithm \ref{algorithm-CEADMM} with $\sigma_i> 2w_ir_i$ for every $i\in[m]$. The following results hold under Assumption \ref{ass-fi}.
 \begin{itemize}
 \item[i)] $\{\L^{k}\}$ is non-increasing. 
 \item[ii)] $\L^{k} \geq f (\by^{k}) \geq f^* >-\infty$ for any integer $ k\geq0$.
 \item[iii)]The limits of all the following terms are zero, namely,  
\begin{eqnarray} \label{limit-5-term-0}
 \arraycolsep=1.4pt\def\arraystretch{1.35} \begin{array}{lll}\underset{k \rightarrow \infty}{\lim}  \triangle \by^{k+1}&=& \underset{k \rightarrow \infty}{\lim} \triangle \bx^{k+1}_i=\underset{k \rightarrow \infty}{\lim} \triangle \bpi_i^{k+1}\\
 &=&\underset{k \rightarrow \infty}{\lim} (\bx_i^{k+1}-\by^{k+1})=\underset{k \rightarrow \infty}{\lim} (\bx_i^{k+1}-\bx_i^{\tau_k+1})=0.
    \end{array} 
 \end{eqnarray}
 \end{itemize}
 
\end{lemma}  
\begin{proof} i) The conclusion follows from \eqref{decreasing-property-0} and  $\theta_i >0$ (due to $\sigma_i>2w_ir_i$) immediately.

 ii)  The gradient Lipschitz continuity of $f_i$   implies
\begin{eqnarray*} 
 \arraycolsep=1.4pt\def\arraystretch{1.35} \begin{array}{lllll}
 w_if_i(\by^{k+1})-w_if_i(\bx_i^{k+1})
  &\overset{\eqref{Lip-r/2}}{\leq}&  \langle \by^{k+1}- \bx_i^{k+1},  \bg_i^{k+1}\rangle +\frac{r_iw_i}{2}\|\bx_i^{k+1}-\by^{k+1}\|^2 \\
 &\overset{\eqref{opt-con-xik1}}{=}&       \langle \by^{k+1}- \bx_i^{k+1},  -\bpi_i^{k+1} \rangle +\frac{r_iw_i}{2}\|\bx_i^{k+1}-\by^{k+1}\|^2,
    \end{array} 
 \end{eqnarray*}
which by  $\sigma_i>2w_ir_i$  allows us to obtain
\begin{eqnarray}  \label{lower-bound-L}
   \arraycolsep=1.4pt\def\arraystretch{1.35}
  \begin{array}{lll}
\L^{k+1}
 &=&  \sum_{i=1}^{m}  L(\by^{k+1},\bx_i^{k+1},\bpi_i^{k+1})  \\
    &=&  \sum_{i=1}^{m} (w_i  f_i(\bx_i^{k+1})+      \langle \bx_i^{k+1}-\by^{k+1}, \bpi_i^{k+1}\rangle)+ \sum_{i=1}^{m} \frac{\sigma_i}{2}\|\bx_i^{k+1}-\by^{k+1}\|^2 \\
     &\geq&  \sum_{i=1}^{m} ( w_i  f_i(\by^{k+1})        +\frac{\sigma_i-w_ir_i}{2}\|\bx_i^{k+1}-\by^{k+1}\|^2 ) \\
&\geq&  \sum_{i=1}^{m}    w_i  f_i(\by^{k+1})  =f(\by^{k+1}) \\
     & \geq& f^*   \overset{\eqref{FL-opt-lower-bound}}{ >} -\infty.
    \end{array} 
 \end{eqnarray}

 iii)  It follows from \eqref{decreasing-property-0}  that
 \begin{eqnarray*} 
 \arraycolsep=1.4pt\def\arraystretch{1.35}  \begin{array}{lcl}
 {\sum}_{k\geq0} \Big[ \frac{\sigma}{2}\| \vartriangle \by^{k+1}  \|^2 + ~ \sum_{i=1}^{m}   \frac{\theta_i   }{2} \| \triangle\bx^{k+1}_i   \|^2 \Big]
&\leq&{\sum}_{k\geq0}  ( \L^{k}  -\L^{k+1})\\
& =&
  \L^{0}- {\lim}_{k\rightarrow\infty}\L^{k+1}\overset{\eqref{lower-bound-L}}{<}+\infty. 
 \end{array}  
 \end{eqnarray*} 
The above condition means 
 $  \| \triangle \by^{k+1}   \|\rightarrow0$ and $\| \triangle \bx^{k+1}_i  \|\rightarrow0,$ yielding  $ \| \triangle \bpi_i^{k+1}  \| \rightarrow0$ by \eqref{Lip-fi}.   This contributes to
 \begin{eqnarray*}  
 \begin{array}{llll}
 ~~{\lim}_{k\rightarrow\infty} \|\sigma_i (\bx_i^{k+1}-\by^{k+1})\|  
   \overset{ \eqref{ceadmm-sub3}}{ =}{\lim}_{k\rightarrow\infty} \|  \triangle \bpi_i^{k+1} \| =0.
    \end{array}
  \end{eqnarray*}
Since $\tau_k \leq  k <\tau_{k}+k_0$, we have $\by^{k+1}=\bx^{\tau_k+1}=\by^{\tau_k+1}$ from \eqref{x-y-relation}. As a result,  
\begin{eqnarray}  \label{xk-xtauk}
 \arraycolsep=1.4pt\def\arraystretch{1.35} \begin{array}{lcl}
 \| \sigma_i(\bx^{k+1}_i -\bx^{\tau_k+1}_i) \|^2 
  &=& \|\sigma_i( \bx^{k+1}_i - \by^{k+1} + \by^{\tau_k+1} - \bx^{\tau_k+1}_i )\|^2 \\ 
    &\overset{\eqref{two-vecs}}{\leq}& 2\| \sigma_i(\bx^{k+1}_i - \by^{k+1})\|^2 + 2\|\sigma_i(\by^{\tau_k+1} - \bx^{\tau_k+1}_i) \|^2 \\ 
    &  \overset{ \eqref{ceadmm-sub3}}{ =}& 2\| \triangle \bpi^{k+1}_i\|^2 + 2\| \triangle \bpi^{\tau_k+1}_i\|^2, 
    \end{array}  
 \end{eqnarray} 
which together with $\triangle \bpi^{k+1}_i\rightarrow0 $ yields $ (\bx^{k+1}_i -\bx^{\tau_k+1}_i)\rightarrow0$. 
Hence, the whole proof is finished.
\end{proof}

 \subsubsection{Proof of Theorem \ref{global-obj-convergence-exact}}
 \begin{proof} i) It follows from Lemma \ref{L-bounded-decreasing} that $\{\L^k\}$ is non-increasing and bounded from below. Therefore, the whole sequence  $\{\L^k\}$ converges.  Since 
\begin{eqnarray} \label{L-f-pi} 
 \arraycolsep=1.4pt\def\arraystretch{1.35}\begin{array}{lcl}
 \L^{k}- F(X^{k}) 
    &\overset{\eqref{FL-opt-ver11}}{=}&  \sum_{i=1}^{m}  (  \langle\bx_i^{k}-\by^{k}, \bpi_i^{k}\rangle +\frac{\sigma_i}{2}\|\bx_i^{k}-\by^{k}\|^2 ) \\
      & \overset{ \eqref{ceadmm-sub3}}{=}& \sum_{i=1}^{m}  ( \frac{1}{\sigma_i}  \langle \triangle \bpi_i^{k} , \bpi_i^{k}\rangle +\frac{1}{2\sigma_i}\|\triangle \bpi_i^{k} \|^2)    \\
         & {=}&  \sum_{i=1}^{m}  ( \frac{1}{2\sigma_i} \|\bpi_i^{k} \|^2  - \frac{1}{2\sigma_i}\| \bpi_i^{k-1}\|^2+\frac{1}{\sigma_i}\|\triangle \bpi_i^{k} \|^2), 
    \end{array} 
 \end{eqnarray}     
and  $  \triangle \bpi_i^{k}\rightarrow0$ in \eqref{limit-5-term-0}, we obtain 
\begin{eqnarray*} 
\qquad   \begin{array}{lll}{\lim}_{k \rightarrow \infty}   F(X^{k})   ={\lim}_{k \rightarrow \infty}  \L^{k}.   \end{array} 
 \end{eqnarray*}  
It follows from Mean Value Theory that
 $$f_i(\by^{k})=f_i(\bx_i^{k})+\langle\by^{k}-\bx_i^{k} , \nabla f_i(\bx_i^{k}(\alpha))\rangle,$$ 
 where $\bx_i^{k}(\alpha):=(1-\alpha)\by^{k}+\alpha\bx_i^{k}$ for some $\alpha\in(0,1)$. The above relation results in
\begin{eqnarray}  \label{f-x-f-y-exact}
 \arraycolsep=1.4pt\def\arraystretch{1.35}  \begin{array}{lll}
  w_if_i(\bx_i^{k})+\langle  \bx_i^{k}-\by^{k},\bpi_i^{k}\rangle
&\overset{\eqref{opt-con-xik1}}{=}& w_if_i(\bx_i^{k})+\langle\bx_i^{k}-\by^{k}, -\bg_i^{k}\rangle \\
  &=&w_i f_i(\by^{k}) +  \langle \bx_i^{k}-\by^{k},w_i\nabla f_i(\bx_i^{k}(\alpha))-\bg_i^{k} \rangle.   
   \end{array} 
 \end{eqnarray}
 The  gradient Lipschitz continuity of $f_i$  yields that   
 \begin{eqnarray*}  
  \begin{array}{lll}\|w_i\nabla f_i(\bx_i^{k}(\alpha))-\bg_i^{k}\|&\leq& w_ir_i\|\bx_i^{k}(\alpha)-\bx_i^{k} \|
  = w_ir_i(1-\alpha) \|\by^{k}-\bx_i^{k}\|,
     \end{array} 
 \end{eqnarray*}
 which by  $(\by^{k}-\bx_i^{k})\rightarrow0$ in \eqref{limit-5-term-0} leads to
 \begin{eqnarray*}  
  \begin{array}{lll}&&\lim_{k \rightarrow \infty} (w_i\nabla f_i(\bx_i^{k}(\alpha))-\bg_i^{k}) =0.    \end{array} 
 \end{eqnarray*}  
This, \eqref{f-x-f-y-exact} and  $(\by^{k}-\bx_i^{k})\rightarrow0$ bring out
 \begin{eqnarray*}  
  \begin{array}{lll} 
  \lim_{k \rightarrow \infty} (w_i f_i(\bx_i^{k})+\langle \bx_i^{k}-\by^{k}, \bpi_i^{k} \rangle ) = \lim_{k \rightarrow \infty} w_i f_i(\by^{k}).     \end{array} 
 \end{eqnarray*}  
The above fact and  \eqref{limit-5-term-0}  enable us to obtain
\begin{eqnarray*}  
 \arraycolsep=1.4pt\def\arraystretch{1.35} \begin{array}{lll}
 {\lim}_{k \rightarrow \infty}  \L^{k}
    &=&  {\lim}_{k \rightarrow \infty}   \sum_{i=1}^{m}  ( w_i  f_i(\bx_i^{k})+      \langle \bx_i^{k}-\by^{k}, \bpi_i^{k}\rangle +\frac{\sigma_i}{2}\|\bx_i^{k}-\by^{k}\|^2  )  \\
    &=&  {\lim}_{k \rightarrow \infty}   \sum_{i=1}^{m}     w_i  f_i(\by^{k})  =  {\lim}_{k \rightarrow \infty}   f(\by^{k}).
    \end{array} 
 \end{eqnarray*}      
ii) Direct verifications render that
 \begin{eqnarray*} 
 \arraycolsep=1.4pt\def\arraystretch{1.35}  \begin{array}{lcl}
{\lim}_{k \rightarrow \infty} \|\bpi_i^{k+1}-\bpi_i^{\tau_k+1}\|
 &\overset{\eqref{opt-con-xik1}}{=}&   {\lim}_{k \rightarrow \infty} \| -\bg_i^{k+1}+\bg_i^{\tau_k+1}\|  \\
    &\overset{\eqref{Lip-r}}{\leq}&   {\lim}_{k \rightarrow \infty} w_ir_i\| \bx^{k+1}_i-\bx^{\tau_k+1}_i\| \\
    &\overset{\eqref{limit-5-term-0}}{=}&0.
   \end{array}
  \end{eqnarray*} 
Using this condition and the following one
 \begin{eqnarray*} 
  \begin{array}{lcl}
 0 &\overset{\eqref{opt-con-xk1}}{=}&   {\lim}_{\tau_k \rightarrow \infty}  \sum_{i=1}^{m} ( -\bpi_i^{\tau_k+1}+ \sigma_i \triangle\bx_i^{\tau_k+1} ) \overset{\eqref{limit-5-term-0}}{=} {\lim}_{\tau_k \rightarrow \infty}  \sum_{i=1}^{m} - \bpi_i^{\tau_k+1} \end{array}\end{eqnarray*} 
can claim that
  \begin{eqnarray} \label{limit-K-grad-pi-exact}
  \begin{array}{llllllll}
   {\lim}_{k\rightarrow \infty}  \sum_{i=1}^{m}  \bpi_i^{k+1} =0.
    \end{array}\end{eqnarray} 
Taking the limit on both sides of \eqref{opt-con-xik1} gives us
 \begin{eqnarray} \label{limit-K-grad-i-exact}
\arraycolsep=1.4pt\def\arraystretch{1.35}
  \begin{array}{lcl}
0&=&  {\lim}_{k \rightarrow \infty} \sum_{i=1}^{m} ~ (\bg_i^{k+1}   +\bpi_i^{k+1})\\
  &\overset{\eqref{limit-K-grad-pi-exact}}{=}& {\lim}_{k\rightarrow \infty} \sum_{i=1}^{m}  \bg_i^{k+1}= {\lim}_{k \rightarrow \infty}   \nabla F(X^{k+1}), 
   \end{array}
  \end{eqnarray} 
which together with $ (\bx_i^{k+1}-\by^{k+1})\rightarrow0$ and the gradient Lipschitz continuity yields that
 \begin{eqnarray*} 
  \begin{array}{llllllll}
  0 =   {\lim}_{k \rightarrow \infty} \sum_{i=1}^{m}  w_i \nabla f_i(\by^{k+1})={\lim}_{k\rightarrow \infty} \nabla f (\by^{k+1}),
   \end{array}
  \end{eqnarray*} 
 completing the whole proof. \end{proof}
 
 \subsubsection{Proof of Theorem \ref{global-convergence-exact} }  
 \begin{proof} 
i) It follows from Lemma \ref{L-bounded-decreasing}  i) and \eqref{lower-bound-L} that
\begin{eqnarray} 
  \begin{array}{lll}
\L^{0} \geq  \L^{k+1}  \geq   \sum_{i=1}^{m}    w_i  f_i(\by^{k+1}) = f(\by^{k+1}), 
    \end{array} 
 \end{eqnarray}
which  implies $\by^{k+1}\in\S(\L^{0})$ and hence  $\{\by^{k+1}\}$ is bounded due to  the boundedness of $\S(\L^{0})$.  This calls forth the boundedness of $\{\bx_i^{k+1}\}$ as $ (\bx_i^{k+1}-\by^{k+1})\rightarrow0$ from \eqref{limit-5-term-0}. The boundedness of $\{\bpi_i^{k+1}\}$ can be ensured because of
  \begin{eqnarray*} 
 \begin{array}{lcl}
 \|\bpi_i^{k+1}\|  \overset{\eqref{opt-con-xik1}}{=}   \|\bg_i^{k+1}\|    \leq   \|\bg_i^{k+1} - \bg_i^{0}\|+ \|\bg_i^{0}\|   
  \overset{\eqref{Lip-r}}{\leq}w_i r_i \|  \bx^{k+1}_i-\bx^{0}_i\|+ \|\bg_i^{0}\| <+\infty,  
     \end{array} 
 \end{eqnarray*}
where `$<$' is from the boundedness of $\{\bx_i^{k+1}\}$. Overall,  sequence $\{(\by^{k+1},X^{k+1},\Pi^{k+1})\}$ is bounded. Let $(\by^{\infty},X^{\infty},\Pi^{\infty})$ be any accumulating point of the sequence, 
 it follows from \eqref{limit-K-grad-i-exact}, \eqref{limit-K-grad-pi-exact},  and $ (\bx_i^{k+1}-\by^{k+1})\rightarrow0$ that 
 \begin{eqnarray*}
  \begin{array}{l}
0= w_i \nabla f_i(\bx^{\infty}_i) + \bpi_i^{\infty}=  \sum_{i=1}^{m} \bpi_i^{\infty}=\bx^{\infty}_i- \by^{\infty}.
   \end{array}
  \end{eqnarray*} 
  Therefore, recalling \eqref{opt-con-FL-opt-ver1},  $(\by^{\infty},X^{\infty},\Pi^{\infty})$  is a stationary point of  problem  (\ref{FL-opt-ver1}) and $\by^{\infty}$ is a stationary point of  problem  (\ref{FL-opt}). 
  
It follows from \cite[Lemma 4.10]{more1983computing},  $\triangle \by^{k} \rightarrow0$ and $\by^{\infty}$ being isolated that  the whole sequence, $\{\by^{k}\}$ converges to $\by^{\infty}$, which by $(\by^k-\bx_i^k)\rightarrow0$ implies that sequence $\{X^{k}\}$ also converges to $X^{\infty}$. Finally, this and $\bpi_i^k=-\bg_i^{k}$ result in the convergence of  sequence  $\{\Pi^k\}$.
  \end{proof}

 \subsubsection{Proof of Corollary \ref{L-global-convergence}}
 
\begin{proof} i) The convexity of $f$ and the optimality of $\bx^*$ lead to
    \begin{eqnarray}  \label{convexity-optimality}
   \begin{array}{lll}
f(\by^{k}) \geq  f(\bx^{*}) \geq  f(\by^{k}) + \langle \nabla f(\by^{k}), \bx^{*}-\by^{k} \rangle. 
    \end{array} 
 \end{eqnarray} 
Theorem \ref{global-obj-convergence-exact} ii) states that   
   \begin{eqnarray*}  
   \begin{array}{lll}
 {\lim}_{k \rightarrow \infty} \nabla F(X^{k})  ={\lim}_{k \rightarrow\infty} \nabla f(\by^{k}) =0.
    \end{array} 
 \end{eqnarray*} 
 Using this and the boundedness of $\{\by^k\}$ from Theorem \ref{global-convergence-exact}, we take the limit of both sides of \eqref{convexity-optimality} to derive that  $ f(\by^{k})\rightarrow f(\bx^{*})$, which recalling Theorem  \ref{global-obj-convergence-exact} i) yields \eqref{L-global-convergence-limit}. 
 
 ii) The conclusion follows from Theorem  \ref{global-convergence-exact} ii) and the fact that  the stationary points are equivalent to optimal solutions if $f$ is convex.
 
 iii)   The strong convexity of $f$ means that
 there is a positive constance $\nu$ such that 
 \begin{eqnarray*}  
  \begin{array}{llll}
  f( \by^k) -f( \bx^*)
  \geq  \langle \nabla f( \bx^*), \by^k-\bx^*\rangle + \frac{\nu}{2}\|\by^k-\bx^*\|^2=  \frac{\nu}{2}\|\by^k-\bx^*\|^2, 
    \end{array} 
 \end{eqnarray*}
where the equality is due to \eqref{grad-x-*=0}. Taking limit of both sides of the above inequality  shows $ \by^k\rightarrow\bx^*$ since $ f( \by^k) \rightarrow f( \bx^*)$. This together with \eqref{limit-5-term-0} yields $ \bx_i^k\rightarrow\bx^*$. Finally, $ \bpi_i^k\rightarrow\bpi_i^*$ because of   $$\|\bpi_i^{k}-\bpi_i^{*}\|\overset{\eqref{opt-con-FL-opt-ver1}}{=}\|\bg_i^{k}-w_i \nabla f_i(\bx^*)\| \overset{\eqref{Lip-r}}{\leq}  w_ir_i\| \bx^{k}_i-\bx^*\|,$$ 
 displaying the desired result.
\end{proof}
 
 \begin{lemma}\label{grad-L-bounded}  Let $\{(\by^{k},X^{k},\Pi^{k})\}$ be the sequence generated by Algorithm \ref{algorithm-CEADMM} with $\sigma_i> 2w_ir_i$ for every $i\in[m]$. If Assumption \ref{ass-fi} holds, then   
\begin{eqnarray}  \label{grad-L-bounded-eq}
  \begin{array}{lllll}
\max\{\|\nabla  F(X^{k+1})\|^2,  \|\nabla  f(\by^{k+1})\|^2\}
&\leq&\sum_{i=1}^m   {2m\sigma_i^2}  (\| \triangle \bx^{k+1}_i\|^2 + \|\triangle\bx_i^{\tau_k+1}\|^2). 
    \end{array}  
 \end{eqnarray} 
\end{lemma}  
\begin{proof}
 Following the fact \eqref{Lip-fi}, for any $j\geq1$, there is
  \begin{eqnarray}\label{Lip-dpi-dxi} 
  \begin{array}{lllll}\| \triangle \bpi_i^{j+1} \|^2    \leq  w_i^2 r_i^2\| \triangle\bx_i^{j+1}\|^2 \leq \frac{\sigma_i^2}{4}\| \triangle\bx_i^{j+1}\|^2.
\end{array}   \end{eqnarray}
This results in
\begin{eqnarray} \label{gradient-of-xk-xtuak}
 \arraycolsep=1.4pt\def\arraystretch{1.35}  \begin{array}{lcl}
  \| \nabla F(X^{k+1})-\nabla F(X^{\tau_k+1}) \|^2
  &=&\|\sum_{i=1}^m  (\bg_i^{k+1}-\bg_i^{\tau_k+1})  \|^2  \\ 
 &\overset{\eqref{Lip-r},\eqref{two-vecs}}{\leq}&m\sum_{i=1}^m   w_i^2r_i^2\| \bx^{k+1}_i -\bx^{\tau_k+1}_i\|^2 \\
  &\overset{\sigma_i>2w_ir_i}{\leq}&m\sum_{i=1}^m \frac{\sigma_i^2}{4}\| \bx^{k+1}_i -\bx^{\tau_k+1}_i\|^2\\
 &\overset{\eqref{xk-xtauk}}{\leq}&m\sum_{i=1}^m   (\frac{1}{2} \| \triangle \bpi^{k+1}_i\|^2 + \frac{1}{2} \| \triangle \bpi^{\tau_k+1}_i\|^2  )  \\
 &\overset{ \eqref{Lip-dpi-dxi}}{\leq}&m\sum_{i=1}^m   (\frac{\sigma_i^2}{8} \| \triangle \bx^{k+1}_i\|^2 + \frac{\sigma_i^2}{8} \| \triangle \bx^{\tau_k+1}_i\|^2  ).
    \end{array}  
 \end{eqnarray}
The optimality conditions contribute to 
  \begin{eqnarray}\label{condition-F-dxi} 
 \arraycolsep=1.4pt\def\arraystretch{1.35}  \begin{array}{lcl}
\|\nabla F(X^{\tau_k+1})\|^2&=&\|\sum_{i=1}^m  \bg_i^{\tau_k+1}\|^2\\
&\overset{\eqref{opt-con-xik1}}{=}&\|-\sum_{i=1}^m  \bpi^{\tau_k+1}_i  \|^2\\
&\overset{\eqref{opt-con-xk1}}{=}&  \|- \sum_{i=1}^m  \sigma_i \triangle \bx_i^{\tau_k+1}\|^2\\
&\overset{\eqref{two-vecs}}{\leq}&  m\sum_{i=1}^m  \sigma_i^2 \|  \triangle \bx_i^{\tau_k+1}\|^2
   \end{array}
  \end{eqnarray}
Now the above two facts \eqref{gradient-of-xk-xtuak} and \eqref{condition-F-dxi} allow us to derive that, for any $k$,
    \begin{eqnarray*} 
 \arraycolsep=1.4pt\def\arraystretch{1.35}  \begin{array}{lcl}
\| \nabla F(X^{k+1})\|^2
 &\overset{\eqref{two-vecs}}{\leq}& 3\| \nabla F(X^{k+1})-\nabla F(X^{\tau_k+1}) \|^2 +  \frac{3}{2}\| \nabla F(X^{\tau_k+1})\|^2\\
 &\leq& m\sum_{i=1}^m   (\frac{3\sigma_i^2}{8} \| \triangle \bx^{k+1}_i\|^2 + \frac{15\sigma_i^2}{8} \| \triangle \bx^{\tau_k+1}_i\|^2  ) \\
 &\leq&  m\sum_{i=1}^m   { 2\sigma_i^2}  (\| \triangle \bx^{k+1}_i\|^2 +  \| \triangle \bx^{\tau_k+1}_i\|^2  ).
   \end{array}
  \end{eqnarray*}  
Direct verifications can deliver that
 \begin{eqnarray}\label{grad-yk1-xk1}
 \arraycolsep=1.4pt\def\arraystretch{1.35}  \begin{array}{lcl}
\|   \nabla f  (\by^{\tau_k+1})- \nabla F(X^{\tau_k+1})\|^2
&=&\| \sum_{i=1}^m  (w_i  \nabla f_i (\by^{\tau_k+1})-\bg_i^{\tau_k+1})\|^2\\  
 &\overset{\eqref{Lip-r},\eqref{two-vecs}}{\leq}& m \sum_{i=1}^m  w_i^2r_i^2 \|   \by^{\tau_k+1} -   \bx^{\tau_k+1}_i \|^2  \\
  &\overset{\sigma_i>2w_ir_i}{\leq}& \frac{m}{4} \sum_{i=1}^m  \sigma_i^2 \|   \by^{\tau_k+1} -   \bx^{\tau_k+1}_i \|^2 \\
  &\overset{\eqref{ceadmm-sub3}}{=}& \frac{m}{4} \sum_{i=1}^m    \|  \triangle \bpi^{\tau_k+1}_i \|^2  \\ 
&\overset{\eqref{Lip-dpi-dxi}}{\leq}& \frac{m}{16}   \sum_{i=1}^m   \sigma_i^2  \|  \triangle \bx^{\tau_k+1}_i \|^2. 
   \end{array}
  \end{eqnarray} 
Since $\by^{k+1}=\bx^{\tau_k+1}=\by^{\tau_k+1}$ from \eqref{x-y-relation}, the above two conditions \eqref{grad-yk1-xk1} and \eqref{condition-F-dxi} contribute to
    \begin{eqnarray*}
 \arraycolsep=1.4pt\def\arraystretch{1.35}  \begin{array}{lcl}
  \| \nabla f (\by^{k+1})\|^2  &= & \| \nabla f (\by^{\tau_k+1})\|^2\\
&\overset{\eqref{two-vecs}}{\leq}& 5\|   \nabla f  (\by^{\tau_k+1})- \nabla F(X^{\tau_k+1})\|^2+\frac{5}{4}\|  \nabla F(X^{\tau_k+1})\|^2\\
   &\leq&  \sum_{i=1}^m \frac{25m\sigma_i^2}{16}  \|  \triangle \bx^{\tau_k+1}_i \|^2 \\
   &\leq&   \sum_{i=1}^m  {2m\sigma_i^2} \|  \triangle \bx^{\tau_k+1}_i \|^2,  
   \end{array}
  \end{eqnarray*}   
 finishing the proof.
\end{proof}
 \subsubsection{Proof of Theorem \ref{complexity-thorem-gradient} }
 \begin{proof}
 It follows from \eqref{grad-L-bounded-eq} and \eqref{decreasing-property-0} that
 \begin{eqnarray} \label{fact-max}
   \arraycolsep=1.4pt\def\arraystretch{1.35}
  \begin{array}{lcl}
\max\{\|\nabla  F(X^{j+1})\|^2,  \|\nabla  f(\by^{j+1})\|^2\}
&\overset{\eqref{grad-L-bounded-eq}}{\leq}&\sum_{i=1}^m    {2m\sigma_i^2}  (\| \triangle \bx^{j+1}_i\|^2 + \|\triangle\bx_i^{\tau_j+1}\|^2)\\
&=&\sum_{i=1}^m  \frac{4m\sigma_i^2}{\theta_i} \frac{\theta_i}{2} (\| \triangle \bx^{j+1}_i\|^2 + \|\triangle\bx_i^{\tau_j+1}\|^2)\\
&\leq& \frac{\rho}{2}\sum_{i=1}^m    \frac{\theta_i}{2} (\| \triangle \bx^{j+1}_i\|^2 + \|\triangle\bx_i^{\tau_j+1}\|^2)\\
&\overset{\eqref{decreasing-property-0}}{\leq}& \frac{\rho}{2} (\L ^j-\L ^{j+1} + \L ^{\tau_j}-\L ^{\tau_j+1}).
    \end{array}  
 \end{eqnarray}
 We note that $\{\L ^{k}\}$ is non-increasing from Lemma \ref{L-bounded-decreasing} i), leading to  $\L ^{tk_0  +1} \geq \L ^{(t+1)k_0}$ for any $t\geq0$.  By letting $s:= \lfloor (k-1)/k_0\rfloor+1$, we have 
   \begin{eqnarray} \label{fact-sum-0K}
 \arraycolsep=1.4pt\def\arraystretch{1.35}  \begin{array}{lllll}
 \sum_{j=0}^{k-1}(\L ^{\tau_j}-\L ^{\tau_j+1})
&\leq& k_0 \sum_{t=0}^{s}( \L ^{tk_0}-\L ^{tk_0+1})\\
&=&  k_0 [\L ^{0}-\sum_{t=0}^{s-1}(\L ^{tk_0  +1}-\L ^{(t+1)k_0} )-\L ^{s k_0  +1}]\\
&\leq&  k_0  (\L ^{0}- \L ^{s k_0  +1} )\\
    \end{array}  
 \end{eqnarray}
Using the above two facts brings out 
  \begin{eqnarray*} 
   \arraycolsep=1pt\def\arraystretch{1.35}
  \begin{array}{lcl}
 && \min_{j=0,1,\ldots,k-1} \max\{\|\nabla  F(X^{j+1})\|^2,  \|\nabla  f(\by^{j+1})\|^2\}\\
  &\leq&
\frac{1}{k}\sum_{j=0}^{k-1}\max\{\|\nabla  F(X^{j+1})\|^2,  \|\nabla  f(\by^{j+1})\|^2\}\\
&\overset{\eqref{fact-max}}{\leq}& \frac{\rho}{2k}\sum_{j=0}^{k-1}(\L ^j-\L ^{j+1} + \L ^{\tau_j}-\L ^{\tau_j+1})\\
&\overset{\eqref{fact-sum-0K}}{\leq}& \frac{\rho}{2k} (\L ^0-\L ^{k})+  \frac{\rho k_0}{k}  (\L ^{0}- \L ^{s k_0  +1} ) \\
&\leq&  \frac{\rho(1+k_0)}{2k} (\L ^0-f^*)\\
&\leq&  \frac{\rho k_0}{k} (\L ^0-f^*), 
    \end{array}  
 \end{eqnarray*}
 where the fourth inequality used   $ \L^0\geq \L^k \geq f(\by^k) \geq f^*$ from Lemma \ref{L-bounded-decreasing}.
 \end{proof}

\subsection{Proofs of all theorems in Section \ref{sec:iceadmm}} 
 
 \begin{lemma}\label{lemma-decreasing-1} Let  $\{(\by^{k},X^{k},\Pi^{k})\}$ be the sequence generated by Algorithm \ref{algorithm-ICEADMM} with $H_i=\Theta_i$  for every $i\in[m]$. If Assumption \ref{ass-fi} holds, then for any $k\geq1$, 
\begin{eqnarray} 
 \label{decreasing-property-1}   
 \begin{array}{lll}
  \varphi^{k+1}-  \varphi^k & \leq& -\frac{\sigma}{2}\| \triangle \by^{k+1} \|^2 - \sum_{i=1}^{m} \frac{\vartheta_i}{2}\| \triangle\bx^{k+1}_i \|^2,
    \end{array}  
 \end{eqnarray} 
 where $\varphi^k$ and $\vartheta_i   $ are given by
 \begin{eqnarray} 
 \label{def-ci-1}  
 \arraycolsep=1.4pt\def\arraystretch{1.35} \begin{array}{lll} 
 \varphi^k&:=&\L^{k}  +  \sum_{i=1}^{m} \frac{6w_i^2 r_i^2}{\sigma_i}   \|  \triangle \bx_i^{k} \|^2,\\
 \vartheta_i&:=&   \sigma_i  - {18w_i^2r_i^2}/{\sigma_i}, ~~ i\in[m].
 \end{array}  \end{eqnarray}
\end{lemma}  
\begin{proof}  Similar to \eqref{opt-con-xk1}, the optimality condition  for \eqref{iceadmm-sub1} is
 \begin{eqnarray}
 \label{opt-con-xk1-11} 
  \begin{array}{lllll}
\forall~k\in\K:~~
 0 =   \sum_{i=1}^{m} (  -\bpi_i^{k+1} + \sigma_i \triangle\bx_i^{k+1} ), 
   \end{array}
  \end{eqnarray} 
And the optimality condition  for   \eqref{iceadmm-sub2}  is, for any $k\geq0$ 
 \begin{eqnarray}
 \label{opt-con-xik1-1}
 0 &=&  \bg_i^{k} +w_i\Theta_i  \triangle\bx_i^{k+1} +\bpi_i^k + \sigma_i (\bx_i^{k+1}-\by^{k+1})\nonumber\\ 
   &=& \bg_i^{k} +w_i\Theta_i \triangle\bx_i^{k+1} +\bpi_i^{k+1},~~ \forall~i\in[m].
  \end{eqnarray} 
  
Again, we decompose the gap in \eqref{three-cases} as  $$\L^{k+1}-\L^{k}  =: e_1^k+e_2^k+e_3^k,$$ 
 where $e_1^k,e_2^k,e_3^k$ are given by \eqref{three-cases-sub}. Same reasoning to \eqref{two-cases-1} and \eqref{two-cases-2}  allows us to show  that
\begin{eqnarray} 
\label{two-cases-3} \begin{array}{lll}  
   e_1^k \leq - \frac{\sigma}{2}\|\triangle \by^{k+1}\|^2. 
   \end{array}
 \end{eqnarray} 
To estimate $e_2^k$, we have the following chain of inequalities,
\begin{eqnarray*} 
 \arraycolsep=1.4pt\def\arraystretch{1.35}\begin{array}{lcl}
 &&  L(\by^{k+1},\bx_i^{k+1},\bpi_i^{k})  - L(\by^{k+1},\bx_i^{k},\bpi_i^{k})   \\ 
     &\overset{\eqref{Def-L}}{=}&         w_i  f_i(\bx_i^{k+1})- w_i f_i(\bx_i^k)   +  \langle \triangle \bx_i^{k+1} , \bpi_i^k\rangle  +         \frac{\sigma_i}{2}\|\bx_i^{k+1}-\by^{k+1}\|^2 -\frac{\sigma_i}{2}\|\bx_i^{k}-\by^{k+1}\|^2    \\
 &\overset{\eqref{grad-lip-theta}}{\leq}&        \langle \triangle\bx_i^{k+1}, \bg_i^{k} + \bpi_i^k\rangle  + \frac{w_i}{2}\| \triangle\bx_i^{k+1}\|^2_{\Theta_i}  +   \frac{\sigma_i}{2}\|\bx_i^{k+1}-\by^{k+1}\|^2 -\frac{\sigma_i}{2}\|\bx_i^{k}-\by^{k+1}\|^2    \\  
 &\overset{\eqref{opt-con-xik1-1}}{\leq}&                 \langle \triangle\bx_i^{k+1}, - w_i\Theta_i      \triangle\bx_i^{k+1}  - \sigma_i ( \bx_i^{k+1}-\by^{k+1})\rangle   \\   
    &+& \frac{w_i}{2}\| \triangle\bx_i^{k+1}\|^2_{\Theta_i}  +     \frac{\sigma_i}{2}\|\bx_i^{k+1}-\by^{k+1}\|^2 -\frac{\sigma_i}{2}\|\bx_i^{k}-\by^{k+1}\|^2     \\  
 &\overset{\eqref{two-vecs}}{\leq}&                -\frac{w_i}{2}\| \triangle\bx_i^{k+1}\|^2_{\Theta_i} - \frac{\sigma_i}{2}\|\triangle\bx_i^{k+1}\|^2,
    \end{array} 
 \end{eqnarray*}
 which  by \eqref{FL-opt-ver11} delivers an upper bound for $e_2^k$ as
 \begin{eqnarray} 
\label{gap-2-1}  
 \arraycolsep=1.4pt\def\arraystretch{1.35} \begin{array}{lll}
   e_2^k&=&  \sum_{i=1}^{m}     (  L(\by^{k+1},\bx_i^{k+1},\bpi_i^{k})  - L(\by^{k+1},\bx_i^{k},\bpi_i^{k})) \\
    &\leq& \sum_{i=1}^{m}     (  -\frac{w_i}{2}\| \triangle\bx_i^{k+1}\|^2_{\Theta_i} - \frac{\sigma_i}{2}\|\triangle\bx_i^{k+1}\|^2). 
    \end{array} 
 \end{eqnarray}
To estimate $e_3^k$,  the  gradient Lipschitz continuity of $f_i$   with  $r_i>0$ and \eqref{opt-con-xik1-1}  deliver that
 \begin{eqnarray}
 \label{Lip-fi-1} 
 \arraycolsep=1.4pt\def\arraystretch{1.35}
 \begin{array}{lcl}
 \|\triangle \bpi_i^{k+1} \|^2 
 & =&  \| - \bg_i^{k}- w_i\Theta_i      \triangle\bx_i^{k+1} +\bg_i^{k-1}+ w_i\Theta_i      \triangle\bx_i^{k} \|^2\\
  &\overset{\eqref{two-vecs}}{\leq}&  3 \| \bg_i^{k}-  \bg_i^{k-1} \|^2+3\|  w_i\Theta_i      \triangle\bx_i^{k+1}   \|^2+3\|   w_i\Theta_i      \triangle\bx_i^{k} \|^2\\ 
&\overset{\eqref{Lip-r} }{\leq}& 3 w_i^2r_i^2(\|  \triangle\bx_i^{k} \|^2+ \|  \triangle\bx_i^{k}\|^2 +  \|  \triangle\bx_i^{k+1}\|^2 ) \\
&=& 3w_i^2r_i^2( 2\|  \triangle\bx_i^{k} \|^2-2 \|  \triangle\bx_i^{k+1}\|^2+3\|  \triangle\bx_i^{k+1}\|^2), 
   \end{array}
  \end{eqnarray} 
  where `$\geq$' also used a fact  $r_iI\succeq\Theta_i$.
This condition  and \eqref{gap-3} give  rise to
 \begin{eqnarray} 
\label{gap-3-1}  
 \arraycolsep=1.4pt\def\arraystretch{1.35}  \begin{array}{lcl}
   e_3^k   \overset{\eqref{gap-3}}{=}   \sum_{i=1}^{m}     \frac{1}{\sigma_i}\|\triangle \bpi_i^{k+1} \|^2  \leq  \sum_{i=1}^{m}    \frac{6w_i^2r_i^2}{\sigma_i} ( \|   \triangle\bx_i^{k}\|^2 -  \| \triangle\bx_i^{k+1} \|^2)+\sum_{i=1}^{m}    \frac{9w_i^2r_i^2}{\sigma_i}\|  \triangle\bx_i^{k+1} \|^2. 
    \end{array} 
 \end{eqnarray}
 Overall, combining \eqref{three-cases}, \eqref{two-cases-3},   \eqref{gap-2-1} and \eqref{gap-3-1}, we obtain
  \begin{eqnarray*}  
   \arraycolsep=1.4pt\def\arraystretch{1.3}
  \begin{array}{lll}
   \L^{k+1}-\L^{k}
    &\leq&    - \frac{\sigma}{2}\|\triangle \by^{k+1} \|^2 - \sum_{i=1}^{m} (\frac{\sigma_i}{2}-\frac{9w_i^2r_i^2}{\sigma_i})\|\triangle\bx_i^{k+1}\|^2\\
    &+& \sum_{i=1}^{m}    \frac{6w_i^2r_i^2}{\sigma_i} (\|  \triangle\bx_i^{k}\|^2 -  \|   \triangle\bx_i^{k+1}\|^2) ,
    \end{array} 
 \end{eqnarray*}
which after simple manipulations displays the result. 
\end{proof}
\begin{lemma}\label{L-bounded-decreasing-1}  Let  $\{(\by^{k},X^{k},\Pi^{k})\}$ be the sequence generated by Algorithm \ref{algorithm-CEADMM} with $H_i=\Theta_i$ and $\sigma_i>3\sqrt{2}w_ir_i$ for every $i\in[m]$. The following results hold under Assumption \ref{ass-fi}.
 \begin{itemize}
 \item[i)] $\{ \varphi^k\}$ is non-increasing. 
 \item[ii)] $ \varphi^k\geq f(\by^k)\geq f^* >-\infty$ for any $ k\geq1$.
 \item[iii)]The limits of all the following terms are zero, namely,  
\begin{eqnarray} \label{limit-5-term-0-inexact}
 \arraycolsep=1.4pt\def\arraystretch{1.35} \begin{array}{lll}\underset{k \rightarrow \infty}{\lim}  \triangle \by^{k+1}&=& \underset{k \rightarrow \infty}{\lim} \triangle \bx^{k+1}_i=\underset{k \rightarrow \infty}{\lim} \triangle \bpi_i^{k+1}\\
 &=&\underset{k \rightarrow \infty}{\lim} (\bx_i^{k+1}-\by^{k+1})=\underset{k \rightarrow \infty}{\lim} (\bx_i^{k+1}-\bx_i^{\tau_k+1})=0.
    \end{array} 
 \end{eqnarray}
 \end{itemize}
\end{lemma}  
\begin{proof}
The proof is the same as that of Lemma \ref{L-bounded-decreasing} and hence omitted here.
\end{proof}
 \subsubsection{Proof of Theorem \ref{global-obj-convergence-inexact}}
 \begin{proof} i) It follows from Lemma \ref{L-bounded-decreasing-1} that $\{\varphi^k\}$ is non-increasing. Same reasoning to show \eqref{lower-bound-L} also allows for showing that sequence $\{\L^{k}\}$ is bounded from below. This together with 
   \begin{eqnarray*} 
  \begin{array}{llll}\varphi^k=\L^{k}   + \sum_{i=1}^{m} \frac{6w_i^2r_i^2}{\sigma_i}   \| \triangle\bx_i^{k}\|^2   \end{array} 
  \end{eqnarray*} leads to the  boundedness   of $\{\varphi^k\}$. Therefore, the whole sequence,  $\{\varphi^k\}$, converges and
  \begin{eqnarray} \label{converge-varphi}
  \begin{array}{llll}
 \varphi^\infty:= {\lim}_{k \rightarrow \infty} \varphi^k   = {\lim}_{k \rightarrow \infty} \L^{k}
   \end{array} 
  \end{eqnarray}
by  $ \triangle \bx_i^{k}\rightarrow0$ in \eqref{limit-5-term-0-inexact}. The remaining proofs  of i) and ii) are similar to those to prove Theorem \ref{global-obj-convergence-exact} i) and ii), and hence omitted here.
\end{proof}
 \begin{lemma}\label{grad-L-bounded-inexact}  Let $\{(\by^{k},X^{k},\Pi^{k})\}$ be the sequence generated by Algorithm \ref{algorithm-ICEADMM} with  $H_i=\Theta_i$ and $\sigma_i>3\sqrt{2}w_ir_i$ for every $i\in[m]$.  Suppose Assumption \ref{ass-fi} holds. If $k_0\geq2$, then   
\begin{eqnarray}  \label{grad-L-bounded-eq-inexact}
 \arraycolsep=1.4pt\def\arraystretch{1.35}  \begin{array}{lllll}
&&\max\{\|\nabla  F(X^{k+1})\|^2,  \|\nabla  f(\by^{k+1})\|^2\}\\
&\leq&\sum_{i=1}^m  3m\sigma_i^2 (\| \triangle \bx^{k}_i\|^2 + \| \triangle \bx^{k+1}_i\|^2 +   \| \triangle \bx^{\tau_k}_i\|^2  +   \| \triangle \bx^{\tau_k+1}_i\|^2).
    \end{array}  
 \end{eqnarray}
 If $k_0=1$, then   
\begin{eqnarray}  \label{grad-L-bounded-eq-inexact-1}
 \arraycolsep=1.4pt\def\arraystretch{1.35}  \begin{array}{lllll}
\max\{\|\nabla  F(X^{k+1})\|^2,  \|\nabla  f(\by^{k+1})\|^2\} \leq\sum_{i=1}^m  3m\sigma_i^2 (\| \triangle \bx^{k}_i\|^2 + \| \triangle \bx^{k+1}_i\|^2).
    \end{array}  
 \end{eqnarray}
\end{lemma}  
\begin{proof} We first prove the case of $k_0\geq2$.  Following  \eqref{Lip-fi-1}  and 
   \begin{eqnarray}\label{sigma-w-r-1} 
  \begin{array}{lllll}
  \sigma_i>3\sqrt{2}w_ir_i,
\end{array}   \end{eqnarray} 
  for any $j\geq1$, there is
  \begin{eqnarray}\label{Lip-dpi-dxi-inexact} 
 \arraycolsep=1.4pt\def\arraystretch{1.35}  \begin{array}{lllll}\| \triangle \bpi_i^{j+1} \|^2    &\leq&  6 w_i^2r_i^2\|   \triangle\bx_i^{j}\|^2 +   3w_i^2r_i^2\| \triangle\bx_i^{j+1} \|^2  \\
  &\leq& \frac{\sigma_i^2}{3}(\|   \triangle\bx_i^{j}\|^2+ \| \triangle\bx_i^{j+1}\|^2),~~\forall~j\geq1.
\end{array}   \end{eqnarray}
Same reasoning to prove \eqref{xk-xtauk} also allows us to derive  
\begin{eqnarray}  \label{xk-xtauk-inexact}
\begin{array}{lll}
 \| \sigma_i(\bx^{k+1}_i -\bx^{\tau_k+1}_i) \|^2 
  &\leq&   2\| \triangle \bpi^{k+1}_i\|^2 + 2\| \triangle \bpi^{\tau_k+1}_i\|^2,  
    \end{array}  
 \end{eqnarray} 
Using the above two facts immediately gives rise to  
\begin{eqnarray} \label{gradient-of-xk-xtuak-inexact}
 \arraycolsep=1.4pt\def\arraystretch{1.35}  \begin{array}{lcl}
  \| \nabla F(X^{k+1})-\nabla F(X^{\tau_k+1} )  \|^2  
  &=&\|\sum_{i=1}^m  (\bg_i^{k+1}-\bg_i^{\tau_k+1})  \|^2  \\
 &\overset{\eqref{Lip-r}}{\leq}&m\sum_{i=1}^m   w_i^2r_i^2\| \bx^{k+1}_i -\bx^{\tau_k+1}_i\|^2\\
  &\overset{\eqref{sigma-w-r-1}}{\leq}&m\sum_{i=1}^m \frac{\sigma_i^2}{18}\| \bx^{k+1}_i -\bx^{\tau_k+1}_i\|^2\\
 &\overset{\eqref{xk-xtauk-inexact}}{\leq}&m\sum_{i=1}^m   (\frac{1}{9} \| \triangle \bpi^{k+1}_i\|^2 + \frac{1}{9} \| \triangle \bpi^{\tau_k+1}_i\|^2  )\\
&\overset{\eqref{Lip-dpi-dxi-inexact}}{\leq}&m\sum_{i=1}^m    \frac{\sigma_i^2}{27} (\| \triangle \bx^{k}_i\|^2 + \| \triangle \bx^{k+1}_i\|^2 +   \| \triangle \bx^{\tau_k}_i\|^2 +  \| \triangle \bx^{\tau_k+1}_i\|^2  ).
    \end{array}  
 \end{eqnarray}
For any  $j\geq0$, replacing $k+1$ by $j+1$ and  $\tau_k+1$ by $j$  in the above formula leads to
 \begin{eqnarray} \label{gradient-of-j-inexact}
  \begin{array}{lllll}
 \| \nabla F(X^{j+1})-\nabla F(X^{j} )  \|^2   & \leq m\sum_{i=1}^m \frac{\sigma_i^2}{18}\| \triangle \bx^{j+1}_i  \|^2.
    \end{array}  
 \end{eqnarray}
 The optimality conditions in \eqref{opt-con-xk1-11}  and \eqref{opt-con-xik1-1} contribute to 
  \begin{eqnarray*}
  \begin{array}{lllll}
\nabla F(X^{\tau_k}) &=&  \sum_{i=1}^m  \bg_i^{\tau_k}=& - \sum_{i=1}^m  (\sigma_i \triangle \bx_i^{\tau_k+1}+w_i\Theta_i  \triangle\bx_i^{\tau_k+1}),
   \end{array}
  \end{eqnarray*}
  which  results in
  \begin{eqnarray}\label{condition-F-dxi-inexact}
  \begin{array}{lcl}
\|\nabla F(X^{\tau_k})\|^2    &\leq&  m\sum_{i=1}^m    ( \sigma_i+w_ir_i )^2\|  \triangle \bx_i^{\tau_k+1}\|^2 \overset{\eqref{sigma-w-r-1}}{\leq} m\sum_{i=1}^m     \frac{28 \sigma_i^2}{18}\|  \triangle \bx_i^{\tau_k+1}\|^2.
   \end{array}
  \end{eqnarray}
Now the above facts in \eqref{gradient-of-xk-xtuak-inexact}, \eqref{gradient-of-j-inexact} and \eqref{condition-F-dxi-inexact} calls forth
    \begin{eqnarray} \label{nabla-F-Xk}
 \arraycolsep=1.4pt\def\arraystretch{1.35}  \begin{array}{lcl}
\| \nabla F(X^{k+1})\|^2 
 &\overset{\eqref{two-vecs}}{\leq}& 5\| \nabla F(X^{k+1})-\nabla F(X^{\tau_k}) \|^2 +   \frac{5}{4}\| \nabla F(X^{\tau_k})\|^2\\
&\overset{\eqref{two-vecs}}{\leq}& 10\| \nabla F(X^{k+1})-\nabla F(X^{\tau_k+1})   \|^2\\
& +&  10\|   \nabla F(X^{\tau_k+1})-\nabla F(X^{\tau_k}) \|^2 + \frac{5}{4}\| \nabla F(X^{\tau_k})\|^2\\
&\leq&m\sum_{i=1}^m  \frac{10\sigma_i^2}{27} (\| \triangle \bx^{k}_i\|^2 + \| \triangle \bx^{k+1}_i\|^2 +   \| \triangle \bx^{\tau_k}_i\|^2)+m\sum_{i=1}^m    \frac{155\sigma_i^2}{54} \| \triangle \bx^{\tau_k+1}_i\|^2  \\
&\leq&m\sum_{i=1}^m  3\sigma_i^2 (\| \triangle \bx^{k}_i\|^2 + \| \triangle \bx^{k+1}_i\|^2 +  \| \triangle \bx^{\tau_k}_i\|^2  +   \| \triangle \bx^{\tau_k+1}_i\|^2).
   \end{array}
  \end{eqnarray}  
Same reasoning to show \eqref{grad-yk1-xk1} also enables us to obtain 
   \begin{eqnarray}\label{grad-yk1-xk1-inexact}
  \arraycolsep=1.4pt\def\arraystretch{1.35}
  \begin{array}{lllll}
\|   \nabla f  (\by^{\tau_k+1})- \nabla F(X^{\tau_k+1})\|^2
&\overset{\eqref{grad-yk1-xk1}}{\leq}& \frac{m}{18} \sum_{i=1}^m    \|  \triangle \bpi^{\tau_k+1}_i \|^2\\
&\overset{\eqref{Lip-dpi-dxi-inexact}}{\leq}& m\sum_{i=1}^m  \frac{\sigma_i^2}{54}(\|   \triangle\bx_i^{\tau_k}\|^2+ \| \triangle\bx_i^{\tau_k+1}\|^2).
   \end{array}
  \end{eqnarray} 
 Since $\by^{k+1}=\bx^{\tau_k+1}=\by^{\tau_k+1}$ from \eqref{x-y-relation}and the above three facts in \eqref{grad-yk1-xk1-inexact}, \eqref{gradient-of-j-inexact}, and  \eqref{condition-F-dxi-inexact}, we  obtain 
    \begin{eqnarray}\label{nabla-f-yk}
 \arraycolsep=1.4pt\def\arraystretch{1.35}  \begin{array}{lcl}
  \| \nabla f (\by^{k+1})\|^2 =  \| \nabla f (\by^{\tau_k+1})\|^2
&\overset{\eqref{two-vecs}}{\leq}&  5\|   \nabla f  (\by^{\tau_k+1})- \nabla F(X^{\tau_k})\|^2  + \frac{5}{4}\| \nabla F(X^{\tau_k})\|^2\\
 &\overset{\eqref{two-vecs}}{\leq}&  10\|   \nabla f  (\by^{\tau_k+1})- \nabla F(X^{\tau_k+1})\|^2 \\
 &+& 10\| \nabla F(X^{\tau_k+1})-\nabla F(X^{\tau_k}) \|^2 + \frac{5}{4}\| \nabla F(X^{\tau_k})\|^2\\
  &\leq&  m   \sum_{i=1}^m   (\frac{10\sigma_i^2}{54}  \|  \triangle \bx^{\tau_k}_i \|^2 + \frac{145\sigma_i^2}{54}   \|  \triangle \bx_i^{\tau_k+1}\|^2)  \\
  &\leq&  m   \sum_{i=1}^m   {3\sigma_i^2}  (\|  \triangle \bx^{\tau_k}_i \|^2 +   \|  \triangle \bx_i^{\tau_k+1}\|^2),
   \end{array}
  \end{eqnarray}   
 finishing the proof for the case of $k_0\geq2$. If $k_0=1$, then $k=\tau_k$ for any $k$, thereby leading to $$\| \nabla F(X^{k+1})-\nabla F(X^{\tau_k+1} )  \|^2  =0.$$
 Similar reasoning to show \eqref{nabla-F-Xk} and \eqref{nabla-f-yk} also allows us to prove \eqref{grad-L-bounded-eq-inexact-1}.
\end{proof}

 \subsubsection{Proof of Theorem  \ref{complexity-thorem-gradient-inexact}}  
 \begin{proof} Let $\Omega:=\{k_0+1,k_0+2,\ldots,k_0+k\}$. If $k_0=1$, then  
  \begin{eqnarray*} 
 \arraycolsep=1.4pt\def\arraystretch{1.35}  \begin{array}{lcl}
 && \min_{j\in\Omega}\max\{\|\nabla  F(X^{j+1})\|^2,  \|\nabla  f(\by^{j+1})\|^2\}\\
 &\leq&
\frac{1}{k}\sum_{j=k_0+1}^{k_0+k}\max\{\|\nabla  F(X^{j+1})\|^2,  \|\nabla  f(\by^{j+1})\|^2\}\\
 &\overset{\eqref{grad-L-bounded-eq-inexact-1}}{\leq}& \frac{1}{k}\sum_{j=k_0+1}^{k_0+k}\sum_{i=1}^m  \frac{6m\sigma_i^2}{\vartheta_i} \frac{\vartheta_i}{2} (\| \triangle \bx^{j}_i\|^2 + \| \triangle \bx^{j+1}_i\|^2)\\
  & {\leq}& \frac{\varrho}{2k}\sum_{j=k_0+1}^{k_0+k}\sum_{i=1}^m   \frac{\vartheta_i}{2} (\| \triangle \bx^{j}_i\|^2 + \| \triangle \bx^{j+1}_i\|^2)\\ 
&\overset{\eqref{max-FX-fy}}{\leq} & \frac{ \varrho}{2k}\sum_{j=k_0+1}^{k_0+k}(\varphi^{j-1}-\varphi^{j} + \varphi^j-\varphi^{j+1} )\\
&=& \frac{ \varrho}{2k} (\varphi^{k_0}-\varphi^{k+k_0} + \varphi^{k_0+1}-\varphi^{k+k_0+1})\\ 
&\leq&  \frac{\rho }{k} (\varphi^1-f^*)=\frac{\rho k_0 }{k} (\varphi^1-f^*),
    \end{array}  
 \end{eqnarray*}
  where the last inequality holds due to $\varphi^1\geq \varphi^k \geq f(\by^k) \geq f^*$ from Lemma \ref{L-bounded-decreasing-1} ii).

For $k_0\geq2$, by denoting $\triangle_i^j:=\| \triangle \bx^{j}_i\|^2 + \| \triangle \bx^{j+1}_i\|^2 +     \| \triangle \bx^{\tau_j}_i\|^2  +   \| \triangle \bx^{\tau_j+1}_i\|^2$, it follows  that \eqref{decreasing-property-1} that
 \begin{eqnarray} \label{max-FX-fy}
 \arraycolsep=1.4pt\def\arraystretch{1.35}  \begin{array}{lcl}
&&\max\{\|\nabla  F(X^{j+1})\|^2,  \|\nabla  f(\by^{j+1})\|^2\}\\
&\overset{\eqref{grad-L-bounded-eq-inexact}}{\leq}& \sum_{i=1}^m  \frac{6m\sigma_i^2}{\vartheta_i} \frac{\vartheta}{2} \triangle_i^j ~{\leq}~\frac{\varrho}{2} \sum_{i=1}^m   \frac{\vartheta_i}{2} \triangle_i^j\\
&\overset{\eqref{decreasing-property-1}}{\leq} &\frac{\varrho}{2}(\varphi^{j-1}-\varphi^{j} + \varphi^j-\varphi^{j+1})+\frac{\varrho}{2}(\varphi^{\tau_j-1}-\varphi^{\tau_j}+ \varphi^{\tau_j}-\varphi^{\tau_j+1})\\
&=& \frac{\varrho}{2}(\varphi^{j-1}-\varphi^{j} + \varphi^j-\varphi^{j+1} + \varphi^{\tau_j-1} -\varphi^{\tau_j+1}).
    \end{array}  
 \end{eqnarray}
We note that $k_0\geq 2$ implies $tk_0  +1 \leq(t+1)k_0-1$ for any $t\geq1$, which together with the non-increasing of $\{\varphi^{k}\}$ from Lemma \ref{L-bounded-decreasing-1} i) results in $\varphi^{1} \geq \varphi^{tk_0  +1} \geq \varphi^{(t+1)k_0-1}$ for any $t\geq0$.  By letting $s:= \lfloor k/k_0\rfloor+1$, we have 
   \begin{eqnarray*} 
 \arraycolsep=1.4pt\def\arraystretch{1.35}  \begin{array}{lllll}
 \sum_{j=k_0+1}^{k_0+k}(\varphi^{\tau_j-1}-\varphi^{\tau_j+1})
&\leq& k_0 \sum_{t=1}^{s}( \varphi^{tk_0-1}-\varphi^{tk_0+1})\\
&=&  k_0 \left(\varphi^{k_0-1}-\sum_{t=1}^{s-1}(\varphi^{tk_0  +1}-\varphi^{(t+1)k_0-1} )-\varphi^{s k_0  +1}\right)\\
&\leq&  k_0  (\varphi^{k_0-1}- \varphi^{s k_0  +1} ).
    \end{array}  
 \end{eqnarray*}
Using the above fact and \eqref{max-FX-fy} brings out 
  \begin{eqnarray*} 
 \arraycolsep=1.4pt\def\arraystretch{1.35}  \begin{array}{lcl}
 && \min_{j\in\Omega} \max\{\|\nabla  F(X^{j+1})\|^2,  \|\nabla  f(\by^{j+1})\|^2\}\\
  &\leq&
\frac{1}{k}\sum_{j=k_0+1}^{k_0+k}\max\{\|\nabla  F(X^{j+1})\|^2,  \|\nabla  f(\by^{j+1})\|^2\}\\
&\overset{\eqref{max-FX-fy}}{\leq} & \frac{\varrho}{2k}\sum_{j=k_0+1}^{k_0+k}(\varphi^{j-1}-\varphi^{j} + \varphi^j-\varphi^{j+1} + \varphi^{\tau_j-1} -\varphi^{\tau_j+1})\\
&\leq& \frac{\varrho}{2k} (\varphi^{k_0}-\varphi^{k+k_0} + \varphi^{k_0+1}-\varphi^{k+k_0+1}) +  \frac{\rho k_0}{2k} (\varphi^{k_0-1}- \varphi^{s k_0  +1} ) \\
&\leq& \frac{\varrho(2+k_0)}{2k} (\varphi^{1}-f^*)\\
&\leq& \frac{ \varrho k_0}{k} (\varphi^{1}-f^*),
    \end{array}  
 \end{eqnarray*}
 where the fourth inequality used   $ \varphi^1\geq \varphi^k \geq f(\by^k) \geq f^*$ from Lemma \ref{L-bounded-decreasing-1} ii) and the last one used $k_0\geq2$.
 \end{proof}

\end{document}